\documentclass{article}

\usepackage[final]{neurips_2021}

\usepackage[utf8]{inputenc} 
\usepackage[T1]{fontenc}    
\usepackage[colorlinks=true, linkcolor=blue, citecolor=blue,urlcolor=black]{hyperref}  
\usepackage{url}            
\usepackage{booktabs}       
\usepackage{amsfonts}       
\usepackage{nicefrac}       
\usepackage{microtype}      
\usepackage{xcolor}         

\definecolor{Green}{rgb}{0.13, 0.65, 0.3}
\definecolor{Amber}{rgb}{0.3, 0.5, 1.0}
\usepackage[]{color-edits}
\addauthor{HL}{Green}
\addauthor{TJ}{red}

\usepackage[utf8]{inputenc} 
\usepackage[T1]{fontenc}    

\usepackage{amsthm}
\usepackage[algo2e, ruled, vlined]{algorithm2e}
\usepackage{algorithmicx}
\usepackage{algorithm}
\usepackage[noend]{algcompatible}
\usepackage{bbm}      
\usepackage{mathtools}    
\usepackage{amsopn}
\usepackage{amssymb}
\usepackage{graphicx}
\usepackage{xcolor}

\definecolor{Green}{rgb}{0.13, 0.65, 0.3}
\definecolor{Amber}{rgb}{0.3, 0.5, 1.0}
\usepackage[]{color-edits}
\addauthor{HL}{Green}
\addauthor{TJ}{Amber}
\usepackage{comment}

\usepackage[skins,theorems]{tcolorbox} 
\tcbset{highlight math style={enhanced,colframe=red,colback=white,arc=0pt,boxrule=1pt}} 

\newcommand{\inner}[1]{ \left\langle {#1} \right\rangle }
\newcommand{\Ind}[1]{ \field{I}{\left\{{#1}\right\}} }
\newcommand{\Indt}[1]{ \field{I}_t{\left({#1}\right)} }
\newcommand{\Indtau}[1]{ \field{I}_\tau{\left({#1}\right)} }

\newcommand{\norm}[1]{\left\|{#1}\right\|}

\newcommand{\scO}{\mathcal{O}}

\newcommand{\gap}{\Delta}
\newcommand{\hatl}{\widehat{\ell}}

\newcommand{\whatq}{\widehat{q}}

\newcommand{\whatQ}{\widehat{Q}}
\newcommand{\whatV}{\widehat{V}}

\newcommand{\clip}{\text{clip}}
\newcommand{\wtilq}{\widetilde{q}}

\newcommand{\wtill}{\widetilde{\ell}}
\newcommand{\selfterm}{\mathbb{G}}

\theoremstyle{theorem} 
	\newtheorem{theorem}{Theorem}[subsection]
	\newtheorem{lemma}[theorem]{Lemma}
	
	\newtheorem{corollary}[theorem]{Corollary}
	\newtheorem{proposition}[theorem]{Proposition}
	\newtheorem{definition}[theorem]{Definition}

\renewcommand{\thetheorem}{%
	\ifnum\value{subsection}>0 
	\thesubsection
	\else
	\thesection
	\fi
	.\arabic{theorem}%
}

\allowdisplaybreaks

\usepackage{amsmath}
\DeclareMathOperator*{\argmin}{\arg\!\min}

\usepackage{amsthm}
\usepackage{amsmath}
\usepackage{amssymb}
\usepackage{graphicx}
\usepackage{mathtools}
\usepackage{enumerate}
\usepackage{enumitem}
\usepackage{footnote}
\usepackage{float}
\usepackage{xspace}
\usepackage{multirow}
\usepackage{nicefrac}
\usepackage{wrapfig}
\usepackage{framed}
\usepackage{url}
\usepackage{tikz}
\usetikzlibrary{arrows,chains,matrix,positioning,scopes}

\newcommand{\calA}{{\mathcal{A}}}

\newcommand{\calE}{{\mathcal{E}}}

\newcommand{\calP}{{\mathcal{P}}}

\newcommand{\paralog}{\beta}
\newcommand{\cnt}{B}
\newcommand{\cons}{J}
\newcommand{\pcons}{\iota}
\newcommand{\gapmin}{\Delta_{\textsc{min}}}
\DeclareMathOperator{\polylog}{{\ensuremath{\mathrm{polylog}}}}

\newcommand{\field}[1]{\mathbb{#1}}

\newcommand{\fR}{\field{R}}

\newcommand{\E}{\field{E}}

\newcommand{\Reg}{{\text{\rm Reg}}}
\newcommand{\EReg}{{\text{\rm EstReg}}}

\newcommand{\order}{\ensuremath{\mathcal{O}}}
\newcommand{\otil}{\ensuremath{\widetilde{\mathcal{O}}}}

\newcommand{\opt}{\mathring{q}}
\newcommand{\optpi}{\mathring{\pi}}

\newcommand{\myComment}[1]{\null\hfill\scalebox{0.9}{\text{\color{black}$\triangleright$ \textsf{#1}}}}

\newcommand{\rbr}[1]{\left(#1\right)}
\newcommand{\Bigrbr}[1]{\Big(#1\Big)}
\newcommand{\Biggrbr}[1]{\Bigg(#1\Bigg)}
\newcommand{\sbr}[1]{\left[#1\right]}
\newcommand{\Bigsbr}[1]{\Big[#1\Big]}
\newcommand{\Biggsbr}[1]{\Bigg[#1\Bigg]}
\newcommand{\cbr}[1]{\left\{#1\right\}}

\newcommand{\abr}[1]{\left|#1\right|}

\DeclareFontFamily{OMX}{MnSymbolE}{}
\DeclareFontShape{OMX}{MnSymbolE}{m}{n}{
    <-6>  MnSymbolE5
   <6-7>  MnSymbolE6
   <7-8>  MnSymbolE7
   <8-9>  MnSymbolE8
   <9-10> MnSymbolE9
  <10-12> MnSymbolE10
  <12->   MnSymbolE12}{}
\DeclareSymbolFont{mnlargesymbols}{OMX}{MnSymbolE}{m}{n}
\SetSymbolFont{mnlargesymbols}{bold}{OMX}{MnSymbolE}{b}{n}
\DeclareMathDelimiter{\llangle}{\mathopen}{mnlargesymbols}{'164}{mnlargesymbols}{'164}
\DeclareMathDelimiter{\rrangle}{\mathclose}{mnlargesymbols}{'171}{mnlargesymbols}{'171}

\usepackage{prettyref}
\newcommand{\pref}[1]{\prettyref{#1}}

\newcommand{\savehyperref}[2]{\texorpdfstring{\hyperref[#1]{#2}}{#2}}
\newrefformat{eq}{\savehyperref{#1}{Eq.~\eqref{#1}}}
\newrefformat{eqn}{\savehyperref{#1}{Equation~\eqref{#1}}}
\newrefformat{lem}{\savehyperref{#1}{Lemma~\ref*{#1}}}
\newrefformat{def}{\savehyperref{#1}{Definition~\ref*{#1}}}
\newrefformat{line}{\savehyperref{#1}{line~\ref*{#1}}}
\newrefformat{thm}{\savehyperref{#1}{Theorem~\ref*{#1}}}
\newrefformat{corr}{\savehyperref{#1}{Corollary~\ref*{#1}}}
\newrefformat{cor}{\savehyperref{#1}{Corollary~\ref*{#1}}}
\newrefformat{col}{\savehyperref{#1}{Corollary~\ref*{#1}}}
\newrefformat{sec}{\savehyperref{#1}{Section~\ref*{#1}}}
\newrefformat{app}{\savehyperref{#1}{Appendix~\ref*{#1}}}
\newrefformat{assum}{\savehyperref{#1}{Assumption~\ref*{#1}}}
\newrefformat{ex}{\savehyperref{#1}{Example~\ref*{#1}}}
\newrefformat{fig}{\savehyperref{#1}{Figure~\ref*{#1}}}
\newrefformat{alg}{\savehyperref{#1}{Algorithm~\ref*{#1}}}
\newrefformat{rem}{\savehyperref{#1}{Remark~\ref*{#1}}}
\newrefformat{conj}{\savehyperref{#1}{Conjecture~\ref*{#1}}}
\newrefformat{prop}{\savehyperref{#1}{Proposition~\ref*{#1}}}
\newrefformat{proto}{\savehyperref{#1}{Protocol~\ref*{#1}}}
\newrefformat{prob}{\savehyperref{#1}{Problem~\ref*{#1}}}
\newrefformat{claim}{\savehyperref{#1}{Claim~\ref*{#1}}}
\newrefformat{que}{\savehyperref{#1}{Question~\ref*{#1}}}
\newrefformat{obs}{\savehyperref{#1}{Observation~\ref*{#1}}}

\newcommand{\ftrl}{\textsc{FTRL}\xspace}

\title{The best of both worlds: stochastic and adversarial episodic MDPs with unknown transition}

\author{%
Tiancheng Jin \\
University of Southern California\\
\texttt{tiancheng.jin@usc.edu} \\
\And
 Longbo Huang \\
Tsinghua University\\
\texttt{longbohuang@tsinghua.edu.cn} \\
\And
  Haipeng Luo \\
University of Southern California\\
\texttt{haipengl@usc.edu} \\
}

\begin{document}

\maketitle

\begin{abstract}
We consider the best-of-both-worlds problem for learning an episodic Markov Decision Process through $T$ episodes, with the goal of achieving $\otil(\sqrt{T})$ regret when the losses are adversarial and simultaneously $\order(\polylog(T))$ regret when the losses are (almost) stochastic.
Recent work by~\citep{jin2020simultaneously} achieves this goal when the fixed transition is known, and leaves the case of unknown transition as a major open question.
In this work, we resolve this open problem by using the same Follow-the-Regularized-Leader (\ftrl) framework together with a set of new techniques.
Specifically, we first propose a loss-shifting trick in the \ftrl analysis, which greatly simplifies the approach of~\citep{jin2020simultaneously} and already improves their results for the known transition case.
Then, we extend this idea to the unknown transition case and develop a novel analysis which upper bounds the transition estimation error by (a fraction of) the regret itself in the stochastic setting, a key property to ensure $\order(\polylog(T))$ regret.
\end{abstract}

\section{Introduction}\label{sec:intro}

We study the problem of learning finite-horizon Markov Decision Processes (MDPs) with unknown transition through $T$ episodes. 
In each episode, the learner starts from a fixed initial state and repeats the following for a fixed number of steps: select an available action, incur some loss, and transit to the next state according to a fixed but unknown transition function.
The goal of the learner is to minimize her regret, which is the difference between her total loss and that of the optimal stationary policy in hindsight. 

When the losses are stochastically generated, \citep{simc2019, yang2021q} show that $\order(\log T)$ regret is achievable (ignoring dependence on some gap-dependent quantities for simplicity).
On the other hand, even when the losses are adversarially generated, \citep{rosenberg19a, jin2019learning} show that $\otil(\sqrt{T})$ regret is achievable.\footnote{Throughout the paper, we use $\otil(\cdot)$ to hide polylogarithmic terms.}
Given that the existing algorithms for these two worlds are substantially different,
\citet{jin2020simultaneously} asked the natural question of whether one can achieve the \emph{best of both worlds}, that is, enjoying (poly)logarithmic regret in the stochastic world while simultaneously ensuring some worst-case robustness in the adversarial world.
Taking inspiration from the bandit literature and using the classic Follow-the-regularized-Leader (\ftrl) framework with a novel regularizer, they successfully achieved this goal, albeit under a strong restriction that the transition has to be known ahead of time.
Since it is highly unclear how to ensure that the transition estimation error is only $\order(\polylog(T))$, extending their results to the unknown transition case is highly challenging and was left as a key open question.

In this work, we resolve this open question and propose the first algorithm with such a best-of-both-worlds guarantee under unknown transition.
Specifically, our algorithm enjoys $\otil(\sqrt{T})$ regret always, and simultaneously $\order(\log^2 T)$ regret if the losses are i.i.d. samples of a fixed distribution.
More generally, our polylogarithmic regret holds under a general condition similar to that of~\citep{jin2020simultaneously}, which requires neither independence nor identical distributions.
For example, it covers the corrupted i.i.d. setting where our algorithm achieves $\otil(\sqrt{C})$ regret with $C \leq T$ being the total amount of corruption.


\paragraph{Techniques}
Our results are achieved via three new techniques.
First, we propose a new \textit{loss-shifting} trick for the \ftrl analysis when applied to MDPs.
While similar ideas have been used for the special case of multi-armed bandits (e.g.,~\citep{wei2018more, zimmert2019optimal, lee2020closer, zimmert2021tsallis}),
its extension to MDPs has eluded researchers, which is also the reason why~\citep{jin2020simultaneously} resorts to a different approach with a highly complex analysis involving analyzing the inverse of the non-diagonal Hessian of a complicated regularizer.
Instead, inspired by the well-known performance difference lemma, we design a key shifting function in the \ftrl analysis, which helps reduce the variance of the stability term  and eventually leads to an adaptive bound with a certain self-bounding property known to be useful for the stochastic world.
To better illustrate this idea, we use the known transition case as a warm-up example in \pref{sec:loss_shifting}, and show that the simple Tsallis entropy regularizer (with a diagonal Hessian) is already enough to achieve the best-of-both-worlds guarantee.
This not only greatly simplifies the approach of~\cite{jin2020simultaneously} (paving the way for extension to unknown transition), but also leads to bounds with better dependence on some parameters, which on its own is a notable result already.

Our second technique is a new framework to deal with unknown transition under adversarial losses, which is important for incorporating the loss-shifting trick mentioned above.
Specifically, when the transition is unknown, prior works~\citep{rosenberg19a, rosenberg2019online, jin2019learning, lee2020bias} perform \ftrl over the set of all plausible occupancy measures according to a confident set of the true transition,
which can be seen as a form of optimism encouraging exploration.
Since our loss-shifting trick requires a fixed transition, we propose to move the optimism from the decision set of \ftrl to the losses fed to \ftrl.
More specifically, we perform \ftrl over the empirical transition in some doubling epoch schedule, and add (negative) bonuses to the loss functions so that the algorithm is optimistic and never underestimates the quality of a policy, an idea often used in the stochastic setting (e.g.,~\citep{azar2017minimax}).
See \pref{sec:algorithms} for the details of our algorithm.

Finally, we develop a new analysis to show that the transition estimation error of our algorithm is only polylogarithmic in $T$, overcoming the most critical obstacle in achieving best-of-both-worlds.
An important aspect of our analysis is to make use of the amount of underestimation of the optimal policy, a term that is often ignored since it is nonpositive for optimistic algorithms.
We do so by proposing a novel decomposition of the regret inspired by the work of~\cite{simc2019},
and show that in the stochastic world, every term in this decomposition can be bounded by a fraction of the regret itself plus some polylogarithmic terms, which is enough to conclude the final polylogarithmic regret bound.
See \pref{sec:analysis} for a formal summary of this idea.

\paragraph{Related work} 
For earlier results in each of the two worlds, we refer the readers to the systematic surveys in~\citep{simc2019,yang2021q, jin2019learning}. 
The work closest to ours is~\citep{jin2020simultaneously} which assumes known transition, and as mentioned, we strictly improve their bounds and more importantly extend their results to the unknown transition case.

Two recent works~\citep{lykouris2021corruption, chen2021improved} also consider the corrupted stochastic setting, where both the losses and the transition function can be corrupted by a total amount of $C$.
This is more general than our results since we assume a fixed transition and only allow the losses to be corrupted.
On the other hand, their bounds are worse than ours when specified to our setting --- 
\citep{lykouris2021corruption} ensures a gap-dependent polylogarithmic regret bound of $\order(C \log^3 T + C^2)$,
while \citep{chen2021improved} achieves $\order(\log^3 T + C)$ but with a potentially larger gap-dependent quantity.
Therefore, neither result provides a meaningful guarantee in the adversarial world when $C=T$, while our algorithm always ensures a robustness guarantee with $\otil(\sqrt{T})$ regret.
Their algorithms are also very different from ours and are not based on \ftrl.

The question of achieving best-of-both-worlds guarantees for the special case of multi-armed bandits was first proposed in~\citep{bubeck2012best}.
Since then, many improvements using different approaches have been established over the years~\citep{seldin2014one, auer2016algorithm, seldin2017improved, wei2018more, lykouris2018stochastic, gupta2019better, zimmert2019beating, zimmert2021tsallis, lee2021achieving}.  
One notable and perhaps surprising approach is to use the \ftrl framework, originally designed only for the adversarial settings but later found to be able to automatically adapt to the stochastic settings as long as certain regularizers are applied~\citep{wei2018more, zimmert2019beating, zimmert2021tsallis}.
Our approach falls into this category, and our regularizer design is also based on these prior works.
As mentioned, however, obtaining our results requires the new loss-shifting technique as well as the novel analysis on controlling the estimation error, both of which are critical to address the extra challenges presented in MDPs.

\section{Preliminaries}\label{sec:prelim}

We consider the problem of learning an episodic MDP through $T$ episodes, where the MDP is formally defined by a tuple $(S,A,L,P,\cbr{\ell_t}_{t=1}^{T})$ with $S$ being a finite state set, $A$ being a finite action set, $L$ being the horizon, $\ell_t: S \times A \rightarrow [0,1]$ being the loss function of episode $t$,  and $P: S\times A\times S \rightarrow [0,1]$ being the transition function so that $P(s'|s,a)$ is the probability of moving to state $s'$ after executing action $a$ at state $s$. 

Without loss of generality~\citep{jin2019learning}, the MDP is assumed to have a layer structure, that is, the state set $S$ is partitioned into $L+1$ subsets $S_0$, $S_1$, \ldots, $S_L$ such that the state transition is only possible from one layer to the next layer (in other words, $P(s'|s,a)$ must be zero unless $s\in S_k$ and $s' \in S_{k+1}$ for some $k \in \{0, \ldots, L-1\}$).
Moreover, $S_0$ contains $s_0$ only (the initial state), and $S_L$ contains $s_L$ only (the terminal state).
We use $k(s)$ to represent the layer to which state $s$ belongs.

Ahead of time, the environment decides an MDP with $P$ and  $\cbr{\ell_t}_{t=1}^{T}$ unknown to the learner. The interaction proceeds through $T$ episodes. 
In episode $t$, the learner selects a stochastic policy $\pi_t: S\times A \rightarrow [0,1]$ where $\pi_t(a|s)$ denotes the probability of taking action $a$ at state $s$.\footnote{%
Note that $\pi_t(\cdot|s_L)$ is not meaningful since no action will be taken at $s_L$.
For conciseness, however, we usually define functions over $S\times A$ instead of $(S\setminus\{s_L\}) \times A$.
} 
Starting from the initial state $s^t_0=s_0$, the learner then repeatedly selects an action $a_k^t$ drawn from $\pi_t\rbr{ \cdot \left\lvert s_k^t \right. }$, suffers loss $\ell_t(s_k^t,a_k^t)$, and transits to the next state $s_{k+1}^t \in S_{k+1}$ for $k=0,\ldots, L-1$, until reaching the terminal state $s_L$. 
At the end of the episode, the learner receives some feedback on the loss function $\ell_t$.
In the \textit{full-information} setting, the learner observes the entire loss function $\ell_t$, while in the more challenging \textit{bandit feedback} setting, the learner only observes the losses of those visited state-action pairs, that is, $\ell_t(s_0^t,a_0^t), \ldots, \ell_t(s_{L-1}^t,a_{L-1}^t)$. 

With slight abuse of notation, we denote the expected loss of a policy $\pi$ for episode $t$ by $\ell_t(\pi) = \E\sbr{ \left. \sum_{k=0}^{L-1} \ell_t(s_k, a_k) \right\rvert P, \pi }$, where the trajectory $\{(s_{k},a_{k})\}_{k=0,\ldots,L-1}$ is the generated by executing policy $\pi$ under transition $P$.
The regret of the learner against some policy $\pi$ is then defined as $\Reg_T(\pi)= \E\sbr{\sum_{t=1}^{T} \ell_t(\pi_t) - \ell_t(\pi)}$,
and we denote by $\optpi$ one of the optimal policies in hindsight such that $\Reg_T(\optpi) = \max_{\pi} \Reg_T(\pi)$.


\paragraph{Adversarial world versus stochastic world} 
We consider two different setups depending on how the loss functions $\ell_1, \ldots, \ell_T$ are generated.
In the adversarial world, the environment decides the loss functions arbitrarily with knowledge of the learner's algorithm (but not her randomness).
In this case, the goal is to minimize the regret against the best policy $\Reg_T(\optpi)$,
with the best existing upper bound being $\otil(L|S|\sqrt{|A|T})$~\citep{rosenberg19a, jin2019learning} and the best lower bound being $\Omega(L\sqrt{|S||A|T})$~\citep{jin2018q} (for both full-information and bandit feedback).

In the stochastic world, following~\citep{jin2020simultaneously} (which generalizes the bandit case of~\citep{zimmert2019optimal,zimmert2021tsallis}), we assume that the loss functions satisfy the following condition:  there exists a deterministic policy $\pi^\star: S \rightarrow A$, a gap function $\gap: S\times A \rightarrow \fR_{+}$ and a constant $C > 0$ such that 
 \begin{equation}
 \label{eq:loss_condition}
 \Reg_T(\pi^\star) \geq \E\sbr{  \sum_{t=1}^{T} \sum_{s\neq s_L} \sum_{a \neq \pi^\star(s)} q_t(s,a) \gap(s,a)   }  - C,
 \end{equation}
where $q_t(s,a)$ is the probability of the learner visiting $(s,a)$ in episode $t$.
This general condition covers the heavily-studied i.i.d. setting where $\ell_1, \ldots, \ell_T$ are i.i.d. samples of a fixed distribution, in which case $C=0$, $\pi^\star$ is simply the optimal policy, and $\gap$ is the gap function with respect to the optimal $Q$-function.
More generally, the condition also covers the corrupted i.i.d. setting with $C$ being the total amount of corruption.
We refer the readers to~\citep{jin2020simultaneously} for detailed explanation.
In this stochastic world, our goal is to minimize regret against $\pi^\star$, that is, $\Reg_T(\pi^\star)$.\footnote{%
Some works (such as~\citep{jin2020simultaneously}) still consider minimizing $\Reg_T(\optpi)$ as the goal in this case.
More discussions are deferred to the last paragraph of \pref{sec:results}.
}
With unknown transition, this general setup has not been studied before,
but for specific examples such as the i.i.d. setting, regret bounds of order $\order(\frac{\log T}{\gapmin})$ where $\gapmin = \min_{s,a\neq \pi^\star(s)}\gap(s,a)$ have been derived~\citep{simc2019, yang2021q}.



\paragraph{Occupancy measure and \ftrl} To solve this problem with online learning techniques, a commonly used concept is the occupancy measure. Specifically, an occupancy measure $q^{\bar{P},\pi}: S\times A \rightarrow [0,1]$ associated with a policy $\pi$ and a transition function $\bar{P}$ is such that $q^{\bar{P},\pi}(s,a)$ equals the probability of  visiting state-action pair $(s,a)$ under the given policy $\pi$ and transition $\bar{P}$. 
Our earlier notation $q_t$ in \pref{eq:loss_condition} is thus simply a shorthand for $q^{P, \pi_t}$.
Moreover, by definition, $\ell_t(\pi)$ can be rewritten as $\inner{q^{P,\pi}, \ell_t}$ by naturally treating $q^{P,\pi}$ and $\ell_t$ as vectors in $\fR^{|S|\times|A|}$, and thus the regret $\Reg_T(\pi)$ can be written as $\E\sbr{\sum_{t=1}^{T} \inner{ q_t - q^{P, \pi}, \ell_t}}$,
connecting the problem to online linear optimization.


Given a transition function $\bar{P}$, we denote by $\Omega(\bar{P}) = \big\{q^{\bar{P},\pi}: \text{$\pi$ is a stochastic policy}\big\}$ the set of all valid occupancy measures associated with the transition $\bar{P}$.
It is known that $\Omega(\bar{P})$ is a simple polytope with $\order(|S||A|)$ constraints~\citep{zimin2013}.
When $P$ is unknown, our algorithm uses an estimated transition $\bar{P}$ as a proxy and searches for a ``good'' occupancy measure within $\Omega(\bar{P})$.
More specifically, this is done by the classic Follow-the-Regularized-Leader (\ftrl) framework which solves the following at the beginning of episode $t$:
\begin{equation}\label{eq:FTRL}
\widehat{q}_t = \argmin_{q \in \Omega(\bar{P})} \inner{q, \sum_{\tau < t} \hatl_\tau} + \phi_t(q),
\end{equation}
where $\hatl_\tau$ is some estimator for $\ell_\tau$ and $\phi_t$ is some regularizer.
The learner's policy $\pi_t$ is then defined through $\pi_t(a|s) \propto \widehat{q}_t(s,a)$.
Note that we have $\widehat{q}_t = q^{\bar{P}, \pi_t}$ but not necessarily $\widehat{q}_t = q_t$ unless $\bar{P} = P$.


\section{Warm-up for Known Transition: A New Loss-shifting Technique}\label{sec:loss_shifting}

One of the key components of our approach is a new loss-shifting technique for analyzing FTRL applied to MDPs.
To illustrate the key idea in a clean manner, in this section we focus on the known transition setting with bandit feedback, the same setting studied by~\citet{jin2020simultaneously}.
As we will show, our method not only improves their bounds, but also significantly simplifies the analysis, which paves the way for extending the result to the unknown transition setting studied in following sections.


First note that when $P$ is known, one can simply take $\bar{P}=P$ (so that $\whatq_t = q_t$) and use the standard importance-weighted estimator $\hatl_\tau(s,a) = \ell_\tau(s,a)\Indtau{s,a}/q_\tau(s,a)$ in the \ftrl framework~\pref{eq:FTRL}, where $\Indtau{s,a}$ is $1$ if $(s,a)$ is visited in episode $\tau$, and $0$ otherwise.
It remains to determine the regularizer $\phi_t$.
While there are many choices of $\phi_t$ leading to $\sqrt{T}$-regret in the adversarial world, obtaining logarithmic regret in the stochastic world requires some special property of the regularizer.
Specifically, generalizing the idea of~\citep{zimmert2019optimal} for multi-armed bandits, \citep{jin2020simultaneously} shows that it suffices to find $\phi_t$
such that the following adaptive regret bound holds
\begin{equation}
\Reg_T(\optpi) \lesssim \E\sbr{\sum_{t=1}^{T} \sum_{s \neq s_L} \sum_{a \neq \pi^\star(s) } \sqrt{\frac{q_t(s,a)}{t}}   }, \label{eq:known_self_bounding_regret_bounding}
\end{equation}
which then automatically implies logarithmic regret under \pref{eq:loss_condition}.
This is because \pref{eq:known_self_bounding_regret_bounding} admits a self-bounding property under \pref{eq:loss_condition} --- one can bound the right-hand side of \pref{eq:known_self_bounding_regret_bounding} as follows using AM-GM inequality (for any $z>0$), which can then be related to the regret itself using \pref{eq:loss_condition}:
\begin{equation}\label{eq:self_bounding_argument}
\E\sbr{\sum_{t=1}^{T} \sum_{s \neq s_L} \sum_{a \neq \pi^\star(s)} \frac{q_t(s,a)\gap(s,a)}{2z} + \frac{z}{2t\gap(s,a)}}
\leq \frac{\Reg_T(\optpi) + C}{2z} + z\sum_{s\neq s_L}\sum_{a\neq \pi^\star(s)}\frac{\log T}{\gap(s,a)}.
\end{equation}
Rearranging and picking the optimal $z$ then shows a logarithmic bound for $\Reg_T(\optpi)$ (see Section~2 of \cite{jin2020simultaneously} for detailed discussions). 

To achieve \pref{eq:known_self_bounding_regret_bounding},
a natural candidate of $\phi_t$ would be a direct generalization of the Tsallis-entropy regularizer of~\citep{zimmert2019optimal}, which takes the form $\phi_t(q) = -\frac{1}{\eta_t}\sum_{s, a}\sqrt{q(s,a)}$ with $\eta_t = 1/\sqrt{t}$.
However, \citet{jin2020simultaneously} argued that it is highly unclear how to achieve \pref{eq:known_self_bounding_regret_bounding} with this natural candidate,
and instead, inspired by~\citep{zimmert2019beating} they ended up using a different regularizer with a complicated non-diagonal Hessian to achieve \pref{eq:known_self_bounding_regret_bounding}, which makes the analysis extremely complex since it requires analyzing the inverse of this non-diagonal Hessian.

Our first key contribution is to show that this natural and simple candidate is in fact (almost) enough to achieve \pref{eq:known_self_bounding_regret_bounding} after all.
To show this, we propose a new a loss-shifting technique in the analysis.
Similar techniques have been used for multi-armed bandits, but the extension to MDPs is much less clear.
Specifically, observe that for any \textit{shifting function} $g_\tau: S\times A \rightarrow \fR$ such that the value of $\inner{q, g_\tau}$ is independent of $q$ for any $q \in \Omega(\bar{P})$, we have
\begin{equation}\label{eq:loss_shifting}
\widehat{q}_t = \argmin_{q \in \Omega(\bar{P})} \inner{q, \sum_{\tau < t} \hatl_\tau} + \phi_t(q) = \argmin_{q \in \Omega(\bar{P})} \inner{q, \sum_{\tau < t} (\hatl_\tau + g_\tau)} + \phi_t(q).
\end{equation}
Therefore, we can pretend that the learner is performing \ftrl over the shifted loss sequence $\{\hatl_\tau + g_\tau\}_{\tau < t}$ (even when $g_\tau$ is unknown to the learner).
The advantage of analyzing \ftrl over this shifted loss sequence is usually that it helps reduce the variance of the loss functions.

For multi-armed bandits, prior works~\citep{wei2018more, zimmert2019optimal} pick $g_\tau$ to be a constant such as the negative loss of the learner in episode $\tau$.
For MDPs, however, this is not enough to show \pref{eq:known_self_bounding_regret_bounding}, as already pointed out by~\citet{jin2020simultaneously} (which is also the reason why they resorted to a different approach).
Instead, we propose the following shifting function:
\begin{equation}
g_\tau(s,a) = \widehat{Q}_\tau(s,a) - \widehat{V}_\tau(s) - \hatl_\tau(s,a), \quad \forall (s,a) \in S\times A, \label{eq:loss_shift_function}
\end{equation}
where $\widehat{Q}_\tau$ and $\widehat{V}_\tau$ are the state-action and state value functions with respect to the transition $\bar{P}$, the loss function $\hatl_\tau$, and the policy $\pi_\tau$, that is:
$\widehat{Q}_\tau(s,a) = \hatl_\tau(s,a) + \E_{s' \sim \bar{P}(\cdot|s,a)}[\widehat{V}_\tau(s')] $ and $\widehat{V}_\tau(s) = \E_{a \sim \pi_\tau(\cdot|s)}[\widehat{Q}_\tau(s,a)]$ (with $\widehat{V}_\tau(s_L) = 0$).
This indeed satisfies the invariant condition since using a well-known performance difference lemma one can show $\inner{q, g_\tau} = -\widehat{V}_\tau(s_0)$ for any $q \in \Omega(\bar{P})$ (\pref{lem:loss_shifting_invariant}).
With this shifting function, the learner is equivalently running \ftrl over the ``advantage'' functions ($\widehat{Q}_\tau(s,a) - \widehat{V}_\tau(s)$ is often called the advantage at $(s,a)$ in the literature).

More importantly, it turns out that when seeing \ftrl in this way, a standard analysis with some direct calculation already shows \pref{eq:known_self_bounding_regret_bounding}.
One caveat is that since $\widehat{Q}_\tau(s,a) - \widehat{V}_\tau(s)$ can potentially have a large magnitude, we also need to stabilize the algorithm by adding a small amount of the so-called log-barrier regularizer to the Tsallis entropy regularizer, an idea that has appeared in several prior works (see~\citep{jin2020simultaneously} and references therein).
We defer all details including the concrete algorithm and analysis to \pref{app:bobw_known_transition}, and show the final results below.

\begin{theorem}
\label{thm:known_transition_main}
When $P$ is known, \pref{alg:known_transition_tsallis_entropy_algoritm} (with parameter $\gamma=1$) ensures the optimal regret $\Reg_T(\optpi) = \order(\sqrt{L|S||A|T})$ in the adversarial world, and simultaneously $\Reg_T(\pi^\star) \leq \Reg_T(\optpi) = \order(U+\sqrt{UC})$ where $U =\frac{L|S|\log T}{\gapmin}+L^4\sum_{s\neq s_L}\sum_{a\neq \pi^\star(s)} \frac{\log T}{\gap(s,a)}$ in the stochastic world. 
\end{theorem}

Our bound for the stochastic world is even better than~\citep{jin2020simultaneously}
(their $U$ has an extra $|A|$ factor in the first term and an extra $L$ factor in the second term).
By setting the parameter $\gamma$ differently, one can also improve $L^4$ to $L^3$, matching the best existing result from~\citep{simc2019} for the i.i.d. setting with $C=0$ (this would worsen the adversarial bound though).
Besides this improvement, we emphasize again that the most important achievement of this approach is that it significantly simplifies the analysis, making the extension to the unknown transition setting possible.

\section{Main Algorithms and Results}\label{sec:algorithms}

We are now ready to introduce our main algorithms and results for the unknown transition case, with either full-information or bandit feedback.
The complete pseudocode is shown in \pref{alg:bobw_framework},
which is built with two main components: a new framework to deal with unknown transitions and adversarial losses (important for incorporating our loss-shifting technique), and special regularizers for \ftrl.
We explain these two components in detail below.

\paragraph{A new framework for unknown transitions and adversarial losses}
When the transition is unknown, a common practice (which we also follow) is to maintain an empirical transition along with a shrinking confidence set of the true transition, usually updated in some doubling epoch schedule.
More specifically, a new epoch is started whenever the total number of visits to some state-action pair is doubled (compared to the beginning of this epoch), thus resulting in at most $\order\rbr{|S||A|\log T}$ epochs.
We denote by $i(t)$ the epoch index to which episode $t$ belongs. 
At the beginning of each epoch $i$, we calculate the empirical transition $\bar{P}_i$ (fixed through this epoch) as: 
\begin{equation}
\label{eq:empirical_mean_transition_def} 
\bar{P}_i(s'|s,a) = \frac{ m_i(s,a,s')  }{m_i(s,a)}, \quad \forall (s,a,s') \in S_k \times A \times S_{k+1}, \; k = 0,\ldots L-1, 
\end{equation}
where $m_i(s,a)$ and $m_i(s,a,s')$ are the total number of visits to $(s,a)$ and $(s,a,s')$ respectively prior to epoch $i$.\footnote{When $m_i(s,a) = 0$, we simply let $\bar{P}_i(\cdot|s,a)$ be an arbitrary distribution.}
The confidence set of the true transition for this epoch is then defined as
\begin{equation*}
\calP_i = \cbr{ \widehat{P}: \abr{ \widehat{P}(s'|s,a) - \bar{P}_i(s'|s,a) } \leq B_i(s,a,s'), \; \forall (s,a,s')\in S_k \times A \times S_{k+1}, k < L },
\end{equation*} 
where $B_i$ is Bernstein-style confidence width (taken from~\cite{jin2019learning}):
\begin{equation}
\label{eq:confidence_width_def} 
B_i(s,a,s') = \min\cbr{  2 \sqrt{\frac{\bar{P}_i(s'|s,a)\ln \rbr{\frac{T|S||A|}{\delta}}}{m_{i}(s,a)}} + \frac{14\ln\rbr{\frac{T|S||A|}{\delta}}}{3m_{i}(s,a)}, \; 1}
\end{equation}
for some confidence parameter $\delta \in (0,1)$.
As~\citep[Lemma~2]{jin2019learning} shows, the true transition $P$ is contained in the confidence set $\calP_i$ for all epoch $i$ with probably at least $1-4\delta$.

When dealing with adversarial losses, prior works~\citep{rosenberg19a, rosenberg2019online, jin2019learning, lee2020bias} perform \ftrl (or a similar algorithm called Online Mirror Descent) over the set of all plausible occupancy measures $\Omega(\calP_i) = \{q \in \Omega(\widehat{P}): \widehat{P} \in \calP_i\}$ during epoch $i$, which can be seen as a form of optimism and encourages exploration.
This framework, however, does not allow us to apply the loss-shifting trick discussed in \pref{sec:loss_shifting} --- indeed, our key shifting function \pref{eq:loss_shift_function} is defined in terms of some fixed transition $\bar{P}$, and the required invariant condition on $\inner{q, g_\tau}$ only holds for $q \in \Omega(\bar{P})$ but not $q \in \Omega(\calP_i)$.

Inspired by this observation, we propose the following new approach.
First, to directly fix the issue mentioned above, for each epoch $i$, we run a new instance of \ftrl simply over $\Omega(\bar{P}_i)$.
This is implemented by keeping track of the epoch starting time $t_i$ and only using the cumulative loss $\sum_{\tau=t_i}^{t-1}\hatl_\tau$ in the \ftrl update (\pref{eq:epoch_FTRL}).
Therefore, in each epoch, we are pretending to deal with a known transition problem, making the same loss-shifting technique discussed in \pref{sec:loss_shifting} applicable.

However, this removes the critical optimism in the algorithm and does not admit enough exploration.
To fix this, our second modification is to feed \ftrl with optimistic losses constructed by adding some (negative) bonus term, an idea often used in the stochastic setting.
More specifically, we subtract $L \cdot B_i(s,a)$ from the loss for each $(s,a)$ pair, where $B_i(s,a) = \min\big\{1, \sum_{s' \in S_{k(s)+1}} B_i(s,a,s')\big\}$; see \pref{eq:adjusted_loss}.
In the full-information setting, this means using $\hatl_t(s,a) = \ell_t(s,a) - L \cdot B_{i}(s,a)$.  
In the bandit setting, note that the importance-weighted estimator discussed in \pref{sec:loss_shifting} is no longer applicable since the transition is unknown (making $q_t$ also unknown), and~\citep{jin2019learning} proposes to use $\frac{\ell_t(s,a) \cdot \Indt{s,a}}{u_t(s,a)}$ instead, where $\Indt{s,a}$ is again the indicator of whether $(s,a)$ is visited during episode $t$, and $u_t(s,a)$ is the so-called upper occupancy measure defined as
\begin{equation}\label{eq:UOB}
u_t(s,a) = \max_{\widehat{P} \in \calP_{i(t)}} q^{\widehat{P}, \pi_t}(s,a)
\end{equation}
and can be efficiently computed via the \textsc{Comp-UOB} procedure of~\citep{jin2019learning}.
Our final adjusted loss estimator is then $\hatl_t(s,a) = \frac{\ell_t(s,a) \cdot \Indt{s,a} }{u_t(s,a)}  - L \cdot B_{i}(s,a)$.
In our analysis, we show that these adjusted loss estimators indeed make sure that we only underestimate the loss of each policy, which encourages exploration.

With this new framework, it is not difficult to show $\sqrt{T}$-regret in the adversarial world using many standard choices of the regularizer $\phi_t$ (which recovers the results of~\citep{rosenberg19a, jin2019learning} with a different approach).
To further ensure polylogarithmic regret in the stochastic world, however, we need some carefully designed regularizers discussed next.

\DontPrintSemicolon 
\setcounter{AlgoLine}{0}
\begin{savenotes}
\begin{algorithm}[t]
	\caption{Best-of-both-worlds for Episodic MDPs with Unknown Transition}
	\label{alg:bobw_framework}
	
	\textbf{Input:} confidence parameter $\delta$.
	
	\textbf{Initialize:} epoch index $i=1$ and epoch starting time $t_i = 1$.
	
	\textbf{Initialize:} $\forall (s,a,s')$, set counters $m_1(s,a) = m_1(s,a,s') = m_0(s,a) = m_0(s,a,s') = 0$.
	
	\textbf{Initialize:} empirical transition $\bar{P}_1$ and confidence width $B_1$ based on \pref{eq:empirical_mean_transition_def} and \pref{eq:confidence_width_def}. 
		
	\For{$t=1,\ldots,T$}{
		Let $\phi_t$ be \pref{eq:Shannon} for full-information feedback or \pref{eq:Tsallis} for bandit feedback, and compute
	   \begin{equation}\label{eq:epoch_FTRL}
	    \widehat{q}_t = \argmin_{q\in \Omega\rbr{\bar{P}_i}}  \inner{q,  \sum_{\tau=t_i}^{t-1}\hatl_\tau} + \phi_t(q).
	    \end{equation}
         
		Compute policy $\pi_t$ from $\widehat{q}_t$ such that
		$\pi_t(a|s) \propto \widehat{q}_t(s,a)$.\footnote{If $\sum_{b\in A} \widehat{q}_t(s,b)=0$, we let $\pi_t$ to be the uniform distribution.}
		
		Execute policy $\pi_t$ and obtain trajectory $(s_k^{t}, a_k^{t})$ for $k = 0, \ldots, L-1$.
		
		Construct adjusted loss estimator $\hatl_t$ such that
		\begin{equation}\label{eq:adjusted_loss}
		\hatl_t(s,a) = \begin{cases}
		\ell_t(s,a) - L \cdot B_{i}(s,a), &\text{for full-information feedback,} \\
		\frac{\ell_t(s,a) \cdot \Indt{s,a} }{u_t(s,a)}  - L \cdot B_{i}(s,a), &\text{for bandit feedback,}
		\end{cases}
		\end{equation}
		where $B_i(s,a) = \min\big\{1, \sum_{s' \in S_{k(s)+1}} B_i(s,a,s')\big\}$, $\Indt{s,a} = \Ind{\exists k, (s,a)=(s_k^{t}, a_k^{t})}$, 
		and $u_t$ is the upper occupancy measure defined in \pref{eq:UOB}.
		
		Increment counters: for each $k<L$, 
		$m_i(s_k^t,a_k^t,s_{k+1}^t) \overset{+}{\leftarrow} 1, \; m_i(s_k^t,a_k^t) \overset{+}{\leftarrow} 1$.\footnote{We use $x \overset{+}{\leftarrow} y$ as a shorthand for the increment operation $x \leftarrow x + y$.}  \\
		
		\If(\myComment{entering a new epoch}){$\exists k, \  m_i(s_k^t, a_k^t) \geq \max\{1, 2m_{i-1}(s_k^t,a_k^t)\}$}
		{
                Increment epoch index $i \overset{+}{\leftarrow} 1$ and set new epoch starting time $t_i = t+1$.  
		
		 Initialize new counters: $\forall (s,a,s')$,  
		$m_i(s,a,s') = m_{i-1}(s,a,s') , m_i(s,a) = m_{i-1}(s,a)$.
		
                Update empirical transition $\bar{P}_i$ and confidence width $B_i$ based on \pref{eq:empirical_mean_transition_def} and \pref{eq:confidence_width_def}. 
}
	}
\end{algorithm}
\end{savenotes}

\paragraph{Special regularizers for \ftrl}
Due to the new structure of our algorithm which uses a fixed transition $\bar{P}_i$ during epoch $i$, the design of the regularizers is basically the same as in the known transition case.
Specifically, in the bandit case, we use the same Tsallis entropy regularizer:
\begin{equation}\label{eq:Tsallis}
\phi_t(q) =  -\frac{1}{\eta_t} \sum_{s \neq s_L} \sum_{a\in A} \sqrt{ q(s,a) } + \beta  \sum_{s\neq s_L} \sum_{a\in A} \ln \frac{1}{q(s,a)},
\end{equation}
where $\eta_t = \nicefrac{1}{\sqrt{t - t_{i(t)}+1}}$ and $\beta = 128L^4$.
As discussed in \pref{sec:loss_shifting}, the small amount of log-barrier in the second part of \pref{eq:Tsallis} is used to stabilize the algorithm, similarly to~\citep{jin2020simultaneously}.

In the full-information case, while we can still use \pref{eq:Tsallis} since the bandit setting is only more difficult, this leads to extra dependence on some parameters.
Instead, we use the following Shannon entropy regularizer:
\begin{equation}\label{eq:Shannon}
\phi_t(q) =  \frac{1}{\eta_t} \sum_{s\neq s_L} \sum_{a\in A} q(s,a) \cdot \ln q(s,a).
\end{equation}
Although this is a standard choice for the full-information setting, the tuning of the learning rate $\eta_t$ requires some careful thoughts.
In the special case of MDPs with one layer (known as the expert problem~\citep{freund1997decision}), it has been shown that choosing $\eta_t$ to be of order $1/\sqrt{t}$ ensures best-of-both-worlds~\citep{mourtada2019optimality, amir2020prediction}.
However, in our general case, due to the use of the loss-shifting trick, we need to use the following data-dependent tuning (with $i$ denoting $i(t)$ for simplicity): $\eta_{t} =  \sqrt{\frac{L\ln(|S||A|)}{64L^5\ln(|S||A|) + M_t }}$ where
\begin{equation*}
M_t = \sum_{\tau=t_{i}}^{t-1}\min\cbr{  \sum_{s\neq s_L}\sum_{a\in A} \widehat{q}_\tau(s,a) \hatl_\tau(s,a)^2, \sum_{s\neq s_L}\sum_{a\in A} \widehat{q}_\tau(s,a) \rbr{\widehat{Q}_\tau(s,a) - \widehat{V}_\tau(s)}^2} ,
\end{equation*}
and similar to the discussion in \pref{sec:loss_shifting}, $\widehat{Q}_\tau$ and $\widehat{V}_\tau$ are the state-action and state value functions with respect to the transition $\bar{P}_{i}$, the adjusted loss function $\hatl_\tau$, and the policy $\pi_\tau$, that is:
$\widehat{Q}_\tau(s,a) = \hatl_\tau(s,a) + \E_{s' \sim \bar{P}_{i}(\cdot|s,a)}[\widehat{V}_\tau(s')] $ and $\widehat{V}_\tau(s) = \E_{a \sim \pi_\tau(\cdot|s)}[\widehat{Q}_\tau(s,a)]$ (with $\widehat{V}_\tau(s_L) = 0$).
This particular tuning makes sure that \ftrl enjoys some adaptive regret bound with a self-bounding property akin to \pref{eq:known_self_bounding_regret_bounding},
which is again the key to ensure polylogarithmic regret in the stochastic world.
This concludes all the algorithm design;
see \pref{alg:bobw_framework} again for the complete pseudocode.

\subsection{Main Best-of-both-worlds Results}\label{sec:results}
We now present our main best-of-both-worlds results.
As mentioned, proving $\sqrt{T}$-regret in the adversarial world is relatively straightforward.
However, proving polylogarithmic regret bounds for the stochastic world is much more challenging due to the transition estimation error, which is usually of order $\sqrt{T}$.
Fortunately, we are able to develop a new analysis that upper bounds some transition estimation related terms by the regret itself, establishing a self-bounding property again. 
We defer the proof sketch to \pref{sec:analysis}, and state the main results in the following theorems.\footnote{%
For simplicity, for bounds in the stochastic world, we omit some $\otil(1)$ terms that are independent of the gap function, but they can be found in the full proof.
}

\begin{theorem}\label{thm:main_full_info}
In the full-information setting, \pref{alg:bobw_framework} with $\delta = \frac{1}{T^2}$ guarantees $\Reg_T(\optpi) = \otil\rbr{L|S|\sqrt{|A|T}}$ always, and simultaneously $\Reg_T(\pi^\star) = \order\rbr{U + \sqrt{UC}}$ 
under Condition~\eqref{eq:loss_condition}, where
$
	U = \order\Big(\frac{\rbr{L^6|S|^2+L^5|S||A|\log(|S||A|)}\log T}{\gapmin} + \sum_{s \neq s_L} \sum_{ a\neq \pi^{\star}(s)} \frac{L^6|S|\log T}{\gap(s,a)}\Big).
$
\end{theorem}
\begin{theorem}\label{thm:main_bandit}
In the bandit feedback setting,  \pref{alg:bobw_framework} with $\delta = \frac{1}{T^3}$ guarantees $\Reg_T(\optpi) = \otil\rbr{(L+\sqrt{|A|})|S|\sqrt{|A|T}}$ always, and simultaneously $\Reg_T(\pi^\star) = \order\rbr{U + \sqrt{UC}}$ under Condition~\eqref{eq:loss_condition}, where
$
	U = \order\Big(\frac{\rbr{L^6|S|^2+L^3|S|^2|A|}\log^2 T}{\gapmin} + \sum_{s \neq s_L} \sum_{ a\neq \pi^{\star}(s)} \frac{\rbr{L^6|S| + L^4|S||A|}\log^2 T}{\gap(s,a)}\Big).
$
\end{theorem} 

While our bounds have some extra dependence on the parameters $L$, $|S|$, and $|A|$ compared to the best existing bounds in each of the two worlds,
we emphasize that our algorithm is the first to be able to adapt to these two worlds simultaneously and achieve $\otil(\sqrt{T})$ and $\order(\polylog(T))$ regret respectively.
In fact, with some extra twists (such as treating differently the state-action pairs that are visited often enough and those that are not), we can improve the dependence on these parameters,
but we omit these details since they make the algorithms much more complicated.

Also, while~\citep{jin2020simultaneously} is able to obtain $\order(\log T)$ regret for the stronger benchmark $\Reg_T(\optpi)$ under Condition~\eqref{eq:loss_condition} and known transition (same as our \pref{thm:known_transition_main}),
here we only achieve so for $\Reg_T(\pi^\star)$ due to some technical difficulty (see \pref{sec:analysis}).
However, recall that for the most interesting i.i.d. case, one simply has $\Reg_T(\pi^\star) = \Reg_T(\optpi)$ as discussed in \pref{sec:prelim};
even for the corrupted i.i.d. case, since $\Reg_T(\optpi)$ is at most $C + \Reg_T(\pi^\star)$, our algorithms ensure $\Reg_T(\optpi) = \order(U + C)$ (note $\sqrt{UC} \leq  U + C$).
Therefore, our bounds on $\Reg_T(\pi^\star)$ are meaningful and strong.

\section{Analysis Sketch}\label{sec:analysis}

In this section, we provide a proof sketch for the full-information setting (which is simpler but enough to illustrate our key ideas).
The complete proofs can be found in \pref{app:bobw_unknown_transition_fullinfo} (full-information) and \pref{app:bobw_unknown_transition_bandit} (bandit).
We start with the following straightforward regret decomposition:
\begin{equation}
\begin{aligned}
\Reg_T(\pi) = 
\E\Bigg[\underbrace{\sum_{t=1}^{T} V^{\pi_t}_t(s_0) - \widehat{V}^{\pi_t}_t(s_0)}_{\textsc{Err}_1 } + \underbrace{\sum_{t=1}^{T} \widehat{V}^{\pi_t}_t(s_0) - \widehat{V}^{\pi}_t(s_0) }_{\textsc{EstReg} } + \underbrace{\sum_{t=1}^{T}\widehat{V}^{\pi}_t(s_0) - V^{\pi}_t(s_0) }_{\textsc{Err}_2}\Bigg]
\end{aligned}
\label{eq:regret_basic_decomp}
\end{equation}
for an arbitrary benchmark $\pi$, where $V_t^\pi$ is the state value function associated with the true transition $P$, the true loss $\ell_t$, and policy $\pi$, while $\widehat{V}_t^\pi$ is the state value function associated with the empirical transition $\bar{P}_{i(t)}$, the adjusted loss $\hatl_t$, and policy $\pi$.
Define the corresponding state-action value functions $Q_t^\pi$ and $\widehat{Q}_t^\pi$ similarly
(our earlier notations $\widehat{V}_t$ and $\widehat{Q}_t$ are thus shorthands for $\widehat{V}_t^{\pi_t}$ and $\widehat{Q}_t^{\pi_t}$). 

In the adversarial world, we bound each of the three terms in \pref{eq:regret_basic_decomp} as follows (see \pref{prop:bobw_fullinfo_adv_lem} for details).
First, $\E\sbr{\textsc{Err}_1}$ measures the estimation error of the loss of the learner's policy $\pi_t$, which can be bounded by $\otil(L|S|\sqrt{|A|T})$ following the analysis of~\cite{jin2019learning}.
Second, as mentioned, our adjusted losses are optimistic in the sense that it underestimates the loss of all policies (with high probability), making $\E\sbr{\textsc{Err}_2}$ an $\order\rbr{ 1 }$ term only.
Finally, $\E\sbr{\textsc{EstReg}}$ is the regret measured with $\bar{P}_{i(t)}$ and $\hatl_t$, which is controlled by the \ftrl procedure and of order $\otil(L\sqrt{|S||A| T})$.
Put together, this proves the $\otil(L|S|\sqrt{|A|T})$ regret shown in \pref{thm:main_full_info}.

In the stochastic world, we fix the benchmark $\pi = \pi^\star$. 
To obtain polylogarithmic regret, an important observation is that we now have to make use of the potentially negative term $\textsc{Err}_2$ instead of simply bounding it by $\order\rbr{1}$ (in expectation).
Specifically, inspired by~\citep{simc2019}, 
we propose a new decomposition on $\textsc{Err}_1$ and $\textsc{Err}_2$ \textit{jointly} as follows (see \pref{app:general_decomp_lem}):
$\textsc{Err}_1 + \textsc{Err}_2 = \textsc{ErrSub} + \textsc{ErrOpt} + \textsc{OccDiff} + \textsc{Bias}$. Here,
\begin{itemize}[leftmargin=2em]
  \setlength\itemsep{0.3em}
\item 
$\textsc{ErrSub} = \sum_{t=1}^{T}\sum_{s\neq s_L} \sum_{a \neq \pi^{\star}(s)} q_t(s,a)\widehat{E}^{\pi^{\star}}_t(s,a)$
measures some estimation error contributed by the suboptimal actions,
where $\widehat{E}^{\pi^{\star}}_t(s,a)=\ell_t(s,a) + \E_{s'\sim P(\cdot|s,a)}\big[\widehat{V}^{\pi^{\star}}_t(s')\big]  - \widehat{Q}^{\pi^{\star}}_t(s,a)$ is a ``surplus'' function (a term taken from~\citep{simc2019});

\item
$\textsc{ErrOpt} = \sum_{t=1}^{T}\sum_{s\neq s_L} \sum_{a = \pi^{\star}(s) } \rbr{ q_t(s,a) - q_t^\star(s,a) } \widehat{E}^{\pi^{\star}}_t(s,a)$
measures some estimation error contributed by the optimal action,
where $q_t^\star(s,a)$ is the probability of visiting a trajectory of the form $(s_0, \pi^\star(s_0)), (s_1, \pi^\star(s_1)), \ldots, (s_{k(s)-1}, \pi^\star(s_{k(s)-1})), (s,a)$ when
executing policy $\pi_t$;

\item
$\textsc{OccDiff} = \sum_{t=1}^{T}\sum_{s\neq s_L} \sum_{a \in A} \rbr{ q_t(s,a) - \widehat{q}_t(s,a)} \rbr{\widehat{Q}^{\pi^{\star}}_t(s,a) -\widehat{V}^{\pi^{\star}}_t(s)}$ measures the occupancy measure difference between $q_t$ and $\widehat{q}_t$;

\item
$\textsc{Bias} = \sum_{t=1}^{T}\sum_{s\neq s_L} \sum_{a\neq \pi^{\star}(s)}  q^\star_t(s,a) \rbr{\widehat{V}^{\pi^{\star}}_t(s) - V^{\pi^\star}_t(s)}$ measures some estimation error for $\pi^\star$, which, similar to $\textsc{Err}_2$, is of order $\order(1)$ in expectation due to optimism. 
\end{itemize}

The next key step is to show that the terms $\textsc{ErrSub}, \textsc{ErrOpt}, \textsc{OccDiff}$, and $\textsc{EstReg}$ can all be upper bounded by some quantities that admit a certain self-bounding property similarly to the right-hand side of \pref{eq:known_self_bounding_regret_bounding}.
We identify four such quantities and present them using functions $\selfterm_1$, $\selfterm_2$, $\selfterm_3$, and $\selfterm_4$, whose definitions are deferred to \pref{app:self_bounding_terms} due to space limit.
Combining these bounds for each term, we obtain the following important lemma.

\begin{lemma}\label{lem:bobw_fullinfo_stoc_self_bound_terms_lem}
	With $\delta = \frac{1}{T^2}$, \pref{alg:bobw_framework} ensures that $\Reg_T(\pi^\star)$ is at most $\order(L^4 |S|^3|A|^2  \ln^2 T)$ plus:
	\begin{align*} 
	\E\Bigg[ \order\Bigg( \underbrace{\selfterm_1\rbr{ L^4|S| \ln T}}_{\text{from $\textsc{ErrSub}$}}  + \underbrace{\selfterm_2\rbr{ L^4|S| \ln T  } }_{\text{from $\textsc{ErrOpt}$}} +   \underbrace{\selfterm_3\rbr{ L^4 \ln T }}_{\text{from $\textsc{OccDiff}$}}  +  \underbrace{\selfterm_4\rbr{ L^5 |S||A| \ln T \ln(|S||A|) } }_{\text{from $\textsc{EstReg}$} }  \Bigg)\Bigg].
	\end{align*}
\end{lemma}

Finally, as mentioned, each of the $\selfterm_1$, $\selfterm_2$, $\selfterm_3$, and $\selfterm_4$ functions can be shown to admit the following self-bounding property,
such that similarly to what we argue in \pref{eq:self_bounding_argument}, picking the optimal values of $\alpha$ and $\beta$ and rearranging leads to the polylogarithmic regret bound shown in \pref{thm:main_full_info}.

\begin{lemma}[Self-bounding property] \label{lem:main_text_self_bounding}
	Under Condition~\eqref{eq:loss_condition}, we have for any $\alpha,\beta \in (0,1)$,
	\begin{align*}
	\E\sbr{{\selfterm_1}(\cons)} & \leq \alpha\cdot \rbr{ \Reg_T(\pi^\star) +  C} + \textstyle \order\rbr{ \frac{1}{\alpha} \cdot  \sum_{s\neq s_L} \sum_{a \neq \pi^{\star}(s)} \frac{\cons}{\gap(s,a)} }, \\ 
	\E\sbr{{\selfterm_2}(\cons)} & \leq \beta \cdot \rbr{ \Reg_T(\pi^\star) +  C} + \textstyle \order\rbr{ \frac{1}{\beta} \cdot \frac{L|S| \cons}{\gapmin} }, \\ 
	\E\sbr{{\selfterm_3}(\cons)} & \leq \rbr{ \alpha + \beta } \cdot \rbr{ \Reg_T(\pi^\star)  +  C} + \textstyle \order\rbr{ \frac{1}{\alpha} \cdot \sum_{s\neq s_L} \sum_{a \neq \pi^{\star}(s)} \frac{L^2|S|\cons}{\gap(s,a)}} + \order\rbr{\frac{1}{\beta} \cdot \frac{L^2|S|^2\cons}{\gapmin} }, \\ 
	\E\sbr{{\selfterm_4}(\cons)} & \leq \beta\cdot \rbr{ \Reg_T(\pi^\star) +  C} + \textstyle\order\rbr{ \frac{1}{\beta} \cdot \frac{\cons}{\gapmin} }.
	\end{align*} 
\end{lemma}

We emphasize again that the proposed joint decomposition on $\textsc{Err}_1 + \textsc{Err}_2$ plays a crucial rule in this analysis and addresses the key challenge on how to bound the transition estimation error by something better than $\sqrt{T}$.
We also point out that in this analysis, only $\textsc{EstReg}$ is related to the \ftrl procedure, while the other three terms are purely based on our new framework to handle unknown transition. 
In fact, the reason that we can only derive a $\polylog(T)$ bound on $\Reg_T(\pi^\star)$ but not directly on $\Reg_T(\optpi)$ is also due to these three terms --- they can be related to the right-hand side of Condition~\eqref{eq:loss_condition} only when we use the benchmark $\pi = \pi^\star$ but not when $\pi = \optpi$.
This is not the case for $\textsc{EstReg}$, which is the reason why~\citet{jin2020simultaneously} are able to derive a bound on $\Reg_T(\optpi)$ directly when the transition is known.
Whether this issue can be addressed is left as a future direction.

\section{Conclusions}
In this work, we propose an algorithm for learning episodic MDPs which achieves favorable regret guarantees simultaneously in the stochastic and adversarial worlds with unknown transition. 
We start from the known transition setting and propose a loss-shifting trick for $\ftrl$ applied to MDPs, which simplifies the method of \cite{jin2020simultaneously} and improves their results. 
Then, we design a new framework to extend our known transition algorithm to the unknown transition case, which is critical for the application of the loss-shifting trick. 
Finally, we develop a novel analysis which carefully upper bounds the transition estimation error by (a fraction of) the regret itself plus a gap-dependent poly-logarithmic term in the stochastic setting, resulting in our final best-of-both-worlds result.

Besides the open questions discussed earlier (such as improving our bounds in \pref{thm:main_full_info} and \pref{thm:main_bandit}), one other key future direction is to remove the assumption that there exists a unique optimal action for each state, which appears to be challenging despite the recent progress for the bandit case~\citep{ito2021parameter}, since the occupancy measure computed from \pref{eq:epoch_FTRL} has a very complicated structure.
Another interesting direction would be to extend the sub-optimality gap function to other fine-grained gap functions, such as that of \cite{dann2021beyond}. 

\begin{ack}
HL is supported by NSF Award IIS-1943607 and a Google Faculty Research Award. 
LH is  supported  in  part  by  the  Technology  and Innovation  Major  Project  of  the  Ministry  of  Science  and Technology  of  China  under  Grants 2020AAA0108400  and 2020AAA0108403.
We thank Max Simchowitz for many helpful discussions, and the anonymous reviewers for their valuable feedback and suggestions. 
\end{ack}

\bibliography{ref}
\bibliographystyle{plainnat}

\appendix
\newpage
\tableofcontents
\newpage 

\paragraph{An important convention}
Note that the value of $m_i(s,a)$ is changing in the algorithm.
For the entire analysis, we see $m_i(s,a)$ as its initial value, which is the number of visits to $(s,a)$ from epoch $1$ to epoch $i-1$.
In this sense, if we let $N$ be the total number of epochs, then $m_{N+1}(s,a)$ is naturally defined as the total number of visits to $(s,a)$ within $T$ episodes.

\section{Best of Both Worlds for MDPs with Known Transition} 
\label{app:bobw_known_transition}


In this section, we show how to extend the loss-shifting technique to MDPs with known transition and obtain best-of-both-worlds results.

\subsection{Loss-shifting Technique} 
\label{app:loss_shifting_technique}

First of all, we introduce a general invariant condition with a fixed transition in \pref{lem:loss_shifting_invariant}
\begin{lemma}
	\label{lem:loss_shifting_invariant}
	Fix the transition function $P$. For any policy $\pi$ and loss function $\mathring \ell: S\times A \rightarrow \fR$, define invariant function $g \in S \times A \rightarrow \fR$ as:
	\begin{equation}
	g^{P, \pi,\mathring{\ell} }(s,a) \triangleq \rbr{ Q^{P, \pi,\mathring{\ell} }(s,a) - V^{P, \pi,\mathring{\ell} }(s) - {\mathring  \ell}(s,a)}, \label{eq:invariant_func}
	\end{equation}
	where $Q^{P, \pi,\mathring{\ell} }$ and $V^{P, \pi,\mathring{\ell} }$ are state-action value and state value functions associated with $\mathring{\ell}$ and the fixed policy $\pi$. 
	Then, it holds for any policy $\pi'$ that 
	\[
	\inner{q^{P, \pi'}, g^{P, \pi,\mathring{\ell} }} \triangleq \sum_{s\neq s_L} \sum_{a\in A} q^{P, \pi'}(s,a) \cdot g^{P, \pi,\mathring{\ell} }(s,a)  =  - V^{P, \pi,\mathring  \ell }(s_0)
	\]
	where $V^{P, \pi,\mathring  \ell }(s_0)$ only depends on $\pi$ and $\mathring  \ell$ (but not $\pi'$). 
\end{lemma}
\begin{proof} 
For notational convenience, we drop the superscripts for fixed transition $P$ and loss function $\mathring{\ell}$. 
By the standard performance difference lemma~\citep[Theorem~5.2.1]{kakade2003sample}, it holds for any policy $\pi'$ that 
\begin{equation}
V^{\pi' }(s_0) - V^{\pi }(s_0)  = \sum_{s\neq s_L} \sum_{a\in A}  q^{\pi'}(s,a) \rbr{ Q^{\pi}(s,a) - V^{\pi}(s)   }. \label{eq:loss_shfifting_eq1} 
\end{equation}

On the other hand, it also holds that 
\begin{equation}
V^{\pi'}(s_0)  = \sum_{s\neq s_L} \sum_{a \in A}  q^{\pi'}(s,a) \mathring{\ell}(s,a). \label{eq:loss_shfifting_eq2} 
\end{equation}

Therefore, subtracting $V^{\pi' }(s_0)$ from \pref{eq:loss_shfifting_eq1} yields that 
\[
- V^{\pi}(s_0) = \sum_{s \neq s_L}\sum_{a \in A}  q^{\pi'}(s,a) \rbr{ Q^{\pi}(s,a) - V^{\pi}(s) - \mathring{\ell}(s,a)  }
\]
which completes the proof after putting back the superscripts for $P$ and $\mathring{\ell}$. 
\end{proof}

As discussed in \pref{sec:loss_shifting}, the invariant function $g^{P, \pi,\mathring{\ell} }$ defined in \pref{eq:invariant_func} allows us to treat \ftrl as dealing with a hypothesized loss sequence, as restated below.

\begin{corollary}
	\label{col:invariant_with_ftrl}
	Consider the selected occupancy measure $\whatq_t$ via \ftrl with respect to a regularizer $\phi_t(\cdot)$ and loss sequence  $\{\hatl_\tau\}_{\tau < t}$ (on the decision set $\Omega(\bar{P})$), then it holds that 
	\begin{align*}
	\whatq_t = \argmin_{q \in \Omega(\bar{P})} \inner{q, \sum_{\tau < t} \hatl_\tau} + \phi_t(q) = \argmin_{q \in \Omega(\bar{P})} \inner{q, \sum_{\tau < t} (\hatl_\tau + g_\tau)} + \phi_t(q).
	\end{align*}
	for any invariant function sequence $\{g_\tau\}_{\tau < t}$ which are constructed with hypothesized losses $\{\mathring{\ell}_\tau\}_{\tau < t}$ and policies $\{\pi'_\tau\}_{\tau < t}$. 
\end{corollary}

\begin{proof} By \pref{lem:loss_shifting_invariant}, one can verify that
\begin{align*}
\inner{q, \sum_{\tau < t} g_\tau} = - \sum_{\tau < t} V^{\bar{P},\pi'_\tau,\mathring{\ell}_\tau}(s_0)
\end{align*}
for any occupancy measure $q \in \Omega(\bar{P})$. 
Therefore, this term does not affect the optimization.
\end{proof}

Then, we consider the ``loss-shifting function'' defined in \pref{eq:loss_shift_function}, that is, constructing $g_t$ via the  loss estimator $\hatl_t$ and the policy $\pi_t$ selected at episode $t$. Importantly, in the known transition setting where $\whatq_t = q_t$, $\hatl_t$ is inverse propensity weighted estimator, in other words, $\hatl_t(s,a) = \nicefrac{\Indt{s,a} \ell_t(s,a)}{q_t(s,a)}$. More specifically, we have 
\begin{equation*}
	g_t(s,a) = \whatQ_t(s,a) -  \whatV_t(s) - \hatl_t(s,a),
\end{equation*}
where 
\[
\whatQ_t(s,a) = \hatl_t(s,a) + \sum_{s' \in S_{k(s)+1}} \bar{P}(s'|s,a)  \whatV_t(s), \quad
\whatV_t(s) = \sum_{a\in A} \pi_t(a|s)  \whatQ_t(s,a)
\]
(with $\whatV_t(s_L) = 0$).
Below we show several useful properties, which are key to achieve the best-of-both-worlds guarantee in the known transition setting. 
\begin{lemma}
	\label{lem:loss_shifting_loss_estimator} With $\bar{P} = P$ being the true transition function (therefore, $\whatq_t = q_t$),  we have
	\begin{itemize}
		\item $q_t(s,a) \whatQ_t(s,a) \leq L$,
		\item $q_t(s) \whatV_t(s) \leq L$, 
		\item $\E_t\sbr{ \rbr{ \whatQ_t(s,a) - \whatV_t(s) }^2 } \leq \frac{2L^2\rbr{ 1 - \pi_t(a|s)}}{q_t(s,a)},$
	\end{itemize} 
	for all state-action pairs $(s,a)$ (where $\E_t$ denotes the conditional expectation given everything before episode $t$).
\end{lemma}

\begin{proof}
Denote by $q_t(s',a'|s,a)$ the probability of visiting $(s',a')$ after taking action $a$ at state $s$ and following $\pi_t$ afterwards.
Then we have $\whatQ_t(s,a) = \sum_{k=k(s)}^{L-1} \sum_{s'\in S_k} \sum_{a'\in A} q_t(s',a'|s,a)\hatl_t(s',a')$.
Therefore, plugging in the definition of $\hatl_t(s,a)$, we verify the following:
\begin{align*}
q_t(s,a) \whatQ_t(s,a) &= \sum_{k=k(s)}^{L-1} \sum_{s'\in S_k} \sum_{a'\in A}\frac{q_t(s,a) q_t(s',a'|s,a)} {q_t(s',a')} \Indt{s',a'} \ell_t(s',a') \\
&\leq \sum_{k=k(s)}^{L-1} \sum_{s'\in S_k} \sum_{a'\in A}\Indt{s',a'} \leq L,
\end{align*}
where the inequality is by $q_t(s,a) q_t(s',a'|s,a) \leq q_t(s',a')$ and $\ell_t(s',a') \in [0,1]$.
This also proves $q_t(s) \whatV_t(s) \leq L$ using the definition of $\whatV_t(s)$.

To prove the last statement, we first note that 
\begin{equation}
\E_t\sbr{ \rbr{ \whatQ_t(s,a) - \whatV_t(s) }^2 } \leq 2 \E_t\sbr{ \rbr{ 1 - \pi_t(a|s)}^2 \whatQ_t(s,a)^2 + \rbr{\sum_{b\neq a} \pi_t(b|s) \whatQ_t(s,b)  }^2 } \label{eq:loss_shifting_loss_estimator_decomp}
\end{equation}
by the fact $\rbr{x-y}^2 \leq 2x^2 + 2y^2$ for all $x,y \in \fR$.  

For the first term in \pref{eq:loss_shifting_loss_estimator_decomp}, we have:
\begin{align*}
\E_t\sbr{ \whatQ_t(s,a)^2 } & = \E_t\sbr{ \rbr{\sum_{k=k(s)}^{L-1} \sum_{s'\in S_k} \sum_{a'\in A}\frac{q_t(s',a'|s,a)} {q_t(s',a')} \Indt{s',a'} \ell_t(s',a') }^2 } \\
& \leq L \cdot \E_t\sbr{ \sum_{k=k(s)}^{L-1} \rbr{ \sum_{s'\in S_k} \sum_{a'\in A}\frac{q_t(s',a'|s,a)} {q_t(s',a')} \Indt{s',a'} \ell_t(s',a') }^2 } \\
& \leq L \cdot \E_t\sbr{ \sum_{k=k(s)}^{L-1}  \sum_{s'\in S_k} \sum_{a'\in A} \frac{q_t(s',a'|s,a)^2}{q_t(s',a')^2} \Indt{s',a'}  } \\ 
& = L \cdot  \sum_{k=k(s)}^{L-1}  \sum_{s'\in S_k} \sum_{a'\in A} \frac{q_t(s',a'|s,a)^2 }{q_t(s',a')}    \\ 
& = \frac{L}{q_t(s,a)} \cdot  \sum_{k=k(s)}^{L-1}  \sum_{s'\in S_k} \sum_{a'\in A} \frac{q_t(s,a) q_t(s',a'|s,a)}{q_t(s',a')} \cdot q_t(s',a'|s,a)  \\
& \leq \frac{L}{q_t(s,a)} \cdot  \sum_{k=k(s)}^{L-1}  \sum_{s'\in S_k} \sum_{a'\in A}  q_t(s',a'|s,a)  \leq \frac{L^2}{q_t(s,a)},
\end{align*}
where the second line uses the Cauchy-Schwartz inequality; the third line follows from the fact $\Indt{s,a} \Indt{s',a'} = 0$ for all $(s,a),(s',a') \in S_k \times A$ such that $(s,a) \neq (s',a')$;
the fourth line uses $\E_t[\Indt{s',a'}] = q_t(s',a')$;
and the last line follows from the fact $q_t(s,a) q_t(s',a'|s,a) \leq q_t(s',a')$. 

Repeating the similar arguments, we bound the second term as 
\begin{align*}
& \E_t\sbr{ \rbr{ \sum_{b\neq a} \pi_t(b|s) \whatQ_t(s,b)}^2 } 
 = \E_t\sbr{ \rbr{\sum_{k=k(s)}^{L-1} \sum_{s'\in S_k} \sum_{a'\in A}  \rbr{ \sum_{b\neq a} \pi_t(b|s)  q_t(s',a'|s,b) }  \hatl_t(s',a') }^2 } \\
&\leq L \cdot \E_t\sbr{ \sum_{k=k(s)}^{L-1} \rbr{\sum_{s'\in S_k} \sum_{a'\in A}  \rbr{ \sum_{b\neq a} \pi_t(b|s)  q_t(s',a'|s,b) }  \hatl_t(s',a') }^2 } \tag{Cauchy-Schwarz inequality} \\
&\leq L \cdot \E_t\sbr{ \sum_{k=k(s)}^{L-1} \sum_{s'\in S_k} \sum_{a'\in A}  \rbr{ \sum_{b\neq a} \pi_t(b|s)  q_t(s',a'|s,b) }^2 \frac{\Indt{s',a'}}{q_t(s',a')^2}   } \tag{$\Indt{s,a} \Indt{s',a'} = 0$ for $(s,a) \neq (s',a')$}\\
& = L \cdot \sum_{k=k(s)}^{L-1} \sum_{s'\in S_k} \sum_{a'\in A}  \rbr{ \frac{\sum_{b\neq a} \pi_t(b|s) q_t(s',a'|s,b) }{q_t(s',a')}} \cdot \rbr{ \sum_{b\neq a} \pi_t(b|s) \cdot q_t(s',a'|s,b) } \\
& = \frac{L}{q_t(s)} \cdot \sum_{k=k(s)}^{L-1} \sum_{s'\in S_k} \sum_{a'\in A}  \rbr{ \frac{\sum_{b\neq a} q_t(s,b) q_t(s',a'|s,b) }{q_t(s',a')}} \cdot \rbr{ \sum_{b\neq a} \pi_t(b|s) \cdot q_t(s',a'|s,b) } \\
&\leq \frac{L}{q_t(s)} \cdot \sum_{k=k(s)}^{L-1} \sum_{s'\in S_k} \sum_{a'\in A}  \rbr{ \sum_{b\neq a} \pi_t(b|s) \cdot q_t(s',a'|s,b) } \\
& = \frac{L}{q_t(s)} \cdot  \sum_{b\neq a} \pi_t(b|s) \cdot \rbr{ \sum_{k=k(s)}^{L-1}\rbr{ \sum_{s'\in S_k} \sum_{a'\in A}   q_t(s',a'|s,b) } }\\
& \leq \frac{L^2}{q_t(s)} \cdot  \sum_{b\neq a} \pi_t(b|s) = \frac{L^2\rbr{ 1 - \pi_t(a|s) } }{q_t(s)}.
\end{align*}


Plugging these bounds into \pref{eq:loss_shifting_loss_estimator_decomp} concludes the proof:
\begin{align*}
\E_t\sbr{ \rbr{ \whatQ_t(s,a) - \whatV_t(s) }^2 } & \leq 2L^2\rbr{\frac{\rbr{ 1 - \pi_t(a|s)}^2}{q_t(s,a)} + \frac{1-\pi_t(a|s)}{q_t(s) } } \\
& = 2L^2\rbr{ 1 - \pi_t(a|s)}\rbr{\frac{1 - \pi_t(a|s)}{q_t(s,a)} + \frac{1}{q_t(s) } } = \frac{2L^2\rbr{ 1 - \pi_t(a|s)}}{q_t(s,a)}.
\end{align*}
\end{proof}

\subsection{Known Transition and Full-information Feedback: $\ftrl$ with Shannon Entropy}
Although not mentioned in the main text, in this section, we discuss a simple application of the loss-shifting technique: achieving the best-of-both-worlds in the full-information feedback setting with known transition via the $\ftrl$ framework with the Shannon entropy regularizer.
Some of the lemmas in this section are useful for proving similar results for the unknown transition case in \pref{app:bobw_unknown_transition_fullinfo}.

Therefore, the specific state-action and state value functions defined in \pref{lem:loss_shifting_loss_estimator} are now constructed based on the received loss vector $\ell_t$, instead of the loss estimator $\hatl_t$. In other words, the loss-shifting function $g_t$ is defined as $g_t(s,a) = \whatQ(s,a) - \whatV(s) - \ell_t(s,a)$ where 
\begin{equation}
\whatQ_t(s,a) = \ell_t(s,a) + \sum_{s' \in S_{k(s)+1}} P(s'|s,a)  \whatV_t(s), \quad \whatV_t(s) = \sum_{a\in A} \pi_t(a|s)  \whatQ_t(s,a). \label{eq:loss_shifting_function_fullinfo_known}
\end{equation}

Our goal is to show that, using an adaptive time-varying learning rate schedule, $\ftrl$ with Shannon entropy is able to attain a self-bounding regret guarantee with full-information feedback. This idea will be further discussed in \pref{app:bobw_unknown_transition_fullinfo}  to address the unknown transition setting.

In particular, the algorithm uses following regularizer for episode $t$:
\begin{equation}
\phi_t(q) = \frac{1}{\eta_t} \sum_{s \neq s_L}\sum_{a\in A} q(s,a) \ln q(s,a) = \frac{1}{\eta_t} \phi(q), \label{eq:shannon_reg_def}
\end{equation}
where the adaptive learning rate $\eta_t$ is defined as $\eta_t= \sqrt{ \frac{ L \ln (|S||A|) }{ M_{t-1} + 64L^3\ln(|S||A|)  } }$ with
\begin{align*}
M_t = \sum_{\tau =1}^{t} \min\cbr{  \sum_{s \neq s_L} \sum_{a \in A} q_\tau(s,a) \rbr{ \whatQ_\tau(s,a) - \whatV_\tau(s)  }^2, \sum_{s \neq s_L} \sum_{a \in A} q_\tau(s,a) \ell_\tau(s,a)^2 }.
\end{align*}
The pseudocode of our algorithm is presented in \pref{alg:known_transition_shannon_entropy_algoritm}. 

\begin{algorithm}[t]
	\caption{Best-of-both-worlds for MDPs with Known Transition and Full-information Feedback}
	\label{alg:known_transition_shannon_entropy_algoritm}
	\begin{algorithmic}
		\FOR{$t=1$ {\bfseries to} $T$}
		\STATE Compute $q_t = \argmin_{q\in\Omega(P)} \big\langle q, \sum_{\tau < t} \ell_\tau \big\rangle + \phi_t(q)$ where $\phi_t(q)$ is defined in \pref{eq:shannon_reg_def}.
		\STATE Execute policy $\pi_t$ where $\pi_t(a | s) = q_t(s,a)/q_t(s)$.
		\STATE Observe the entire loss function $\ell_t$. 
		\ENDFOR
	\end{algorithmic}
\end{algorithm}

In the known transition setting, we assume the loss functions satisfy a more general condition compared to Condition~\eqref{eq:loss_condition}:  there exists a deterministic policy $\pi^\star: S \rightarrow A$, a gap function $\gap: S \times A \rightarrow \fR_{+}$ and a constant $C>0$ such that 
\begin{equation}
\Reg_T(\optpi) \geq \E\sbr{ \sum_{t=1}^{T} \sum_{s \neq s_L}\sum_{a \neq \pi^\star(s)} q_t(s,a) \gap(s,a) - C}. \label{eq:loss_condition_known_transition}
\end{equation}
Note that this is only weaker than Condition~\eqref{eq:loss_condition} since $\Reg_T(\optpi) \geq \Reg_T(\pi^\star)$.

Then, we show that \pref{alg:known_transition_shannon_entropy_algoritm} ensures a worst-case guarantee $\Reg_T(\optpi) = \otil(L\sqrt{T})$, and simultaneously an adaptive regret bound which further leads to logarithmic regret under Condition~\eqref{eq:loss_condition_known_transition} (\pref{col:known_transition_shannon_entropy_main}). Importantly, the worst-case regret bound matches the lower bound of learning MDPs with known transition and full-information feedback~\citep{zimin2013}. 
\begin{theorem} 
	\label{thm:known_transition_shannon_entropy_main}
	\pref{alg:known_transition_shannon_entropy_algoritm} ensures
	that $\Reg_T(\optpi)$ is bounded by
	\begin{equation}\label{eq:known_shannon_main}
	\order\rbr{  \sqrt{ \min\cbr{L^2T, L^3\E\sbr{\sum_{t=1}^{T} \sum_{s \neq s_L} \sum_{a\neq \pi(s)} q_t(s,a)}} \ln(|S||A|) } + L^2\ln(|S||A|)  } 
	\end{equation}	
	for any mapping $\pi: S \rightarrow A$. 
\end{theorem} 
\begin{proof} 
Due to the invariant property (that $\inner{q, g_t}$ is independent of $q \in \Omega(P)$),
we can apply \pref{lem:shannon_entropy_reg_bound} with $\hatl_t$ being either $\ell_t$ or $\ell_t + g_t$ for any $t$ --- note that the condition $\eta_t\hatl_t(s,a) \geq -1$ is always satisfied since $\ell_t(s,a) \in [0,1]$ and $\whatQ_t(s,a) - \whatV_t(s) \in [-L,L]$.
 Therefore, we have for any $u\in \Omega(P)$,
\begin{align*}
\sum_{t=1}^T\inner{q_t - u, \ell_t} & \leq 
 \frac{ L\ln(|S||A|)}{\eta_{T+1}} \\
 +  & \sum_{t=1}^{T} \eta_t \min\cbr{ \sum_{s \neq s_L} \sum_{a\in A} q_t(s,a) \rbr{ \whatQ_t(s,a) - \whatV_t(s)  }^2,  \sum_{s \neq s_L} \sum_{a \in A} q_t(s,a) \ell_t(s,a)^2}  \\
& =  \frac{ L\ln(|S||A|)}{\eta_{T+1}}  + \sum_{t=1}^{T} \eta_t \rbr{ M_t - M_{t-1}}, \tag{definition of $M_t$} \\
& = \frac{ L\ln(|S||A|)}{\eta_{T+1}}  + \sum_{t=1}^{T} \eta_t \rbr{ \sqrt{M_t} + \sqrt{M_{t-1}}}\rbr{\sqrt{M_t} - \sqrt{M_{t-1}}}, \\
& \leq \frac{ L\ln(|S||A|)}{\eta_{T+1}}  + 2\sum_{t=1}^{T} \eta_t \sqrt{M_{t-1} + L}\rbr{\sqrt{M_t} - \sqrt{M_{t-1}}}. \tag{$M_t \leq M_{t-1} + L$}
\end{align*}
Further plugging in the definition of $\eta_t$ and taking expectation, we arrive at
\begin{align*} 
\Reg_T(\optpi)
&\leq \E\sbr{\frac{ L\ln(|S||A|)}{\eta_{T+1}}  + 2\sqrt{  L \ln (|S||A|) } \sum_{t=1}^T \rbr{\sqrt{M_t} - \sqrt{M_{t-1}}}} \\
&= \E\sbr{\sqrt{ L \ln (|S||A|)\rbr{M_T + 64 L^3 \ln |S||A|}} + 2\sqrt{L M_T \ln(|S||A|) }} \\
&= \order\rbr{\sqrt{L \E\sbr{M_T} \ln(|S||A|)} + L^2\ln(|S||A| )  }.
\end{align*}

It remains to bound $M_T$. First, we note that 
\begin{align*}
M_T & = \sum_{t=1}^{T}\min\cbr{  \sum_{s \neq s_L} \sum_{a \in A} q_t(s,a) \rbr{ \whatQ_t(s,a) - \whatV_t(s)  }^2,  \sum_{s \neq s_L} \sum_{a \in A} q_t(s,a) \ell_t(s,a)^2 }\\
	& \leq \min\cbr{  \sum_{t=1}^{T}\sum_{s \neq s_L} \sum_{a \in A} q_t(s,a) \rbr{ \whatQ_t(s,a) - \whatV_t(s)  }^2,  \sum_{t=1}^{T}\sum_{s \neq s_L} \sum_{a \in A} q_t(s,a) \ell_t(s,a)^2 }\\
	& \leq \min\cbr{  \sum_{t=1}^{T}\sum_{s \neq s_L} \sum_{a \in A} q_t(s,a) \rbr{ \whatQ_t(s,a) - \whatV_t(s)  }^2, L T}.
\end{align*}
where the second line follows from the fact $\min\cbr{a,b} + \min\cbr{c,d}\leq \min\cbr{a+c,b+d}$, and the third line uses the property $0\leq \ell_t(s,a) \leq 1$ for all state-action pairs $(s,a)$. 

On the other hand, we have 
\begin{align}
\rbr{ \whatQ_t(s,a) - \whatV_t(s) }^2 & \leq 2 \sbr{ \rbr{ 1 - \pi_t(a|s) }^2 \whatQ_t(s,a)^2  + \rbr{ \sum_{b\neq a} \pi_t(b|s) \whatQ_t(s,b) }^2 } \nonumber\\
& \leq 2L^2 \cdot \sbr{ \rbr{ 1 - \pi_t(a|s) }^2 + \rbr{ 1 - \pi_t(a|s)}^2} \nonumber\\
& \leq 4L^2 \rbr{ 1 - \pi_t(a|s) }, \label{eq:step1}
\end{align}
where we use the facts $(a-b)^2 \leq 2(a^2+b^2)$ and $0\leq \whatQ_t(s,a)\leq L$ for all state-action pairs $(s,a)$. 
Therefore, we have for any mapping $\pi : S \rightarrow A$,
\begin{align}
&\sum_{t=1}^{T} \sum_{s\neq s_L} \sum_{a\in A} q_t(s,a) \rbr{ \whatQ_t(s,a) - \whatV_t(s) }^2 \nonumber\\
&  \leq 4L^2 \cdot \sum_{t=1}^{T} \sum_{s\neq s_L} \sum_{a\in A} q_t(s,a) \rbr{ 1 - \pi_t(a|s) } \nonumber\\ 
& \leq 4L^2 \cdot \sum_{t=1}^{T} \sum_{s\neq s_L} \rbr{ q_t(s) \cdot \rbr{ 1 - \pi_t(\pi(s)|s) } + \sum_{a\neq \pi(s)} q_t(s,a)  } \nonumber\\
& = 8L^2 \cdot \sum_{t=1}^{T} \sum_{s\neq s_L} \sum_{a\neq \pi(s)} q_t(s,a), 
\label{eq:step2}
\end{align}
which finishes the proof.
\end{proof}

\begin{corollary} \label{col:known_transition_shannon_entropy_main}
Suppose Condition~\eqref{eq:loss_condition_known_transition} holds. \pref{alg:known_transition_shannon_entropy_algoritm} guarantees that: 
\[
\Reg_T(\optpi) = \order\rbr{ U + \sqrt{CU}}, \text{ where } U = \frac{L^3\ln(|S||A|)}{\gapmin} .
\]
\end{corollary}
\begin{proof} By \pref{thm:known_transition_shannon_entropy_main}, $\Reg_T(\optpi)$ is bounded by 
\[
\kappa \cdot \rbr{  \sqrt{ L^3\ln(|S||A|) \cdot \E\sbr{ \sum_{t=1}^{T} \sum_{s \neq s_L} \sum_{a\neq \pi^\star(s)} q_t(s,a)  }   } + L^2\ln(|S||A|)     } 
\]
where $\kappa\geq 1$ is a universal constant, and $\pi^\star$ is the mapping specified in Condition~\eqref{eq:loss_condition_known_transition}. 

For any $z >1$, $\Reg_T(\optpi)$ is bounded by
\begin{align*}
& \kappa \sqrt{ L^3\ln(|S||A|) \cdot \E\sbr{ \sum_{t=1}^{T} \sum_{s \neq s_L} \sum_{a\neq \pi^\star(s)} q_t(s,a)  }   } + \kappa L^2\ln(|S||A|) \\
& = \sqrt{ \frac{z\kappa^2  L^3\ln(|S||A|) }{2\gapmin } \cdot \rbr{ \frac{2}{z} \cdot \E\sbr{ \sum_{t=1}^{T} \sum_{s \neq s_L} \sum_{a\neq \pi^\star(s)} q_t(s,a)\gapmin  }} } + \kappa L^2\ln(|S||A|) \\
& \leq \frac{ \Reg_T(\optpi) + C }{z}+ \frac{z\kappa^2  L^3\ln(|S||A|)}{4\gapmin} + \kappa L^2\ln(|S||A|)  \\
& \leq \frac{ \Reg_T(\optpi) + C }{z} + z \cdot 2\kappa^2 U,
\end{align*}
where the third line uses the AM-GM inequality and \pref{eq:loss_condition_known_transition}, and the last line uses the shorthand $U$ and the facts $\kappa,z >1$ and $\gapmin \leq 1$. 

Therefore, by defining $x= z - 1 > 0$, we can rearrange and arrive at 
\begin{align*}
\Reg_T(\optpi) & \leq \frac{C }{z-1}+ \frac{z^2}{z-1} \cdot 2\kappa^2 U   \\
& =\frac{C }{x}+ \frac{(x+1)^2}{x} \cdot \rbr{2\kappa^2 U }  \\
& = \frac{1}{x} \cdot \rbr{ C +  2\kappa^2 U }  + x \cdot \rbr{ 2\kappa^2 U} + 4\kappa^2 U,
\end{align*}
where we replace all $z$'s in the second line. Picking the optimal $x = \sqrt{\frac{C +  2\kappa^2 U}{2\kappa^2 U}}$ gives 
\begin{align*}
\Reg_T(\optpi)  & \leq 2 \sqrt{\rbr{ C +  2\kappa^2 U }  \cdot  \rbr{ 2\kappa^2 U}   } + 4\kappa^2 U\\
  & \leq 8\kappa^2 U + 2\sqrt{2}\kappa \cdot \sqrt{ C U} \\
  & = \order\rbr{  U + \sqrt{UC}}, 
\end{align*}
where the second line follows from the fact $\sqrt{x+y} \leq \sqrt{x} + \sqrt{y}$. 
\end{proof}

\begin{lemma} \label{lem:shannon_entropy_reg_bound}
Suppose $q_t = \argmin_{q\in\Omega(P)} \big\langle q, \sum_{\tau < t} \hatl_\tau \big\rangle + \phi_t(q)$, where $\phi_t(q) = \frac{1}{\eta_t}\phi(q)$ for some $\eta_t > 0$, $\phi(q) =  \sum_{s\neq s_L}\sum_{a\in A} q(s,a)\ln q(s,a)$, and $\eta_t \hatl_t(s,a) \geq -1$ holds for all $t$ and $(s,a)$.
Then
\begin{equation*}
\sum_{t=1}^T\inner{q_t - u, \hatl_t} \leq \frac{L\ln(|S||A|)}{\eta_{T+1}}  + \sum_{t=1}^{T} \eta_t \cdot \sum_{s\neq s_L} \sum_{a \in A} q_t(s,a) \hatl_t(s,a)^2,
\end{equation*}	
holds for any $u \in \Omega(P)$.
\end{lemma}
\begin{proof} 
Let $\Phi_t = \min_{q\in \Omega(P)} \inner{ q, \sum_{\tau=1}^{t-1} \hatl_\tau   } + \phi_t(q)$ and $D_F(u,v)$ being the Bregman divergence with convex function $F$, that is, $D_F(u,v) = F(u) - F(v) - \inner{ u - v, \nabla F(v) }$. 

Then, we have
\begin{align*}
\Phi_t & = \inner{q_t,  \sum_{\tau=1}^{t-1} \hatl_\tau   }  + \phi_t(q_t) \\
& = \inner{q_{t+1} ,  \sum_{\tau=1}^{t-1} \hatl_\tau   } + \phi_t(q_{t+1}) - \rbr{  \inner{q_{t+1}  - q_t ,  \sum_{\tau=1}^{t-1} \hatl_\tau   }  + \phi_t(q_{t+1}) -  \phi_t(q_{t}) } \\
& \leq \inner{q_{t+1} ,  \sum_{\tau=1}^{t-1} \hatl_\tau   } + \phi_t(q_{t+1}) - \rbr{ -  \inner{q_{t+1}  - q_t , \nabla \phi_t(q_t)    }  + \phi_t(q_{t+1}) -  \phi_t(q_{t}) } \\
& = \inner{q_{t+1} ,  \sum_{\tau=1}^{t-1} \hatl_\tau   } + \phi_t(q_{t+1}) - D_{\phi_t}(q_{t+1}, q_t )\\
& = \Phi_{t+1} - \inner{q_{t+1}, \hatl_t} -  \rbr{ \phi_{t+1}(q_{t+1})  - \phi_{t}(q_{t+1}) } -  D_{\phi_t}(q_{t+1}, q_t ),
\end{align*}
where the third line follows from the first order optimality condition of $q_t$, that is, $\inner{q_{t+1} - q_t, \nabla \phi_t(q_t) + \sum_{\tau=1}^{t-1} \hatl_\tau  } \geq 0$. 

Taking the summation over all episodes gives 
\begin{align*}
\Phi_1 & = \Phi_{T+1} - \sum_{t=1}^{T} \inner{q_{t+1}, \hatl_t}  - \sum_{t=1}^{T} \rbr{ \phi_{t+1}(q_{t+1})  - \phi_{t}(q_{t+1}) } - \sum_{t=1}^{T}  D_{\phi_t}(q_{t+1}, q_t ).
\end{align*}
Therefore, we have 
\begin{align*}
&\sum_{t=1}^{T} \inner{ q_t - u, \hatl_t  } \\
& = \sum_{t=1}^{T} \inner{ q_t - u, \hatl_t  }  + \Phi_{T+1} - \Phi_{1} -  \sum_{t=1}^{T} \inner{q_{t+1}, \hatl_t}  - \sum_{t=1}^{T} \rbr{ \phi_{t+1}(q_{t+1})  - \phi_{t}(q_{t+1}) } - \sum_{t=1}^{T}  D_{\phi_t}(q_{t+1}, q_t ) \\
& = \sum_{t=1}^{T} \rbr{ \inner{ q_t - q_{t+1}, \hatl_t  } - D_{\phi_t}(q_{t+1}, q_t ) } -  \sum_{t=1}^{T}  \inner{ u, \hatl_t  } +  \Phi_{T+1}  - \Phi_{1}   - \sum_{t=1}^{T} \rbr{ \phi_{t+1}(q_{t+1})  - \phi_{t}(q_{t+1}) } \\
& \leq \underbrace{\sum_{t=1}^{T} \rbr{ \inner{ q_t - q_{t+1}, \hatl_t  } - D_{\phi_t}(q_{t+1}, q_t ) } }_{\textsc{Stability}} + \underbrace{\phi_{T+1}(u) - \phi_1(q_1)  - \sum_{t=1}^{T} \rbr{ \phi_{t+1}(q_{t+1})  - \phi_{t}(q_{t+1}) } }_{\textsc{Penalty}}
\end{align*}
where the last line follows from the optimality condition $\Phi_{T+1} \leq \sum_{t=1}^{T} \inner{u,  \hatl_t } + \phi_{T+1}(u)$.

To bound the stability term, we first consider relaxing the constraint and taking the maximum as: 
\begin{align*}
\inner{ q_t - q_{t+1}, \hatl_t  } - D_{\phi_t}(q_{t+1}, q_t )  \leq \max_{q\in \fR_{+}^{S \times A} }  \inner{ q_t - q, \hatl_t  } - D_{\phi_t}(q, q_t ).
\end{align*}
Denote by $\wtilq_t$ the maximizer of the right hand side. Setting the gradient to zero yields the equality $ \nabla \phi_t(q_t) - \nabla \phi_t(\wtilq_t) = \hatl_t$. By direction calculation, one can verify that  $\wtilq_t(s,a) = q_t(s,a) \cdot \exp\rbr{ - \eta_t \cdot \hatl_t(s,a) }$ for all state-action pairs, and the following inequality that 
\begin{align*}
\inner{ q_t - q_{t+1}, \hatl_t  } - D_{\phi_t}(q_{t+1}, q_t ) & \leq  \inner{ q_t  - \wtilq_t , \hatl_t  } - D_{\phi_t}(\wtilq_t, q_t ) \\
& = \inner{ q_t  - \wtilq_t , \hatl_t  } - \phi_t(\wtilq_t) + \phi_t(q_t) - \inner{\wtilq_t - q_t, \nabla \phi_t(q_t) } \\
& =  D_{\phi_t}( q_t, \wtilq_t )
\end{align*}
where the second equality uses the equality $ \nabla \phi_t(q_t) - \nabla \phi_t(\wtilq_t) = \hatl_t$.

Moreover, the term  $D_{\phi_t}( q_t, \wtilq_t )$ can be bounded as: 
\begin{align*}
D_{\phi_t}( q_t, \wtilq_t ) & = \frac{1}{\eta_t} \sum_{s\neq s_L} \sum_{a \in A} \rbr{  q_t(s,a) \ln \rbr{ \frac{q_t(s,a)}{\wtilq_t(s,a)} }  -  q_t(s,a) + \wtilq_t(s,a)} \\
& = \frac{1}{\eta_t} \sum_{s\neq s_L} \sum_{a \in A}  q_t(s,a) \cdot \rbr{   \eta_t \hatl_t(s,a)   -  1 +  \exp\rbr{ - \eta_t \cdot \hatl_t(s,a) } } \\
& \leq \eta_t \sum_{s\neq s_L} \sum_{a \in A} q_t(s,a) \hatl_t(s,a)^2 
\end{align*}
where the last inequality follows from the facts $y - 1 + e^{-y} \leq y^2$ for $y > -1$ and $\eta_t \cdot \hatl_t(s,a) \geq - 1$ for all sate-action pairs. 

On the other hand, the penalty term is at most
\[
\phi_{T+1}(u) - \phi_1(q_1)  - \sum_{t=1}^{T} \rbr{ \phi_{t+1}(q_{t+1})  - \phi_{t}(q_{t+1}) } 
\leq  - \frac{\phi(q_1)}{\eta_1} -\sum_{t=1}^{T}\rbr{ \frac{1}{\eta_{t+1}} - \frac{1}{\eta_{t}} } \phi(q_t),
\]
since $\phi(u)\leq 0$. Moreover, note that for any valid occupancy measure $q$, it holds that 
\begin{align*}
\phi(q) & = \sum_{k = 0}^{L-1} \sum_{s \in S_k } \sum_{a\in A} q(s,a) 
 \geq  - \sum_{k=0}^{L-1}  \ln(|S_k||A|) \geq - L\ln(|S||A|).
\end{align*}

Therefore, the penalty term is bounded by 
\begin{align*}
& - \frac{\phi(q_1)}{\eta_1} -\sum_{t=1}^{T}\rbr{ \frac{1}{\eta_{t+1}} - \frac{1}{\eta_{t}} } \phi(q_t) \\
& \leq L\ln(|S||A|) \cdot \rbr{ \frac{1}{\eta_1} + \sum_{t=1}^{T}\rbr{ \frac{1}{\eta_{t+1}} - \frac{1}{\eta_{t}} } } = \frac{L\ln(|S||A|)}{\eta_{T+1}}.
\end{align*}

Finally, combining the bounds for the stability and penalty terms finishes the proof. 
\end{proof}

\subsection{Known Transition and Bandit Feedback: $\ftrl$ with Tsallis Entropy}
\label{sec:known_transition_tsallis_entropy_algoritm}

In this section, we consider the bandit feedback setting with known transition.
We use the following hybrid regularizer with learning rate $\eta_t = \nicefrac{\gamma}{\sqrt{t}}$ for episode $t$: 
\begin{equation}
\phi_t(q) = \frac{\phi_{H}(q)}{\eta_t} + \underbrace{\paralog \sum_{s \neq s_L}\sum_{a\in A} \log \frac{1}{q(s,a)}}_{= \phi_{L}(q) },
\label{eq:hybrid_tsallis_reg_def}
\end{equation}
where $\phi_{L}$ is a fixed log-barrier regularizer, and  $\phi_{H}(q)$ is the $\nicefrac{1}{2}$-Tsallis entorpy:
\[
\phi_{H}(q) = - \sum_{s \neq s_L}\sum_{a \in A} \sqrt{q(s,a)}.
\]
 We present the pseudocode of our  algorithm in \pref{alg:known_transition_tsallis_entropy_algoritm}, and show the ensured guarantees in \pref{thm:known_transition_tsallis_entropy_main}, which is a more detailed version of \pref{thm:known_transition_main}. 
In particular, the adaptive regret bound~\pref{eq:self_bounding_tsallis_entropy}  is a strict improvement of~\citep[Theorem~1]{jin2020simultaneously} and leads to the best-of-both-worlds guarantee automatically.
 We emphasize that the key to achieve such a guarantees  is the loss-shifting function defined in \pref{eq:loss_shift_function}.

\begin{algorithm}[t]
	\caption{Best-of-both-worlds for MDPs with Known Transition and Bandit Feedback}
	\label{alg:known_transition_tsallis_entropy_algoritm}
	\begin{algorithmic}
		\FOR{$t=1$ {\bfseries to} $T$}
		\STATE compute $q_t = \argmin_{q\in\Omega} \big\langle q,  \sum_{\tau < t} \hatl_\tau \big\rangle + \phi_t(q)$ where $\phi_t(q)$ is defined in \pref{eq:hybrid_tsallis_reg_def}.
		\STATE execute policy $\pi_t$ where $\pi_t(a | s) = q_t(s,a)/q_t(s)$.
		\STATE observe $(s_0, a_0, \ell_t(s_0, a_0)), \ldots, (s_{L-1}, a_{L-1}, \ell_t(s_{L-1}, a_{L-1}))$.
		\STATE construct estimator $\hatl_t$ such that:  $\forall (s,a), \hatl_t(s,a) = \frac{\ell_t(s,a)}{q_t(s,a)}\Ind{s_{k(s)}=s,a_{k(s)}= a}$.
		\ENDFOR
	\end{algorithmic}
\end{algorithm}

\begin{theorem}
	\label{thm:known_transition_tsallis_entropy_main}
	With $\paralog=64L$ and $\gamma = 1$, \pref{alg:known_transition_tsallis_entropy_algoritm} ensures
	that $\Reg_T(\optpi)$ is bounded by
	\begin{equation}\label{eq:self_bounding_tsallis_entropy}
	\sum_{t=1}^T \otil\rbr{\min\cbr{
			\E\sbr{B \sum_{s\neq s_L}\sum_{a\neq \pi(s)}\sqrt{\frac{q_t(s,a)}{t}} + D\sqrt{\sum_{s\neq s_L}\sum_{a\neq \pi(s)}\frac{q_t(s,a)+\opt(s,a)}{t}}},
			\sqrt{\frac{L|S||A|}{t}}
		}
	}
	\end{equation}	
	for any mapping $\pi: S \rightarrow A$, where $\cnt = L^2$ and $D = \sqrt{L|S|}$. Therefore, the regret of \pref{alg:known_transition_tsallis_entropy_algoritm} is always bounded as $\Reg_T(\optpi) = \otil\rbr{ \sqrt{L|S||A|T} }$.
	Moreover,
	under Condition~\eqref{eq:loss_condition_known_transition}, $\Reg_T(\optpi)$ is bounded by $\order\rbr{ U + \sqrt{UC} }$ where 
	$
	U = \frac{L|S| \log T}{\gapmin} + \sum_{s \neq s_L}\sum_{a\neq \pi^\star(a)}\frac{L^4 \log T }{\gap(s,a)} + L|S||A|\log T. 
	$
\end{theorem} 

\begin{proof}
By~\citep[Lemma~5]{jin2020simultaneously}, with a sufficiently large log-barrier component (in particular, $\paralog  = 64L$ suffices), the regret can be decomposed and  bounded as:
\begin{equation*}
\begin{split}
\E\sbr{ \sum_{t=1}^{T} \inner{ q_t - \opt, \hatl_t}  } & \leq  \underbrace{\sum_{t=1}^{T}\rbr{\frac{1}{\eta_t} - \frac{1}{\eta_{t-1}}}\E\sbr{\phi_{H}(\opt) -\phi_{H}(q_t)}  }_{\textsc{Penalty}} + \underbrace{8\sum_{t=1}^{T}\eta_t\E\bigg[\norm{\hatl_t}^2_{\nabla^{-2}\phi(q_t)}\bigg] }_{\textsc{Stability}}  \\ & \; \;+ \order\rbr{L|S||A|\log T}.
\end{split}
\end{equation*}
where $\opt$ is the occupancy measure of an deterministic optimal policy $\optpi: S \rightarrow A$. Moreover, with the help of \pref{col:invariant_with_ftrl}, we can in fact bound $\Reg_T(\optpi)$ as 
\begin{equation}
\begin{split}
\Reg_T(\optpi) & \leq  \underbrace{\sum_{t=1}^{T}\rbr{\frac{1}{\eta_t} - \frac{1}{\eta_{t-1}}}\E\sbr{\phi_{H}(\opt) -\phi_{H}(q_t)}}_{\textsc{Penalty}} + \order\rbr{L|S||A|\log T} \\
&  \\ & \; \; +\underbrace{8\sum_{t=1}^{T}\eta_t \E\Bigg[\min\cbr{  \E_t\bigg[\norm{\hatl_t}^2_{\nabla^{-2}\phi(q_t)} \bigg],  \E_t\bigg[ \norm{\hatl_t + g_t}^2_{\nabla^{-2}\phi(q_t)} \bigg] } \Bigg]  }_{\textsc{Stability}} .
\end{split}
\label{eq:bobw_known_bandit_reg_decomp}
\end{equation}
where $g_t$ is the specific loss-shifting function  defined in \pref{eq:loss_shift_function}. This is again because adding the loss-shifting function $g_t$ does not influence the outcomes of $\ftrl$ and thus in the analysis, one can decide whether to add $g_t$ or not for episode $t$ in hindsight to establish a tighter adaptive regret bound. 

Before analyzing the stability term, we point out that $\phi_{H}(\opt) -\phi_{H}(q_t)$ can be bounded as
\begin{align}
\rbr{\phi_H(\opt) -\phi_H(q_t)} & \leq  \sum_{s \neq s_L} \sum_{a \neq \pi(s) } \sqrt{q_t(s,a)} +  2\sqrt{|S|L \sum_{s \neq s_L}\sum_{a\neq \pi(s)} q_t(s,a) + \opt(s,a)   }   \label{eq:penalty_bound}
\end{align}
for any mapping $\pi: S \rightarrow A$ according to~\citep[Lemma~6]{jin2020simultaneously} (take $\alpha$ in their lemma to be $0$). On the other hand, we also have $\phi_H(\opt) -\phi_H(q_t) \leq - \phi_H(q_t) \leq \sqrt{L|S||A|}$ by  the Cauchy-Schwarz inequality. Combining these two cases and the fact $\frac{1}{\eta_{t}} - \frac{1}{\eta_{t-1}} = \frac{1}{\gamma} \cdot \rbr{ \sqrt{t} - \sqrt{t-1} }  \leq \frac{1}{\gamma} \cdot \frac{1}{\sqrt{t}}$, the penalty term is bounded by 
\[
\frac{1}{\gamma}\sum_{t=1}^T \E\sbr{\min\cbr{
\sqrt{\frac{L|S||A|}{t}}, \rbr{\sum_{s \neq s_L}\sum_{a\neq \pi(s)}\sqrt{\frac{q_t(s,a)}{t}}}
+ 2\sqrt{|S|L\sum_{s \neq s_L}\sum_{a\neq \pi(s)} \frac{q_t(s,a) + \opt(s,a)}{t}}
}}
\]


We now bound the stability term. By direct calculation, we have
\begin{align}
\E_t\sbr{\norm{\hatl_t + g_t}^2_{\nabla^{-2}\phi(q_t)}} &  = \sum_{s \neq s_L}\sum_{a\in A} q_t(s,a)^{\nicefrac{3}{2}} \E_t\sbr{\rbr{\hatl_t(s,a) + g_t(s,a) }^2} \notag \\ 
& \leq 2L^2 \sum_{s \neq s_L}\sum_{a\in A}  \sqrt{q_t(s,a)} \cdot \rbr{1-\pi_t(a|s)},  \label{eq:hybrid_tsallis_stab_main}
\end{align}
where the second line applies the properties of the loss-shifting function in \pref{lem:loss_shifting_loss_estimator}. 

For any mapping $\pi: S\rightarrow A$, we can further bound  \pref{eq:hybrid_tsallis_stab_main} as 
\begin{align*}
& 2L^2 \sum_{s \neq s_L}\sum_{a\in A}  \sqrt{q_t(s,a)} \cdot \rbr{1-\pi_t(a|s)} \\
& \leq 2L^2 \sum_{s \neq s_L} \sum_{a\neq \pi(s) } \sqrt{q_t(s,a)} + 2L^2 \sum_{s \neq s_L} \sqrt{q_t(s)} \cdot \rbr{ \sum_{a \neq \pi(s)} \pi_t(a|s)} \\
& \leq 4L^2 \sum_{s \neq s_L} \sum_{a\neq \pi(s) } \sqrt{q_t(s,a)}, 
\end{align*}
where the third line follows from the fact $x \leq \sqrt{x}$ for $x\in [0,1]$. 

Therefore, for any mapping $\pi : S\rightarrow A$, the stability term is bounded by
\begin{equation}
\sum_{t=1}^{T} 8\eta_t \E\sbr{\norm{\hatl_t + g_t }^2_{\nabla^{-2}\phi(q_t)}}  \leq 32 L^2 \cdot \sum_{t=1}^{T}\eta_t \E\sbr{\sum_{s \neq s_L}\sum_{a\neq \pi(s)}\sqrt{q_t(s,a)}}. \label{eq:stability_bound}
\end{equation}

On the other hand, without the loss-shifting function, the stability term is simultaneously bounded as
\begin{align*}
&\sum_{t=1}^{T} 8\eta_t \E\sbr{\norm{\hatl_t }^2_{\nabla^{-2}\phi(q_t)}} = \sum_{t=1}^{T} 8\eta_t  \E\sbr{ \sum_{s \neq s_L}\sum_{a\in A}  q_t(s,a)^{\nicefrac{3}{2}} \cdot \hatl_t(s,a)^2  } \\
& \leq \sum_{t=1}^{T} 8\eta_t  \E\sbr{ \sum_{s \neq s_L}\sum_{a\in A}  \sqrt{q_t(s,a)} }
\leq  \sum_{t=1}^{T} 8\eta_t \sqrt{L|S||A|}. \tag{Cauchy-Schwarz inequality}
\end{align*}

Plugging \pref{eq:penalty_bound} and \pref{eq:stability_bound} into the \pref{eq:bobw_known_bandit_reg_decomp} shows that \pref{alg:known_transition_tsallis_entropy_algoritm} ensures the following self-bounding regret bound for $\Reg_T(\optpi)$:
\begin{equation}
\begin{split}
&\frac{1}{\gamma}\sum_{t=1}^T \E\sbr{\min\cbr{
\sqrt{\frac{L|S||A|}{t}}, \rbr{\sum_{s \neq s_L}\sum_{a\neq \pi(s)}\sqrt{\frac{q_t(s,a)}{t}}}
+ 2\sqrt{|S|L\sum_{s \neq s_L}\sum_{a\neq \pi(s)} \frac{q_t(s,a) + \opt(s,a)}{t}}
}} \\
&32 \gamma \cdot \sum_{t=1}^{T} \E\sbr{ \min\cbr{ \sqrt{\frac{L|S||A|}{t}} , L^2 \sum_{s \neq s_L}\sum_{a\neq \pi(s)}\sqrt{\frac{q_t(s,a)}{t}}}} 
 + \order\rbr{L|S||A|\log T}
\end{split}
\label{eq:hybrid_tsallis_reg}
\end{equation}
for any mapping $\pi: S\rightarrow A$. 
Picking $\gamma = 1$ and using $\min\cbr{a,b} + \min\cbr{c,d} \leq \min\cbr{a+c,b+d}$ proves \pref{eq:self_bounding_tsallis_entropy}.

The (optimal) worst-case bound $\Reg_T(\optpi)  = \otil(\sqrt{L|S||A|T})$ can be obtained by using the second argument of the min operator in \pref{eq:self_bounding_tsallis_entropy},
while the logarithmic regret bound under Condition~\eqref{eq:loss_condition_known_transition} is obtained by using the first argument of the min operator and the exact same reasoning as in~\citep[Appendix A.1]{jin2020simultaneously}.
\end{proof} 

We point out that with a different choice $\gamma = \nicefrac{1}{L}$, \pref{alg:known_transition_tsallis_entropy_algoritm} achieves a regret bound of $\Reg_T(\optpi)  = \order\rbr{V + \sqrt{VC}}$ under Condition~\eqref{eq:loss_condition_known_transition}, where 
\[
V = \frac{L^3|S| \log T}{\gapmin} + \sum_{s \neq s_L}\sum_{a\neq \pi^\star(a)}\frac{L^2 \log T }{\gap(s,a)} + L|S||A|\log T  
\]
which matches the best existing regret bound in \cite{simc2019}. 
This choice of $\gamma$ worsens the worst-case bound though.


\newpage
\section{Best of Both Worlds for MDPs with Unknown Transition and Full Information}
\label{app:bobw_unknown_transition_fullinfo}
In this part, we will prove the best of both worlds results for the full-information setting. We present the bound for the adversarial world in \pref{prop:bobw_fullinfo_adv_lem}, and that for the stochastic world in \pref{prop:bobw_fullinfo_stoc_lem} (part of which is a restatement of \pref{lem:bobw_fullinfo_stoc_self_bound_terms_lem}).
Together, they prove \pref{thm:main_full_info}. 

\begin{proposition} \label{prop:bobw_fullinfo_adv_lem} Consider the decomposition  $\Reg_T(\optpi) = \E\sbr{ \textsc{Err}_1  + \textsc{EstReg}  + \textsc{Err}_2}$ stated in \pref{eq:regret_basic_decomp}. Then, with $\delta = \frac{1}{T^2}$ \pref{alg:bobw_framework} ensures:
\begin{itemize} 
		\item $\E\sbr{ \textsc{Err}_1} = \otil\rbr{ L|S|\sqrt{|A|T }+ L^3|S|^3|A|}$, 
		\item $\E\sbr{\textsc{Err}_2 } = \otil\rbr{ 1 }$, 
		\item $\E\sbr{\textsc{EstReg}} = \otil\rbr{  L\sqrt{|S||A| T } + L^2|S|^2|A|^{\frac{3}{2}} + L^3|S||A| }$.  
\end{itemize} 
\end{proposition}

\begin{proposition}\label{prop:bobw_fullinfo_stoc_lem}
	With $\delta = \frac{1}{T^2}$, \pref{alg:bobw_framework} ensures that  $\Reg_T(\pi^\star)$  is bounded as  
	\begin{align*} 
	& \order\rbr{\E\Bigg[ \underbrace{\selfterm_1\rbr{ L^4|S| \ln T}}_{\textsc{ErrSub} }  + \underbrace{\selfterm_2\rbr{ L^4|S| \ln T  } }_{\textsc{ErrOpt} } +   \underbrace{\selfterm_3\rbr{ L^4 \ln T }}_{\textsc{OccDiff} }  +  \underbrace{\selfterm_4\rbr{ L^5 |S||A| \ln T \ln(|S||A|) } } \Bigg]  }\\
	& \quad + \order\rbr{L^4 |S|^3 |A|^2  \ln^2 T  },
	\end{align*}
	where $\selfterm_1$-$\selfterm_4$ are defined in \pref{def:self_bounding_terms}.
	Under Condition~\eqref{eq:loss_condition}, this bound implies $\Reg_T(\pi^\star) = \order\rbr{U + \sqrt{UC}  + V }$ where
	\[
	U = \frac{\rbr{L^6|S|^2+L^5|S||A|\log(|S||A|)}\log T}{\gapmin} + \sum_{s \neq s_L} \sum_{ a\neq \pi^{\star}(s)} \frac{L^6|S|\log T}{\gap(s,a)}, \quad V = L^4 |S|^3 |A|^2  \ln^2 T.
	\]
\end{proposition}

Before diving into the proof details, we first give formal definitions of several notations mentioned in \pref{sec:algorithms} and \pref{sec:analysis} for the full-information setting.
Through out this paper, we denote by $\calA$ the event that $P\in \calP_i$ for all  $i$,
which happens with probability at least $1-4\delta$ based on~\citep[Lemma~2]{jin2019learning}.
 We denote by $N$ the total number of epochs, and set $t_{N+1} = T+1$ for convenience (recall that $t_i$ is the first episode for epoch $i$). 

Then, recall the $\whatQ_t^\pi$ and  $\whatV_t^\pi$ defined in \pref{sec:analysis}, that is, the state-action and state value functions associated with the empirical transition $\bar{P}_{i(t)}$ and the adjusted loss $\hatl_t$, formally defined as:
\begin{equation}\label{eq:full_info_est_QV_def}
\begin{aligned}
\whatQ_t^\pi(s,a) = \hatl_t(s,a) + \sum_{s' \in S_{k(s)+1}} \bar{P}_{i(t)}(s'|s,a) \whatV_t(s'), \quad \whatV_t^\pi(s) = \sum_{a \in A} \pi(a|s) \whatQ_t^\pi(s,a),
\end{aligned}
\end{equation}
and $\whatQ_t^\pi(s_L, a) = 0$ for all $a$.
Also recall that the notation $\whatQ_t$ and $\whatV_t$ used in the loss-shifting function are shorthands for $\whatQ_t^{\pi_t}$ and $\whatV_t^{\pi_t}$.
Similarly, the true state-action and state value functions of episode $t$ are defined as: 
\begin{equation}\label{eq:full_info_true_QV_def}
\begin{aligned}
Q_t^\pi(s,a) = \ell_t(s,a) + \sum_{s' \in S_{k(s)+1}} P(s'|s,a) V_t(s'), \quad V_t^\pi(s) = \sum_{a \in A} \pi(a|s) Q_t^\pi(s,a),
\end{aligned}
\end{equation}
with $Q_t^\pi(s_L, a) = 0$ for all $a$.
For notational convenience, we let $\pcons = \frac{T|S||A|}{\delta}$ and assume that $\delta \in \rbr{0,1}$, and denote by $T_k$ the set of transition tuples at layer $k$, that is, $T_k = \cbr{ (s,a,s') \in S_k \times A \times S_{k+1}}$. 

\subsection{Optimism of Adjusted losses and Other Lemmas} 
First, we show that the adjusted loss $\hatl_t$ defined in \pref{eq:adjusted_loss} ensures the optimism of the estimated state-action and state value functions as stated in \pref{lem:full_info_optimism}. As discussed in \pref{sec:analysis}, this certain kind of optimism ensures that $\E\sbr{\textsc{Err}_2}$ is bounded by a constant with a sufficiently small confidence parameter $\delta$.  
\begin{lemma}\label{lem:full_info_optimism} Using the notations in \pref{eq:full_info_est_QV_def} and \pref{eq:full_info_true_QV_def} and conditioning on the event $\calA$, we have
	\[
	\whatQ_t^{\pi}(s,a)\leq Q_t^\pi(s,a), \forall (s,a)\in S\times A, t \in [T].
	\]
\end{lemma} 
\begin{proof} We prove this result via a backward induction from layer $L$ to layer $0$. 
	
	\textbf{Base case:} for $s_L$, $\whatQ_t^{\pi}(s,a) = Q_t^\pi(s,a) = 0$ holds always. 
	
	\textbf{Induction step:} Suppose $\whatQ_t^{\pi}(s,a)\leq Q_t^\pi(s,a)$ holds for all the states $s$ with $k(s) > h$. Then, for any state $s$ in layer $h$, we have 
	\begin{align*} 
	\whatQ_t^{\pi}(s,a) & = \ell_t(s,a) + \sum_{s' \in S_{k(s)+1}} \bar{P}_{i(t)}(s'|s,a) \whatV_t^{\pi}(s') - L \cdot B_t(s,a) \\
	& \leq  \ell_t(s,a) + \sum_{s' \in S_{k(s)+1}} \bar{P}_{i(t)}(s'|s,a) V_t^{\pi}(s') - L \cdot B_t(s,a) \tag{Induction  hypothesis} \\
	& \leq \ell_t(s,a) + \sum_{s' \in S_{k(s)+1}} P(s'|s,a) V_t^{\pi}(s') \\
	& \quad + \sum_{s' \in S_{k(s)+1}} \rbr{ \bar{P}_{i(t)}(s'|s,a)  - P(s'|s,a) } V_t^{\pi}(s') - L\cdot B_{i(t)}(s,a) \\
	& = Q_t^\pi(s,a) + \sum_{s' \in S_{k(s)+1}} \rbr{ \bar{P}_{i(t)}(s'|s,a)  - P(s'|s,a) } V_t^{\pi}(s') - L\cdot B_{i(t)}(s,a) 
	\end{align*}
	where the first line follows from the definition of $\hatl_t$.
	
	Clearly, when $B_{i(t)}(s,a) = 1$, we have 
	\begin{align*}
	 \sum_{s' \in S_{k(s)+1}} \rbr{ \bar{P}_{i(t)}(s'|s,a)  - P(s'|s,a) } V_t^{\pi}(s') - L\cdot B_{i(t)}(s,a)  \leq \sum_{s' \in S_{k(s)+1}} \bar{P}_{i(t)}(s'|s,a) \cdot L - L = 0 
	\end{align*}
	where the inequality follows from the fact $0\leq V_t^{\pi}(s') \leq L $. 
	
	On the other hand, when $\sum_{s' \in S_{k(s)+1}}B_{i(t)}(s,a,s') = B_{i(t)}(s,a)$, we have 
	\begin{align*}
	& \sum_{s' \in S_{k(s)+1}} \rbr{ \bar{P}_{i(t)}(s'|s,a)  - P(s'|s,a) } V_t^{\pi}(s') - L\cdot B_{i(t)}(s,a)  \\
	& \leq \sum_{s' \in S_{k(s)+1}}  B_{i(t)}(s,a,s') \cdot L - L\cdot B_{i(t)}(s,a) = 0
	\end{align*}
	where the second line uses the definition of event $\calA$. 
	 
	 Combining these two cases shows that $\whatQ_t^{\pi}(s,a) \leq Q_t^\pi(s,a)$ holds for all state-action pairs $(s,a)$ at layer $h$, finishing the induction.
%
\end{proof}

Next, we analyze the estimated regret suffered within one epoch. With slightly abuse of notation, we denote by $\EReg_i(\pi)$ the difference between the total loss suffered within epoch $i$ and that of the fixed policy $\pi$ with respect to the empirical transition $\bar{P}_i$ and the adjusted losses within epoch $i$, that is, 
\begin{equation}\label{eq:epoch_est_reg_def}
\EReg_i(\pi) = \E\sbr{ \sum_{t=t_i}^{t_{i+1}-1} \inner{ q^{\bar{P}_i, \pi_t } - q^{\bar{P}_i, \pi }, \hatl_t}} = \E\sbr{ \sum_{t=t_i}^{t_{i+1}-1} \inner{ \whatq_t - q^{\bar{P}_i, \pi }, \hatl_t}}.
\end{equation}
In addition, we let $\EReg_i = \max_{\pi} \EReg_i(\pi)$ be the maximum regret suffered within epoch $i$. 

\begin{lemma}\label{lem:bobw_fullinfo_shannon_reg} For full-information feedback, \pref{alg:bobw_framework} ensures that $\EReg_i$ is bounded by $\order\rbr{ L^3\ln(|S||A|) }$ plus:  
	\begin{equation}\label{eq:bobw_fullinfo_shannon_reg}
	\order\rbr{ \E\sbr{ \sqrt{  L\ln \rbr{|S||A|}\cdot \min\cbr{  L^4\sum_{t=t_i}^{t_{i+1}-1} \sum_{s \in S } \sum_{a \neq \pi(s) } \whatq_t(s,a), \sum_{t=t_i}^{t_{i+1}-1} \sum_{s \neq s_L } \sum_{a \in A} \whatq_t(s,a)\hatl_t(s,a)^2}  }} }. 
	\end{equation}
\end{lemma}

\begin{proof}
The proof follows the same steps as in that of \pref{thm:known_transition_shannon_entropy_main}.
Due to the invariant property, the loss of episode $t$ fed to \ftrl can be seen as either $\hatl_t(s,a)$ or $\whatQ_t(s,a) - \whatV_t(s,a)$.
By the definition of $\eta_t$, we have both $\eta_t \hatl_t(s,a) \geq -1$ and $\eta_t(\whatQ_t(s,a) - \whatV_t(s,a)) \geq -1$.
Therefore, we can apply \pref{lem:shannon_entropy_reg_bound} and bound $\EReg_i$ by
\begin{align*}
\E\sbr{\frac{ L\ln(|S||A|)}{\eta_{t_{i+1}}} 
 +  \sum_{t=t_i}^{t_{i+1}-1} \eta_t \min\cbr{ \sum_{s \neq s_L} \sum_{a\in A} \whatq_t(s,a) \rbr{ \whatQ_t(s,a) - \whatV_t(s)  }^2,  \sum_{s \neq s_L} \sum_{a \in A} \whatq_t(s,a) \ell_t(s,a)^2}}.
\end{align*}
The tuning of $\eta_t$ makes sure that the above is further bounded by $\order\rbr{ L^3\ln(|S||A|) }$ plus $\sqrt{L\ln \rbr{|S||A|}}$ multiplied with
\begin{align*}
\order\rbr{ \E\sbr{ \sqrt{\min\cbr{\sum_{t=t_i}^{t_{i+1}-1} \sum_{s \neq s_L} \sum_{a\in A} \whatq_t(s,a) \rbr{ \whatQ_t(s,a) - \whatV_t(s)  }^2,  \sum_{t=t_i}^{t_{i+1}-1}\sum_{s \neq s_L} \sum_{a \in A} \whatq_t(s,a) \ell_t(s,a)^2}  }} }; 
\end{align*}
see the beginning of the proof of \pref{thm:known_transition_shannon_entropy_main} for the same reasoning.
Finally, it remains to bound 
$\sum_{s\neq s_L} \sum_{a\in A} \whatq_t(s,a) \rbr{ \whatQ_t(s,a) - \whatV_t(s) }^2$
by $8L^4  \sum_{s\neq s_L} \sum_{a\neq \pi(s)} \whatq_t(s,a)$.
This is again by the same reasoning as \pref{eq:step1} and \pref{eq:step2}, except that $\whatQ_t(s,a)$ now has a range in $[-L^2, L^2]$ which explains the extra $L^2$ factor. 
\end{proof}

\subsection{Proof for the Adversarial World (\pref{prop:bobw_fullinfo_adv_lem})}
\label{app:bobw_fullinfo_adv_proof}
We analyze the regret based on the decomposition in \pref{eq:regret_basic_decomp} and consider bounding the terms $\E\sbr{\textsc{Err}_1}$, $\E\sbr{\textsc{Err}_2}$ and $\E\sbr{\textsc{EstReg}}$ separately. 

\paragraph{$\textsc{Err}_1$}
Following the similar idea of \cite{jin2019learning}, we decompose this term as:  
\begin{align*}
\textsc{Err}_1 & = \sum_{t=1}^{T} V^{\pi_t}_t(s_0) - \whatV^{\pi_t}_t(s_0) = \sum_{t=1}^{T} \inner{q_t, \ell_t } - \inner{ \whatq_t, \hatl_t } \\
&  = \sum_{t=1}^{T} \inner{q_t, \ell_t } - \inner{ \whatq_t, \ell_t } + L \cdot \sum_{t=1}^{T} \inner{\whatq_t, B_{i(t)} } \\
& =  \sum_{t=1}^{T} \inner{q_t, \ell_t } - \inner{ \whatq_t, \ell_t } + L \cdot \sum_{t=1}^{T} \inner{ q_t, B_{i(t)} } +  L \cdot \sum_{t=1}^{T} \inner{ \whatq_t - q_t, B_{i(t)} } \\
& \leq \sum_{t=1}^{T} \sum_{s \neq s_L} \sum_{a \in A} \abr{ q_t(s,a) - \whatq_t(s,a) } +  L \cdot \sum_{t=1}^{T} \inner{ q_t, B_{i(t)} } +  L \cdot \sum_{t=1}^{T} \inner{ \whatq_t - q_t, B_{i(t)} } 
\end{align*}
where the last line follows from the fact $0 \leq \ell_t(s,a) \leq 1$. According to this decomposition, we next consider bounding the expectation of these three terms separately. 

First, we focus on the second term:
\begin{align}
& \E\sbr{ L \cdot \sum_{t=1}^{T} \inner{ q_t, B_{i(t)} }  } \notag \\
& \leq L \cdot \E\sbr{ \sum_{k = 0}^{L-1}\sum_{s\in S_k} \sum_{a \in A}  \sum_{t=1}^{T} q_t(s,a) \rbr{ 2 \sqrt{ \frac{ |S_{k(s)+1}|\ln \pcons }{\max \cbr{m_{i}(s,a) , 1} }} + \frac{14|S_{k(s)+1}|\ln \pcons }{3\max \cbr{m_{i}(s,a) , 1} } }  } \notag \\
& = \order\rbr{  L \cdot \sum_{k=0}^{L-1} \rbr{ \sqrt{|S_k||S_{k+1}||A|T\ln \pcons } +   |S_{k(s)+1}||S_k||A|\rbr{ 2 + \ln T } \ln \pcons  } } \notag \\
& \leq \order\rbr{   L\sqrt{|A|T\ln \pcons }  \cdot \sum_{k=0}^{L-1} \rbr{  |S_k| + |S_{k+1}| } + L|S|^2|A|\ln^2 \pcons } \notag \\
& = \order\rbr{   L|S|\sqrt{|A|T\ln \pcons } + L|S|^2|A|\ln^2 \pcons   } \notag \\
&=\otil\rbr{   L|S|\sqrt{|A|T} + L|S|^2|A|} 
 \label{eq:bobw_fullinfo_adv_term2} 
\end{align}
where the second line follows \pref{lem:sa_conf_width_bound}, the third line follows from \pref{lem:aux_exp} and the fourth line applies AM-GM inequality. 

Then, for the first term, with the help from residual term $r_t$ defined in \pref{def:residual_terms},  we have 
\begin{align}
& \E\sbr{\sum_{t=1}^{T} \sum_{s \neq s_L} \sum_{a \in A} \abr{ q_t(s,a) - \whatq_t(s,a) } } \notag \\
& \leq \E\sbr{ 4\sum_{t=1}^{T} \sum_{s \neq s_L} \sum_{a \in A} \sum_{k=0}^{k(s)-1} \sum_{(u,v,w)\in T_k} q_t(u,v)  \sqrt{ \frac{ P(w|u,v) \ln \pcons }{ \max\cbr{m_{i(t)}(u,v)  ,1}} }q_t(s,a|w)+  \sum_{t=1}^{T} \sum_{s \neq s_L} \sum_{a \in A} r_t(s,a)  } \notag \\ 
& \leq \E\sbr{ 4L \cdot \sum_{t=1}^{T} \sum_{u \neq s_L} \sum_{v \in A} \sum_{ w \in S_{k(u)+1}} q_t(u,v)  \sqrt{ \frac{ P(w|u,v) \ln \pcons }{ \max\cbr{m_{i(t)}(u,v)  ,1}} }+  \sum_{t=1}^{T} \sum_{s \neq s_L} \sum_{a \in A} r_t(s,a)  } \notag \\ 
& \leq \E\sbr{ 4L \cdot \sum_{t=1}^{T} \sum_{u \neq s_L} \sum_{v \in A} q_t(u,v) \sqrt{ \frac{ \abr{S_{k(u)+1}}\ln \pcons }{ \max\cbr{m_{i(t)}(u,v)  ,1}} }+  \sum_{t=1}^{T} \sum_{s \neq s_L} \sum_{a \in A} r_t(s,a)  } \notag \\ 
& = \order\rbr{  L|S|\sqrt{|A|T\ln \pcons } + L^2|S|^3|A|^2\ln^2\pcons +\delta |S||A| T   } \notag \\
&=\otil\rbr{  L|S|\sqrt{|A|T} + L^2|S|^3|A|^2} \label{eq:bobw_fullinfo_adv_term1} 
\end{align}
where the second line uses the bound of $\abr{ q_t(s,a) - \whatq_t(s,a) }$ in \pref{lem:residual_term_property}; the third line follows from the fact  $\sum_{s \neq s_L}\sum_{a \in A}q_t(s,a|w) \leq L$; the forth line uses the Cauchy-Schwarz inequality; the fifth line follows the same argument in \pref{eq:bobw_fullinfo_adv_term2} and applies the expectation bound of residual terms in \pref{lem:residual_term_property};
and the last line plugs in the value of $\delta = 1/T^2$.

For the last term, using the bound of $\abr{\whatq_t(s,a) - q_t(s,a)}$ in \pref{lem:residual_term_property}, we arrive at
\begin{align}
& \E\sbr{  L \cdot \sum_{t=1}^{T} \inner{ \whatq_t - q_t, B_{i(t)} } } \notag \\
& \leq \E\sbr{  L \cdot \sum_{t=1}^{T} \sum_{s \neq s_L} \sum_{a\in A} B_{i(t)}(s,a) \cdot \rbr{ 4\sum_{k=0}^{k(s)-1} \sum_{(u,v,w)\in T_k} q_t(u,v)  \sqrt{ \frac{ P(w|u,v) \ln \pcons }{ \max\cbr{m_{i(t)}(u,v)  ,1}} }q_t(s,a|w) + r_t(s,a) }   } \notag \\
& \leq \E\sbr{  4L \cdot\sum_{t=1}^{T} \sum_{s \neq s_L} \sum_{a\in A}\sum_{s' \in S_{k(s)+1}} B_{i(t)}(s,a,s') \rbr{ \sum_{k=0}^{k(s)-1} \sum_{(u,v,w)\in T_k} q_t(u,v)  \sqrt{ \frac{ P(w|u,v) \ln \pcons }{ \max\cbr{m_{i(t)}(u,v)  ,1}} }q_t(s,a|w)  }   }  \notag \\
& \quad + \E\sbr{ L\cdot  \sum_{t=1}^{T} \sum_{s \neq s_L} \sum_{a\in A}  r_t(s,a) }, \notag 
\end{align}
where the last line follows from the fact $B_{i(t)}(s,a)\leq 1$. 

According to the definition of the residual term in \pref{def:residual_terms}, we have 
\begin{equation*}
r_t(s,a) \geq  \sum_{s' \in S_{k(s)+1}} B_{i(t)}(s,a,s') \cdot \rbr{\sum_{k=0}^{k(s)-1} \sum_{(u,v,w)\in T_k} q_t(u,v)  \sqrt{ \frac{ P(w|u,v) \ln \pcons }{ \max\cbr{m_{i(t)}(u,v)  ,1}} } }q_t(s,a|w)
\end{equation*}
(in particular, the second summand in the definition of $r_t(s,a)$ is an upper bound of the right-hand side above).
Therefore, we have $\E\sbr{  L \cdot \sum_{t=1}^{T} \inner{ \whatq_t - q_t, B_{i(t)} } }$ further bounded by 
\begin{align}
& \E\sbr{  \rbr{ 4L + L } \cdot \sum_{t=1}^{T} \sum_{s \neq s_L  }\sum_{a \in A} r_t(s,a) } \leq \order\rbr{   L^3|S|^3|A|^2\ln^2\pcons +\delta \cdot L|S||A| T   }
= \otil\rbr{   L^3|S|^3|A|^2}
 \label{eq:bobw_fullinfo_adv_term3} 
\end{align}
where the last inequality uses the expectation bound of residual terms in \pref{lem:residual_term_property}.

Combining all bounds yields 
\begin{align*}
\E\sbr{ \textsc{Err}_1 } = \otil\rbr{ L|S|\sqrt{|A|T} + L^3|S|^3|A|} .
\end{align*}

\paragraph{$\textsc{Err}_2$} According to \pref{lem:full_info_optimism}, \pref{lem:exp_high_prob_bound}, and the fact $\abr{\whatV_t^\pi(s)}\leq L^2$, we have 
\[
\E\sbr{ \textsc{Err}_2} = 
\E\sbr{\sum_{t=1}^{T}\widehat{V}^{\pi}_t(s_0) - V^{\pi}_t(s_0)} \leq  L^2T \Pr[\calA^c]\leq 4L^2T \delta = \otil(1).
\]

\paragraph{$\textsc{EstReg}$} By \pref{lem:bobw_fullinfo_shannon_reg},  we have $\E\sbr{ \textsc{EstReg}}$ bounded as 
\begin{align*}
& \E\sbr{\sum_{i = 1}^{N} \sum_{t=t_i}^{t_{i+1}-1}\inner{ \whatq_t -  q^{\bar{P}_i, \optpi} , \hatl_t}} \leq \E\sbr{\sum_{i = 1}^{N} \EReg_i } \\
& \leq \E\sbr{ \otil\rbr{ \sum_{i = 1}^{N} \sqrt{  L \sum_{t=t_i}^{t_{i+1}-1}   \sum_{s \in S } \sum_{a \in A } \whatq_t(s,a) \hatl_t(s,a)^2 } +  L^3} } \\
& \leq \otil\rbr{   \sqrt{ \E\sbr{  L|S||A|  \sum_{t=1}^{T}   \sum_{s \in S } \sum_{a \in A } \whatq_t(s,a)\hatl_t(s,a)^2 }} + L^3|S||A|   }
\end{align*} 
where the last line follows from the fact $N\leq 4 |S||A|\rbr{ \log T + 1}$ according to \pref{lem:bobw_N_bound} and uses 	Cauchy-Schwarz inequality.

Next, we continue to bound the following key term:
\begin{align*}
& \E\sbr{ \sum_{t=1}^{T}  \sum_{s \in S } \sum_{a \in A } \whatq_t(s,a)\hatl_t(s,a)^2}  \\ 
& = \E\sbr{ \sum_{t=1}^{T} \sum_{s \in S } \sum_{a \in A } \whatq_t(s,a)\rbr{\ell_t(s,a) - L\cdot B_{i(t)}(s,a)}^2}  \\ 
& \leq 2 \cdot \E\sbr{ \sum_{t=1}^{T} \sum_{s \in S } \sum_{a \in A } \whatq_t(s,a)\rbr{ \ell_t(s,a)^2 + L^2 \cdot B_{i(t)}(s,a)^2} } \\
& \leq 2 LT + 2 L^2  \cdot \E\sbr{ \sum_{t=1}^{T} \inner{ \whatq_t, B_{i(t)}} } \\
& =  2 LT + 2L  \cdot \rbr{  L \cdot \E\sbr{ \sum_{t=1}^{T} \inner{ \whatq_t - q_t, B_{i(t)}} }  + L \cdot \E\sbr{ \sum_{t=1}^{T} \inner{ q_t, B_{i(t)}} } }, 
\end{align*}
where the third line uses $\rbr{x+y}^2\leq 2\rbr{x^2 +y^2}$ and the fourth line uses $B_{i(t)}(s,a)\leq 1$. 
Moreover, in the previous analysis of the term $\textsc{Err}_1$, we bound the terms in the bracket with 
\begin{align*}
& \E\sbr{ L \cdot \sum_{t=1}^{T} \inner{ q_t, B_{i(t)}} } \leq  \otil\rbr{  L|S|\sqrt{|A|T } + L|S|^2|A|  },  \tag{from  \pref{eq:bobw_fullinfo_adv_term2}} \\
& \E\sbr{ L \cdot \sum_{t=1}^{T} \inner{ \whatq_t - q_t, B_{i(t)}}  } \leq  \otil\rbr{    L^3|S|^3|A|^2 }. \tag{from \pref{eq:bobw_fullinfo_adv_term3}}
\end{align*} 

Therefore, we have
\[
\E\sbr{ \sum_{t=1}^{T}  \sum_{s \in S } \sum_{a \in A } \whatq_t(s,a)\hatl_t(s,a)^2}
 = \otil\rbr{LT + L|S|\sqrt{|A|T } + L^3|S|^3|A|^2 }
 = \otil\rbr{LT + L^3|S|^3|A|^2 },
\]
which further proves
\[
\E[\textsc{EstReg}] = \otil\rbr{L\sqrt{|S||A|T} + L^2|S|^2|A|^{\frac{3}{2}} + L^3|S||A| }.
\]

\subsection{Proof for the Stochastic World (\pref{prop:bobw_fullinfo_stoc_lem})}
\label{app:bobw_fullinfo_stoc_proof}

As discussed in \pref{sec:analysis}, we decompose $\textsc{Err}_1 + \textsc{Err}_2$ as (see \pref{col:general_perf_decomp}):
\begin{align*}
\textsc{Err}_1 + \textsc{Err}_2 &  =  \sum_{t=1}^{T}\sum_{s \neq s_L} \sum_{a \neq \pi^{\star}(s)} q_t(s,a)\widehat{E}^{\pi^{\star}}_t(s,a)  \tag{\textsc{ErrSub}}\\  
& \quad + \sum_{t=1}^{T}\sum_{s \neq s_L} \sum_{a = \pi^{\star}(s) } \rbr{ q_t(s,a) - q_t^\star(s,a) } \widehat{E}^{\pi^{\star}}_t(s,a)  \tag{\textsc{ErrOpt}} \\
& \quad + \sum_{t=1}^{T}\sum_{s \neq s_L} \sum_{a \in A} \rbr{ q_t(s,a) - \whatq_t(s,a)} \rbr{\whatQ^{\pi^{\star}}_t(s,a) -\whatV^{\pi^{\star}}_t(s)} \tag{\textsc{OccDiff}}  \\
& \quad + \sum_{t=1}^{T}\sum_{s \neq s_L} \sum_{a\neq \pi^{\star}(s)}  q_t^\star(s,a) \rbr{\whatV^{\pi^{\star}}_t(s) - V^{\pi^\star}_t(s) } \tag{\textsc{Bias}} 
\end{align*}
where $\widehat{E}_t^\pi(s,a)$ is defined as: 
\[
\widehat{E}_t^\pi(s,a) = \ell_t(s,a) + \sum_{s' \in S_{k(s)+1}} P(s'|s,a) \whatV_t^\pi(s') - \whatQ_t^\pi(s,a).
\]
Then, we proceed to bound each of the five terms: $\textsc{ErrSub}$, $\textsc{ErrOpt}$, $\textsc{OccDiff}$, $\textsc{Bias}$, and $\textsc{EstReg}$. 

\paragraph{$\textsc{ErrSub}$} 
Conditioning on $\calA$, we know that 
\begin{align*}
\widehat{E}^{\pi^{\star}}_t(s,a) & = LB_{i(t)}(s,a) + \sum_{s' \in S_{k(s)+1}} \rbr{ P(s'|s,a) - \bar{P}_{i(t)}(s'|s,a) } \whatV_t^{\pi^\star}(s') \\
& \leq LB_{i(t)}(s,a) + L^2 \cdot \sum_{s' \in S_{k(s)+1}} B_{i(t)}(s,a,s') \\
& \leq 4L^2 \cdot \sum_{s' \in S_{k(s)+1}} \rbr{ \sqrt{ \frac{\bar{P}_{i(t)}(s'|s,a) \ln \pcons }{ \max\cbr{m_{i(t)}(s,a),1}     }  } + \frac{7 \ln\pcons }{ 3\max\cbr{m_{i(t)}(s,a),1}} } \\
& \leq 4L^2 \rbr{ \sqrt{ \frac{|S| \ln \pcons }{ \max\cbr{m_{i(t)}(s,a),1}     }  } + \frac{7|S| \ln\pcons }{ 3\max\cbr{m_{i(t)}(s,a),1}} }, 
\end{align*}
where the second line follows from the event $\calA$ and the fact $\abr{\whatV_t^\pi(s) }\leq L^2$, and the last line applies the Cauchy-Schwarz inequality. 

Therefore, under event $\calA$, $\textsc{ErrSub}$ can be bounded as:
\begin{align*}
\textsc{ErrSub} &  \leq \sum_{t=1}^{T}\sum_{s \neq s_L} \sum_{a \neq \pi^{\star}(s)} q_t(s,a) \cdot 4L^2 \rbr{ \sqrt{ \frac{|S| \ln \pcons }{ \max\cbr{m_{i(t)}(s,a),1}     }  } + \frac{7|S| \ln\pcons }{ 3\max\cbr{m_{i(t)}(s,a),1}} } \\
& \leq 4 \selfterm_1\rbr{ L^4|S| \ln \pcons } + \frac{28|S|L^2 \ln\pcons }{3}\sum_{t=1}^{T}\sum_{s \neq s_L}\sum_{a\in A}\frac{q_t(s,a)   }{ 3\max\cbr{m_{i(t)}(s,a),1}},
\end{align*}
where the second line follows from the definition of $\selfterm_1(\cdot)$ in \pref{def:self_bounding_terms}.

With the help of \pref{lem:exp_high_prob_bound} and the fact $\abr{\textsc{ErrSub}} \leq L^3 T$, we have 
\begin{align}
\E\sbr{ \textsc{ErrSub} } & \leq \order\rbr{ L^3T \delta + \E\sbr{\selfterm_1\rbr{ L^4|S| \ln \pcons }} }  + \E\sbr{  \frac{28|S|L^2 \ln \pcons }{3}\sum_{t=1}^{T}\sum_{s \neq s_L}\sum_{a\in A}\frac{q_t(s,a)   }{ 3\max\cbr{m_{i(t)}(s,a),1}}   } \notag \\
& =  \order\rbr{\E\sbr{\selfterm_1\rbr{ L^4|S| \ln \pcons }} + L^2|S|^2|A| \ln^2\pcons }, \label{eq:fullinfo_errsub_bound}
\end{align}
where the last line uses \pref{lem:aux_exp}.

\paragraph{$\textsc{ErrOpt}$} By the similar arguments above, we have $\textsc{ErrOpt}$ bounded by the following given event $\calA$: 
\begin{align*}
\textsc{ErrOpt} &  \leq \sum_{t=1}^{T}\sum_{s \neq s_L} \sum_{a = \pi^{\star}(s)} ( q_t(s,a)  - q_t^\star(s,a) )\cdot 4L^2 \rbr{ \sqrt{ \frac{|S| \ln \pcons }{ \max\cbr{m_{i(t)}(s,a),1}     }  } + \frac{7|S| \ln\pcons }{ 3\max\cbr{m_{i(t)}(s,a),1}} }.
\end{align*}

Using the definition of $\selfterm_2(\cdot)$ in \pref{def:self_bounding_terms} and \pref{lem:exp_high_prob_bound}, we have 
\begin{align}
\E\sbr{ \textsc{ErrOpt} } & \leq \order\rbr{ L^3T \delta + \E\sbr{\selfterm_2\rbr{ L^4|S| \ln \pcons }} }  + \E\sbr{  \frac{28|S|L^2 \ln \pcons }{3}\sum_{t=1}^{T}\sum_{s \neq s_L}\sum_{a\in A}\frac{q_t(s,a)   }{ 3\max\cbr{m_{i(t)}(s,a),1}}   } \notag \\
& =  \order\rbr{ \E\sbr{\selfterm_2\rbr{ L^4|S| \ln \pcons }} + L^2|S|^2|A| \ln^2\pcons  }. \label{eq:fullinfo_erropt_bound}
\end{align}

\paragraph{$\textsc{OccDiff}$}  First, we have 
\begin{align*}
\textsc{OccDiff} & = \sum_{t=1}^{T}\sum_{s \neq s_L} \sum_{a \in A} \rbr{ q_t(s,a) - \whatq_t(s,a)} \rbr{\whatQ^{\pi^{\star}}_t(s,a) -\whatV^{\pi^{\star}}_t(s)} \\
& = \sum_{t=1}^{T}\sum_{s \neq s_L} \sum_{a \neq \pi^\star(s) } \rbr{ q_t(s,a) - \whatq_t(s,a)} \rbr{\whatQ^{\pi^{\star}}_t(s,a) -\whatV^{\pi^{\star}}_t(s)}  \\
& \leq 2L^2 \sum_{t=1}^{T}\sum_{s \neq s_L} \sum_{a \neq \pi^{\star}(s)} \abr{ q_t(s,a) - \whatq_t(s,a)},
\end{align*} 
where the second line follows from the fact $\whatV^{\pi^{\star}}_t(s) = \whatQ^{\pi^{\star}}_t(s,a)$ for all state-action pairs $(s,a)$ satisfying $a = \pi^\star(s)$, and the last line uses the fact $\whatQ^{\pi^{\star}}_t(s,a) - \whatV^{\pi^{\star}}_t(s) \leq 2L^2$ for all state-action pairs.
With the help of the residual terms in \pref{def:residual_terms} and \pref{lem:residual_term_property},  we further bound $\textsc{OccDiff}$ as 
\begin{equation}
\begin{aligned}
& 2L^2 \sum_{t=1}^{T}\sum_{s \neq s_L} \sum_{a \neq \pi^{\star}(s)} \abr{q_t(s,a) - \whatq_t(s,a)} \\
& \leq 2L^2 \sum_{t=1}^{T}\sum_{s \neq s_L} \sum_{a \neq \pi^{\star}(s)} r_t(s,a)  \\
& \quad +  8L^2  \sum_{t=1}^{T}\sum_{s \neq s_L} \sum_{a \neq \pi^{\star}(s)} \sum_{k=0}^{k(s)-1} \sum_{(u,v,w)\in T_k} q_t(u,v)  \sqrt{ \frac{ P(w|u,v) \ln \pcons }{ \max\cbr{m_{i(t)}(u,v)  ,1}} }q_t(s,a|w) \\
&= \order\rbr{L^4|S|^3|A|^2 \ln^2 \pcons + L^2|S||A| T \cdot \delta  +\selfterm_3(L^4\ln \pcons)}
\end{aligned}
\label{eq:fullinfo_occdiff_step}
\end{equation}
where the last line is by the definition of $\selfterm_3(\cdot)$ in \pref{def:self_bounding_terms}.
Therefore, we conclude 
\begin{align}
\E\sbr{\textsc{OccDiff}} & \leq \order\rbr{ L^4|S|^3|A|^2 \ln^2 \pcons +  \E\sbr{ \selfterm_3(L^4\ln \pcons) }}. \label{eq:fullinfo_occdiff_bound}
\end{align}

\paragraph{$\textsc{Bias}$}  Conditioning on the event $\calA$,  $\textsc{Bias}$ is nonpositive due to \pref{lem:full_info_optimism}.
Then, by \pref{lem:exp_high_prob_bound}, we bound the expectation of $\textsc{Bias}$ by
\begin{equation}
\E\sbr{ \textsc{Bias} } \leq 0 + \E\sbr{ \Ind{\calA^c} } \cdot  L^3T=  \order\rbr{1} .   \label{eq:fullinfo_bias_bound}
\end{equation}

\paragraph{$\textsc{EstReg}$} By the analysis of estimated regret in \pref{lem:bobw_fullinfo_shannon_reg}, we have $\E\sbr{\textsc{EstReg}}$ bounded by 
(with $C_{\textsc{EstReg}} =  L^5 |S||A|\ln T\ln \rbr{|S||A|}$)
\begin{align*}
&  \order\rbr{ \E\sbr{ \sum_{i = 1}^{N} \sqrt{  L^5 \ln \rbr{|S||A|}\cdot \sum_{t=t_i}^{t_{i+1}-1}   \sum_{s \in S } \sum_{a \neq \pi^\star(s) } \whatq_t(s,a)  } +  L^3\ln (|S||A|)}  } \\
& \leq \order\rbr{   \E\sbr{ \sqrt{ C_{\textsc{EstReg}} \cdot \sum_{t=1}^{T}   \sum_{s \in S } \sum_{a \neq \pi^\star(s) } \whatq_t(s,a)  }} + L^3|S||A|\ln T \ln (|S||A|) } \\
& \leq \order\rbr{   \E\sbr{ \sqrt{ C_{\textsc{EstReg}} \cdot \sum_{t=1}^{T}   \sum_{s \in S } \sum_{a \neq \pi^\star(s) } q_t(s,a)  }} +  L^3|S||A|\ln T \ln (|S||A|)  } \\
& \quad + \order\rbr{ \E\sbr{ \sqrt{ C_{\textsc{EstReg}} \cdot \sum_{t=1}^{T}   \sum_{s \in S } \sum_{a \neq \pi^\star(s) } \abr{ \whatq_t(s,a) - q_t(s,a)}  }} } \\
& \leq \order\Biggrbr{ \E\sbr{ \selfterm_4(L^5 |S||A|\ln T\ln \rbr{|S||A|})  } +  L^5 |S||A|\ln T\ln \rbr{|S||A|} } \\
& \quad + \order\rbr{ \E\sbr{ \sum_{t=1}^{T}   \sum_{s \in S } \sum_{a \neq \pi^\star(s) } \abr{ \whatq_t(s,a) - q_t(s,a)}  } }
\end{align*} 
where the second line uses the Cauchy-Schwarz inequality and the fact $N\leq 4 |S||A|\rbr{ \log T + 1}$ according to \pref{lem:bobw_N_bound}; the third line uses the fact that $\sqrt{x} \leq \sqrt{y}+ \sqrt{\abr{x-y}}$ for $x,y>0$; the last line uses the definition of $\selfterm_4(\cdot)$ in \pref{def:self_bounding_terms} and the AM-GM inequality.

Note that in the analysis of $\textsc{OccDiff}$ (see \pref{eq:fullinfo_occdiff_step}), we have already shown that 
\begin{align}
\sum_{t=1}^{T}\sum_{s \neq s_L} \sum_{a \neq \pi^{\star}(s)} \abr{q_t(s,a) - \whatq_t(s,a)} 
= \order\rbr{ L^2|S|^3|A|^2 \ln^2 \pcons +  \E\sbr{ \selfterm_3(\ln \pcons) }}. 
\label{eq:q_difference}
\end{align}

Combining everything, we have $\E\sbr{\textsc{EstReg}}$  bounded by:
\begin{equation}
\begin{aligned}
\order\Bigrbr{  \E\sbr{ \selfterm_4(L^5 |S||A|\ln T\ln \rbr{|S||A|}) + \selfterm_3(\ln \pcons) } + L^3|S|^3|A|^2 \ln^2 \pcons }.
\end{aligned}
\label{eq:fullinfo_estreg_bound}
\end{equation}

Finally, combining everything we have shown that \pref{alg:bobw_framework} ensures the following regret bound for $\Reg_T(\pi^\star)$:
\begin{align*}
&\order\rbr{  \E\sbr{ \selfterm_1\rbr{ L^4|S| \ln\pcons  } } } \tag{from \pref{eq:fullinfo_errsub_bound} for  \textsc{ErrSub} }  \\
& \;\; +\order\rbr{ \E\sbr{ \selfterm_2\rbr{ L^4|S| \ln \pcons  } } }  \tag{from \pref{eq:fullinfo_erropt_bound} for \textsc{ErrOpt} }    \\
& \;\; + \order\rbr{ \E\sbr{  \selfterm_3\rbr{ L^4 \ln\pcons } } }  \tag{from \pref{eq:fullinfo_occdiff_bound} for \textsc{OccDiff}  }  \\
& \;\; +\order\rbr{  \E\sbr{ \selfterm_4\rbr{ L^5|S||A|  \ln \rbr{|S||A|} \ln T } } }   \tag{from \pref{eq:fullinfo_estreg_bound}  for \textsc{EstReg}  }   \\
& \;\; + \order\rbr{L^4 |S|^3 |A|^2  \ln^2\pcons }. 
\end{align*}

Now suppose that Condition~\eqref{eq:loss_condition} holds. For some universal constant $\kappa > 0$, $\Reg_T(\pi^\star)$ is bounded as 
\begin{align*}
\Reg_T(\pi^\star) & \leq  \kappa \cdot \rbr{ \E\sbr{ \selfterm_1\rbr{ L^4|S| \ln\pcons  } }  +  \E\sbr{ \selfterm_2\rbr{ L^4|S| \ln \pcons  } }  + \E\sbr{  \selfterm_3\rbr{ L^4 \ln\pcons } } } \\
& \;\;  + \kappa \cdot \rbr{ \E\sbr{ \selfterm_4\rbr{ L^5|S||A|  \ln \rbr{|S||A|} \ln T }    } } 
+ \kappa \cdot \rbr{L^4 |S|^3 |A|^2  \ln^2\pcons }.
\end{align*}

For any $z>0$, by \pref{lem:self_bounding_term_1}, \pref{lem:self_bounding_term_2}, \pref{lem:self_bounding_term_3} and \pref{lem:self_bounding_term_4} with $\alpha = \beta = \frac{1}{12z\kappa}$  we have 
\begin{align*}
\Reg_T(\pi^\star) & \leq \frac{\Reg_T(\pi^\star) + C}{z} \\
& \;\;  + 12z \cdot  \rbr{  \sum_{s \neq s_L} \sum_{ a\neq \pi^\star(s)} \frac{8\kappa^2}{\gap(s,a)} } \cdot \bigg(  L^4|S| \ln\pcons  +  L^6|S|  \ln\pcons \big)  \\
& \;\;  + 12z \cdot \rbr{ \frac{\kappa^2}{\gapmin} } \cdot \bigg( 8L^5|S|\ln \pcons  +  8L^6|S|^2  \ln\pcons + \frac{L^5|S||A|  \ln \rbr{|S||A|} \ln T}{4}    \bigg) \\
& \;\;  + \kappa \cdot \rbr{L^4 |S|^3 |A|^2  \ln^2\pcons} \\
&\leq  \frac{\Reg_T(\pi^\star) + C}{z} +  288z \kappa^2 \cdot U +  2\kappa \cdot V,
\end{align*}
where the last line uses the shorthands $U$ and $V$ defined in \pref{prop:bobw_fullinfo_stoc_lem}. 

Rearranging the terms arrive at:
\begin{align*}
\Reg_T(\pi^\star)  & \leq \frac{C}{z-1} + \frac{z^2}{z-1} \cdot 288 \kappa^2 U + \frac{z}{z-1} \cdot 2\kappa \cdot V \\
& = \frac{C}{x} + \frac{(x+1)^2}{x} \cdot 288 \kappa^2 U +  \frac{x+1}{x}  \cdot 2\kappa \cdot V  \\
& = \frac{1}{x} \cdot \rbr{ C + 288 \kappa^2 U  +  2\kappa \cdot V   }  + x \cdot  288 \kappa^2 U  + 2\kappa \cdot V + 576 \kappa^2 U
\end{align*}
where we replace all $z$'s by $x = z - 1 > 0$ in the second line. Finally, by selecting the optimal $x$ to balance the first two terms, we have 
\begin{align*}
\Reg_T(\pi^\star) & \leq 2 \sqrt{ \rbr{ C + 288 \kappa^2 U  +  2\kappa \cdot V   } \cdot 288 \kappa^2 U   } + 2\kappa V  + 576 \kappa^2 U \\
& = \order\rbr{ U + \sqrt{ UC } + V    },
\end{align*}
finishing the entire proof for \pref{prop:bobw_fullinfo_stoc_lem}.

\newpage
\section{Best of Both Worlds for MDPs with Unknown Transition and Bandit Feedback}
\label{app:bobw_unknown_transition_bandit}

In this section, we prove the best of both worlds results for the bandit setting with unknown transition. We present the bound for the adversarial world in \pref{prop:bobw_bandit_adv_prop}, and that for the stochastic world in \pref{prop:bobw_bandit_stoc_prop}.
Together, they prove \pref{thm:main_bandit}. 

\begin{proposition}\label{prop:bobw_bandit_adv_prop} With $\delta = \frac{1}{T^3}$, \pref{alg:bobw_framework} ensures 
	\begin{align*}
	\Reg_T(\optpi)= \otil\rbr{ \rbr{  L + \sqrt{A} } |S|\sqrt{|A|T} }.
	\end{align*}
\end{proposition} 
\begin{proposition}
	Suppose Condition~\eqref{eq:loss_condition} holds. With $\delta = \frac{1}{T^3}$, \pref{alg:bobw_framework} ensures that $\Reg_T(\pi^\star)$ is bounded by $\order\rbr{U + \sqrt{CU} + V }$   
	where $V =  L^6|S|^3|A|^3 \ln^2 T$ and $U$ is defined as
	\[
	U =  \sum_{s \neq s_L} \sum_{ a\neq \pi^\star(s)} \sbr{ \frac{L^6|S|\ln T +  L^4|S||A|\ln^2 T}{\gap(s,a)}} + \sbr{\frac{L^6|S|^2 \ln T +  L^3|S|^2|A| \ln^2 T}{\gapmin}}.
	\]
	\label{prop:bobw_bandit_stoc_prop}
\end{proposition} 
The analysis is similar to that for the full-information setting, except that we need to handle some bias terms caused by the new loss estimators. 
To this end, we denote by $\wtill_t$ the conditional expectation of $\hatl_t$, that is 
\begin{equation}
\wtill_t(s,a) = \E_t\sbr{ \hatl_t(s,a)  } = \frac{q_t(s,a) }{ u_t(s,a) } \cdot \ell_t(s,a) - L \cdot B_{i(t)}(s,a). \label{eq:bandit_mean_adj_loss_def}
\end{equation}
Then we define the following:
\begin{definition} For any policy $\pi$, the estimated  state-action and state value functions associated with $\bar{P}_{i(t)}$ and loss function $\wtill_t$ are defined as: 
	\begin{equation}
	\label{eq:bandit_est_Q_def}
	\begin{split}
	\widetilde{Q}_t^{\pi}(s,a) & = \wtill_t(s,a) + \sum_{s' \in S_{k(s)+1} } \bar{P}_{i(t)}(s'|s,a) \widetilde{V}_t^{\pi}(s'), \quad \forall (s,a) \in( S -\cbr{ s_L }) \times A,  \\
	\widetilde{V}_t^{\pi}(s) & = \sum_{a \in A} \pi(a|s) \widetilde{Q}_t^{\pi}(s,a), \quad  \forall s \in S,  \\
	\widetilde{Q}_t^{\pi}(s_L,a) & = 0, \quad  \forall a \in A .
	\end{split}
	\end{equation}
\end{definition}

On the other hand, the true state-action and value functions are again defined as:
\begin{equation}
\label{eq:bandit_true_Q_def}
\begin{split}
Q_t^{\pi}(s,a) & = \ell_t(s,a) + \sum_{s' \in S_{k(s)+1}} P(s'|s,a) V_t^{\pi}(s'), \quad \forall (s,a) \in( S -\cbr{ s_L }) \times A,  \\
V_t^{\pi}(s) & = \sum_{a \in A} \pi(a|s) Q_t^{\pi}(s,a), \quad  \forall s \in S,  \\
Q_t^{\pi}(s_L,a) & = 0, \quad  \forall a \in A .
\end{split}
\end{equation}
where $P$ denotes the true transition function. 


Besides the definition of event $\calA$, we also define $\calA_i$ to be the event  $P \in \calP_i$. 
Importantly, the value of $\Ind{\calA_i}$ is only based on observations prior to epoch $i$.
For notational convenience, we again let $\pcons = \frac{T|S||A|}{\delta}$ and assume $\delta \in \rbr{0,1}$.

Similarly to the full-information setting, we decompose the regret against policy $\pi$, $\Reg(\pi) = \E\sbr{\sum_{t=1}^{T} V_t^{\pi_t}(s_0) - V_t^{\pi}(s_0) }$,  as 
\begin{align}
\E\Biggsbr{\underbrace{\sum_{t=1}^{T} V^{\pi_t}_t(s_0) - \widetilde{V}^{\pi_t}_t(s_0)}_{\textsc{Err}_1 }}  + \E\Biggsbr{\underbrace{\sum_{t=1}^{T} \widetilde{V}^{\pi_t}_t(s_0) - \widetilde{V}^{\pi}_t(s_0) }_{\textsc{EstReg} }}  + \E\Biggsbr{\underbrace{\sum_{t=1}^{T}\widetilde{V}^{\pi}_t(s_0) - V^{\pi}_t(s_0) }_{\textsc{Err}_2}}. \label{eq:bandit_decomposition}
\end{align}
Note that, the second term is exactly
\[
\E\sbr{\textsc{EstReg}} = \E\sbr{\sum_{t=1}^{T}  \inner{ q^{\bar{P}_{i(t)},\pi_t} - q^{\bar{P}_{i(t)},\pi},\wtill_t }  }
= \E\sbr{\sum_{t=1}^{T}  \inner{ q^{\bar{P}_{i(t)},\pi_t} - q^{\bar{P}_{i(t)},\pi },\hatl_t }},
\]
which is controlled by the \ftrl process.

\subsection{Auxiliary Lemmas} 
First, we show the following optimism lemma.


\begin{lemma} \label{lem:bandit_optimism} With the notations defined in \pref{eq:bandit_est_Q_def} and \pref{eq:bandit_true_Q_def}, the following holds conditioning on  event $\calA$:
	\[\widetilde{Q}_t^{\pi}(s,a)\leq Q_t^\pi(s,a), \forall (s,a)\in S\times A, t \in [T].\]
	Specifically, we have 
	\[\inner{q^{\bar{P}_{i(t)},\pi}, \wtill_t }= \widetilde{V}_t^{\pi}(s_0) \leq V_t^{\pi}(s_0) = \inner{q^{P,\pi}, \ell_t }.\]
\end{lemma} 
\begin{proof} We prove this result via a backward induction from layer $L$ to layer $0$. 
	
	\textbf{Base case:} for $s_L$, $\widetilde{Q}_t^{\pi}(s,a) = Q_t^\pi(s,a) = 0$ holds always. 
	
	\textbf{Induction step:} Suppose $\widetilde{Q}_t^{\pi}(s,a)\leq Q_t^\pi(s,a)$ holds for all states $s$ with $k(s) > h$. Then, for any state $s$ with $k(s) = h$, we have 
	\begin{align*} 
	\widetilde{Q}_t^{\pi}(s,a) & = \frac{q_t(s,a)}{u_t(s,a)} \cdot \ell_t(s,a)+ \sum_{s' \in S_{k(s)+1}} \bar{P}_{i(t)}(s'|s,a) \widetilde{V}_t^{\pi}(s') - L \cdot B_{i(t)}(s,a)&& (\text{\pref{eq:bandit_mean_adj_loss_def}}   ) \\
	& \leq \frac{q_t(s,a)}{u_t(s,a)} \cdot \ell_t(s,a) + \sum_{s' \in S_{k(s)+1}} \bar{P}_{i(t)}(s'|s,a) V_t^{\pi}(s') - L \cdot B_{i(t)}(s,a) && (\text{induction  hypothesis}) \\
	& \leq \frac{q_t(s,a)}{u_t(s,a)} \cdot \ell_t(s,a) + \sum_{s' \in S_{k(s)+1}} P(s'|s,a) V_t^{\pi}(s') \\
	& \quad + \sum_{s' \in S_{k(s)+1}} \rbr{ \bar{P}_{i(t)}(s'|s,a)  - P(s'|s,a) } V_t^{\pi}(s') - L\cdot B_{i(t)}(s,a) \\
	& \leq \frac{q_t(s,a)}{u_t(s,a)} \cdot \ell_t(s,a) + \sum_{s' \in S_{k(s)+1}} P(s'|s,a) V_t^{\pi}(s') \\
	& \leq  \ell_t(s,a) + \sum_{s' \in S_{k(s)+1}} P(s'|s,a) V^{\pi}(s') 
	= Q_t^\pi(s,a),
	\end{align*}
	where the forth step follows from the same arguments in \pref{lem:full_info_optimism}, and the last step holds since under event $\calA$, we have $q_t(s,a) \leq u_t(s,a)$ by the definition of $u_t$.
	This finishes the induction.

\end{proof}


Next, we provide a sequence of boundedness results, useful for regret analysis.
\begin{lemma}[Lower Bound of Upper Occupancy Bound]\label{lem:lower_bound_uob} \pref{alg:bobw_framework} ensures $u_t(s) \geq \frac{1}{|S|t}$ for all $t$ and $s$.
\end{lemma}
\begin{proof} We prove by constructing a special transition function $\widehat{P}_{i(t)}$ within the confidence set $\calP_{i(t)}$, which ensures $q^{\bar{P}_{i(t)}, \pi_t}(s) \geq \frac{1}{|S|t}$ for all state-action pairs. 
	Specifically, let $\widehat{P}_{i(t)}$ be such that
	\begin{align*}
	\widehat{P}_{i(t)}(s'|s,a) = \frac{1}{t} \cdot \frac{1}{|S_{k(s)+1}|}  + \frac{t-1}{t} \cdot \bar{P}_{i(t)}(s'|s,a), \quad \forall (s,a,s') \in T_k , k < L. 
	\end{align*}
	Clearly, $\widehat{P}_{i(t)}(\cdot |s,a)$ is a valid transition distribution over $S_{k(s)+1}$ for all state-action pairs. 
	Then, we prove that $\widehat{P}_{i(t)} \in \calP_i$ by 
	\begin{align*}
	\abr{  \widehat{P}_{i(t)}(s'|s,a) - \bar{P}_{i(t)}(s'|s,a) } = \frac{1}{t} \cdot \abr{  \bar{P}_{i(t)}(s'|s,a) -  \frac{1}{|S_{k(s)+1}|}   } \leq \frac{1}{t} \leq \frac{14\ln \rbr{\frac{T|S||A|}{\delta}}}{3\max\cbr{m_{i(t)(s,a)},1} }
	\end{align*} 
	where the last inequality follows from the fact that $m_{i(t)}(s,a) \leq t$. 
	
	Then, for any state $s\neq s_0$, we have by the definition of occupancy measures
	\begin{align*}
	q^{\widehat{P}_{i(t)}, \pi_t}(s) & = \sum_{s' \in S_{k(s)-1}} \sum_{a' \in A} q^{\widehat{P}_{i(t)}, \pi_t}(s', a') \cdot \widehat{P}_{i(t)}(s|s',a') \\
	& \geq \sum_{s' \in S_{k(s)-1}} \sum_{a' \in A} q^{\widehat{P}_{i(t)}, \pi_t}(s',a')  \cdot \frac{1}{ \abr{S_{k(s)}} t  } \\
	& =  \frac{1}{ \abr{S_{k(s)}} t  }  \geq \frac{1}{|S|t}
	\end{align*}
	 Clearly, for $s_0$ it holds that $q^{\widehat{P}_{i(t)}, \pi_t}(s_0) = 1 \geq \nicefrac{1}{|S|t}$, which finishes the proof. 
\end{proof}

\begin{corollary} \label{col:bandit_hatl_bound}  \pref{alg:bobw_framework} ensures that, the adjusted loss $\hatl_t$ defined in \pref{eq:adjusted_loss} for bandit-feedback is bounded as:
\begin{equation*}
\abr{ \hatl_t(s,a) } \leq L + \frac{\Indt{s,a}}{q_t(s,a)} \cdot |S|t.  
\end{equation*}
Also, we have 
\begin{equation*}
\E\sbr{ \left. \frac{\Indt{s,a}}{q_t(s,a)}  \right\rvert  \calA_{i(t)} } =  \E\sbr{ \left.  \frac{\Indt{s,a}}{q_t(s,a)}  \right\rvert  \calA_{i(t)}^c  } = 1.
\end{equation*} 
\end{corollary}
\begin{proof} By \pref{lem:lower_bound_uob}, we have
\begin{align*}
\abr{ \hatl_t(s,a) } \leq \frac{\Indt{s,a} }{u_t(s) \cdot \pi_t(a|s)} + L \leq \frac{\Indt{s,a} }{q_t(s) \cdot \pi_t(a|s)} \cdot |S|t  + L  = L + \frac{\Indt{s,a}}{q_t(s,a)} \cdot |S|t,
\end{align*}
where the first inequality follows from $B_i(s,a)\leq 1$ and $\ell_t(s,a) \leq 1$, and the second inequality uses \pref{lem:lower_bound_uob} and the fact $q_t(s) \leq 1$. 

For the second statement, we have 
\begin{align*}
\E\sbr{ \left. \frac{\Indt{s,a}}{q_t(s,a)}  \right\rvert  \calA_{i(t)} }  = \E\Biggsbr{  \E_t\sbr{ \left. \frac{\Indt{s,a}}{q_t(s,a)}} \right\rvert  \calA_{i(t)} }  = \E\Bigsbr{ \left. 1  \right\rvert  \calA_{i(t)} } = 1, 
\end{align*}
By the same arguments we can prove $\E\sbr{ \left. \frac{\Indt{s,a}}{q_t(s,a)}  \right\rvert  \calA_{i(t)}^c } = 1$ as well.
\end{proof}

\begin{lemma} \label{lem:bound_on_wtill}
\pref{alg:bobw_framework} ensures that, the expected adjusted loss $\wtill_t$ defined in \pref{eq:bandit_mean_adj_loss_def} is bounded as: 
	\begin{equation*}
	\abr{ \wtill_t(s,a) } \leq L + |S|\cdot t \leq 2|S|\cdot t,  \quad \forall (s,a) \in S\times A, t \in [T].
	\end{equation*}
\end{lemma}
\begin{proof} By \pref{eq:bandit_mean_adj_loss_def}, we know that  
	\begin{align*}
	\abr{ \wtill_t(s,a) } = \abr{  \frac{q_t(s,a)}{u_t(s,a)} \cdot \ell_t(s,a) - L \cdot B_{i(t)}(s,a) } \leq \frac{q_t(s)}{u_t(s)} + L  \leq L + |S|\cdot t 
	\end{align*}
	where the last inequality follows from \pref{lem:lower_bound_uob}. Combining with the fact $|S|\geq L$ finishes the proof. 
\end{proof}
\begin{corollary} \label{col:bandit_est_Q_bound}
	\pref{alg:bobw_framework} ensures that, the estimated state-action value functions defined in \pref{eq:bandit_est_Q_def} are bounded as:
	\begin{align*}
	\abr{ \widetilde{Q}_t^\pi(s,a) } \leq 2L|S|t, \quad \forall (s,a) \in S\times A, t \in [T].
	\end{align*} 
\end{corollary}

\begin{proof}
This is directly by \pref{lem:bound_on_wtill} and the definition of $\widetilde{Q}_t^\pi(s,a)$.
\end{proof}

Next, we analyze the estimated regret in each epoch.
Reloading the notation from the full-information setting,
we define
\begin{equation*}
\EReg_i(\pi) = \E\sbr{ \sum_{t=t_i}^{t_{i+1}-1} \inner{ q^{\bar{P}_i, \pi_t } - q^{\bar{P}_i, \pi }, \hatl_t}} = \E\sbr{ \sum_{t=t_i}^{t_{i+1}-1} \inner{ \whatq_t - q^{\bar{P}_i, \pi }, \hatl_t}}.
\end{equation*}

\begin{lemma}\label{lem:bobw_bandit_tsallis_reg}
 With $\beta = 128L^4$, for any epoch $i$, \pref{alg:bobw_framework} ensures 
	\begin{equation}
	\begin{aligned}
	\EReg_i(\pi) 
	& \leq \order\rbr{ \E\sbr{ \sum_{t=t_i}^{t_{i+1}-1} \eta_t \cdot \rbr{  \sqrt{L|S||A|} + L^2 \sum_{s \neq s_L} \sum_{a\in A} \whatq_t(s,a) \cdot B_{i(t)}(s,a)^2} }  } \\ 
	& \quad + \order\rbr{ L^4|S||A|\log T + \delta \cdot \E\sbr{ L|S|T\rbr{ t_{i+1} - t_i   } } },
	\end{aligned}
	\label{eq:bobw_bandit_tsallis_reg_worstcase}
	\end{equation}
	for any policy $\pi$, and simultaneously
	\begin{equation}
	\begin{aligned}
	\EReg_i(\pi)  & \leq \order\rbr{ \E\sbr{ \sqrt{L|S|}  \sum_{t=t_i}^{t_{i+1}-1} \eta_t \cdot \sqrt{\sum_{s  \neq s_L}\sum_{a\neq \pi(s)} \whatq_t(s,a)  } } } \\
	& \quad + \order\rbr{ L^2 \cdot \E\sbr{  \sum_{t=t_i}^{t_{i+1}-1} \eta_t \cdot \sum_{s \neq s_L} \sum_{a\neq \pi(s)} \sqrt{ \whatq_t(s,a) } } } \\
	& \quad + \order\rbr{L^4  |A| \cdot \E\sbr{  \sum_{t=t_i}^{t_{i+1}-1} \eta_t \cdot  \sum_{s \neq s_L } \sum_{a \in A}  \whatq_t(s,a) \cdot B_{i(t)}(s,a)^2    } } \\ 
	& \quad + \order\rbr{ L^4|S||A|\log T + \delta \cdot \E\sbr{ L|S|T\rbr{ t_{i+1} - t_i   } } },
	\end{aligned}
	\label{eq:bobw_bandit_tsallis_reg_adaptive}
	\end{equation}
	for any deterministic policy $\pi: S\rightarrow A$.
\end{lemma}

\begin{proof} The proof is largely based on that of \pref{thm:known_transition_tsallis_entropy_main}, but with some careful treatments based one whether $\calA_i$ holds or not.
Let $q = q^{\bar{P}_{i}, \pi}$ be the occupancy measure we want to compete against.
When $\calA_i$ does not hold, we first derive the following naive bound on 
$\sum_{t=t_i}^{t_{i+1}-1} \inner{ \whatq_t - q ,\hatl_t  }$:
	\begin{align*}
	& \sum_{t=t_i}^{t_{i+1}-1} \inner{ \whatq_t - q, \hatl_t  } \leq \sum_{t=t_i}^{t_{i+1}-1} \sum_{s \neq s_L} \sum_{a\in A} \rbr{\whatq_t(s,a) + q(s,a)} \cdot  \abr{\hatl_t(s,a)}  \\
	& \leq \sum_{t=t_i}^{t_{i+1}-1} \sum_{s \neq s_L} \sum_{a\in A} \rbr{ \whatq_t(s,a) + q(s,a) }\cdot \rbr{ L +  \frac{\Indt{s,a} }{u_t(s,a) }\cdot |S|t } \tag{\pref{col:bandit_hatl_bound}} \\
	&\leq 2L^2 \cdot \rbr{ t_{i+1} - t_i   } + |S|T \cdot \sum_{t=t_i}^{t_{i+1}-1} \sum_{s \neq s_L} \sum_{a\in A} \rbr{ \whatq_t(s,a) + q(s,a) } \cdot \frac{\Indt{s,a} }{q_t(s,a) }.
	\end{align*} 
	Therefore, we have the conditional expectation $\E\sbr{ \left. \sum_{t=t_i}^{t_{i+1}-1}\inner{ \whatq_t - q, \hatl_t  } \right\rvert \calA_i^c}$ bounded by
	\begin{align*}
	& \E\sbr{ \left.  2L^2 \cdot \rbr{ t_{i+1} - t_i   } + |S|t \cdot \sum_{t=t_i}^{t_{i+1}-1} \sum_{s \neq s_L} \sum_{a\in A} \rbr{ \whatq_t(s,a) + q(s,a) } \cdot \frac{\Indt{s,a} }{q_t(s,a) } \right\rvert \calA_i^c} \\
	& \leq \E\sbr{ \left.  ( 2L^2 + 2L|S|T) \cdot \rbr{ t_{i+1} - t_i   }   \right\rvert \calA_i^c} \tag{\pref{col:bandit_hatl_bound}} \\
	& \leq \order\Bigrbr{  \E\sbr{ \left.  L|S|T \cdot \rbr{ t_{i+1} - t_i   } \right\rvert \calA_i^c}    }.
	\end{align*}

Next, we condition on event $\calA_i$.
In this case, by the same argument as \citep[Lemma~5]{jin2020simultaneously}
and also our loss-shifting technique, \pref{alg:bobw_framework} with $\beta = 128L^4$ ensures that 
	$\sum_{t=t_i}^{t_{i+1}-1} \inner{ \whatq_t - q, \hatl_t  }$ is bounded by 
	\begin{align}
	& \order\rbr{L^4|S||A|\log T} + 
	 \sum_{t=t_i+1}^{t_{i+1}-1}\rbr{ \frac{1}{\eta_t} - \frac{1}{\eta_{t-1}} } \rbr{ \phi_{H}(q) - \phi_H(\whatq_t)  } \nonumber \\
	&\quad + 8 \sum_{t=t_i}^{t_{i+1}-1} \eta_t \min\cbr{ \sum_{s \neq s_L} \sum_{a\in A} \whatq_t(s,a)^{\nicefrac{3}{2}} \rbr{ \widehat{Q}_t(s,a) - \widehat{V}_t(s)  }^2,   \sum_{s \neq s_L} \sum_{a\in A} \whatq_t(s,a)^{\nicefrac{3}{2}} \hatl_t(s,a)^2 }  \label{eq:reg_under_A}
	\end{align}
	where $\phi_{H}(q) = - \sum_{s \neq s_L}\sum_{a \in A} \sqrt{q(s,a)}$,
	and $\widehat{Q}_t$ and $\widehat{V}_t$ are state-action and state value functions associated with the loss estimator $\hatl_t$ and the empirical transition $\bar{P}_{i(t)}$:
	\[
	\widehat{Q}_t(s,a) = \hatl_t(s,a) + \sum_{s'\in S_{k(s)+1}} \bar{P}_{i(t)}(s'|s,a) \widehat{V}_t(s'), \quad \widehat{V}_t(s) = \sum_{a\in A} \pi_t(a|s) \widehat{Q}_t(s,a). 
	\]
	
Below, we discuss how to proceed from here to prove \pref{eq:bobw_bandit_tsallis_reg_worstcase} and \pref{eq:bobw_bandit_tsallis_reg_adaptive} respectively.

\paragraph{Proving \pref{eq:bobw_bandit_tsallis_reg_worstcase}}

In this case, we take the second argument of the min operator from \pref{eq:reg_under_A} and bound $\phi_{H}(q) - \phi_H(\whatq_t) \leq \sum_{s \neq s_L}\sum_{a \in A} \sqrt{\whatq_t(s,a)} $ trivially by $\sqrt{L|S||A|}$ using Cauchy-Schwarz inequality, leading to
	\begin{align*}
	&\sum_{t=t_i}^{t_{i+1}-1} \inner{ \whatq_t - q, \hatl_t  } \\
	& \leq  \order\rbr{L|S||A|\log T}+ \sqrt{L|S||A|} \cdot \sum_{t=t_i}^{t_{i+1}-1} \eta_t + 8 \sum_{t=t_i}^{t_{i+1}-1} \eta_t \cdot \sum_{s \neq s_L} \sum_{a\in A} \whatq_t(s,a)^{\nicefrac{3}{2}}\hatl_t(s,a)^2  \tag{$\frac{1}{\eta_t} - \frac{1}{\eta_{t-1}} \leq \eta_t$ since $\frac{1}{\eta_t} = \sqrt{t - t_i +1}$}\\
	& \leq \order\rbr{L|S||A|\log T}  + 2 \sqrt{L|S||A|} \cdot \sum_{t=t_i}^{t_{i+1}-1} \eta_t + 16 \sum_{t=t_i}^{t_{i+1}-1} \eta_t \cdot \sum_{s \neq s_L} \sum_{a\in A} \frac{\whatq_t(s,a)^{\nicefrac{3}{2}} \cdot \Indt{s,a} }{u_t(s,a)^2} \\
	& \quad + 16L^2 \sum_{t=t_i}^{t_{i+1}-1} \eta_t \cdot \sum_{s \neq s_L} \sum_{a\in A} \whatq_t(s,a)^{\nicefrac{3}{2}} \cdot B_{i(t)}(s,a)^2   \\
	& \leq \order\rbr{L|S||A|\log T}  + 2 \sqrt{L|S||A|} \cdot \sum_{t=t_i}^{t_{i+1}-1} \eta_t + 16 \sum_{t=t_i}^{t_{i+1}-1} \eta_t \cdot \sum_{s \neq s_L} \sum_{a\in A}  \frac{\sqrt{\whatq_t(s,a)} \cdot \Indt{s,a} }{q_t(s,a)}  \\
	& \quad + 16 L^2 \sum_{t=t_i}^{t_{i+1}-1} \eta_t \cdot \sum_{s \neq s_L} \sum_{a\in A} \whatq_t(s,a) \cdot B_{i(t)}(s,a)^2   
	\end{align*}
	where the second step follows from the definition of $\hatl_t$ in \pref{eq:adjusted_loss} and the last step follows from the fact $\whatq_t(s,a) \leq u_t(s,a)$ and $q_t(s,a) \leq u_t(s,a)$ since $\bar{P}_i, P \in \calP_i$ according to event $\calA_i$. 
	
	Therefore, by \pref{lem:exp_high_prob_bound_cond} we have for any policy $\pi$ that, 
	\begin{align*}
	\E\sbr{ \EReg_i(\pi) } & \leq  \E\sbr{2\sqrt{L|S||A|} \cdot \sum_{t=t_i}^{t_{i+1}-1} \eta_t + 16 \sum_{t=t_i}^{t_{i+1}-1} \eta_t \cdot \sum_{s \neq s_L} \sum_{a\in A} \frac{\sqrt{\whatq_t(s,a)} \cdot \Indt{s,a} }{q_t(s,a)}} \\
	& \quad + \E\sbr{ 16 L^2\sum_{t=t_i}^{t_{i+1}-1} \eta_t \cdot \sum_{s \neq s_L} \sum_{a\in A} \whatq_t(s,a) \cdot B_{i(t)}(s,a)^2 } \\    
	& \quad + \order\rbr{L^4|S||A|\log T + \delta \cdot \E\sbr{ L|S|T\rbr{ t_{i+1} - t_i   } }  } \\
	& \leq  \order\rbr{ \E\sbr{ \sqrt{L|S||A|} \cdot \sum_{t=t_i}^{t_{i+1}-1} \eta_t } + \E\sbr{ L^2\sum_{t=t_i}^{t_{i+1}-1} \eta_t \cdot \sum_{s \neq s_L} \sum_{a\in A} \whatq_t(s,a) \cdot B_{i(t)}(s,a)^2}  } \\ 
	& \quad + \order\rbr{ L^4|S||A|\log T + \delta \cdot \E\sbr{ L|S|T\rbr{ t_{i+1} - t_i   } } }
	\end{align*}
	where the second step takes the conditional expectation of $\Indt{s,a}$ and applies the Cauchy-Schwarz inequality to get $\sum_{s\neq s_L}\sum_{a\in A} \sqrt{\whatq_t(s,a)} \leq \sqrt{L|S||A|}$. This finishes the proof of  \pref{eq:bobw_bandit_tsallis_reg_worstcase}.

\paragraph{Proving \pref{eq:bobw_bandit_tsallis_reg_adaptive}}
	
	In this case, recall that $\pi$ is a deterministic policy, so that
	\begin{align*}
	  \phi_{H}(q) - \phi_H(\whatq_t)  
	&= \sum_{s\neq s_L} \sqrt{\whatq_t(s)}\rbr{\sum_{a\in A}\sqrt{\pi_t(a|s)} - 1} + \sum_{s\neq s_L}\rbr{\sqrt{\whatq_t(s)} - \sqrt{q(s)}}.
	\end{align*}
	Using~\citep[Lemma~16]{jin2020simultaneously} to bound the first term (take $\alpha$ in their lemma to be $0$), and~~\citep[Lemma~19]{jin2020simultaneously} to bound the second, we obtain
	\begin{align*}
	\phi_{H}(q) - \phi_H(\whatq_t)  
	=
	\sum_{s \neq s_L}\sum_{a\neq \pi(s)}\sqrt{\whatq_t(s,a)}
	+ \sqrt{L|S| \sum_{s  \neq s_L}\sum_{a\neq \pi(s)} \whatq_t(s,a)  }.
	\end{align*}
	
	Therefore, taking the first argument of the min operator from \pref{eq:reg_under_A} and using $\frac{1}{\eta_t} - \frac{1}{\eta_{t-1}} \leq \eta_t$ again,
	we arrive at 
	\begin{equation}
	\begin{aligned}
	\sum_{t=t_i}^{t_{i+1}-1} \inner{ \whatq_t - q, \hatl_t  }  & \leq \sqrt{L|S|}  \sum_{t=t_i}^{t_{i+1}-1} \eta_t \cdot \sqrt{\sum_{s  \neq s_L}\sum_{a\neq \pi(s)} \whatq_t(s,a)  } \\
	&\quad +  \sum_{t=t_i}^{t_{i+1}-1}  \eta_t \cdot \sum_{s \neq s_L}\sum_{a\neq \pi(s)}\sqrt{\whatq_t(s,a)}\\
	& \quad  + 8 \sum_{t=t_i}^{t_{i+1}-1} \eta_t \cdot \sum_{s \neq s_L} \sum_{a\in A} \whatq_t(s,a)^{\nicefrac{3}{2}}\rbr{ \widehat{Q}_t(s,a) - \widehat{V}_t(s)  }^2 \\
	& \quad + \order\rbr{L^4|S||A|\log T}. 
	\end{aligned}
	\label{eq:tsallis_adaptive_regret_decomp}
	\end{equation}
	
	Finally, we apply \pref{lem:shifted_stability} to bound the term $\sum_{s \neq s_L}\sum_{a \in A} \whatq_t(s,a)^{\nicefrac{3}{2}}\rbr{ \widehat{Q}_t(s,a) - \widehat{V}_t(s)  }^2$, and use \pref{lem:exp_high_prob_bound_cond} again to take expectation and arrive at \pref{eq:bobw_bandit_tsallis_reg_adaptive} (with the help of \pref{eq:O+W}).
\end{proof}
	
\begin{lemma}\label{lem:shifted_stability}
Under event $\calA$, we have for any $t$,
\begin{align*}
&\sum_{s \neq s_L}\sum_{a \in A} \whatq_t(s,a)^{\nicefrac{3}{2}}\rbr{ \widehat{Q}_t(s,a) - \widehat{V}_t(s)  }^2 \\
&\leq 
4 L^4 |A|  \sum_{s' \neq s_L } \sum_{a' \in A} \widehat{q}_t(s',a') \cdot B_{i(t)}(s',a')^2
 + \sum_{s \neq s_L}\sum_{a \in A} \sqrt{ \whatq_t(s,a) } \cdot \rbr{O_t(s,a) + W_t(s,a) }
\end{align*}
where 
\begin{align*}
O_t(s,a) &= 4 L  \cdot \rbr{ 1 - \pi_t(a|s) }   \sum_{k = k(s)}^{L-1} \sum_{s' \in S_k} \sum_{a' \in A} \widehat{q}_t(s',a'|s,a) \frac{\Indt{s',a'}}{ q_t(s',a')  }, \\
W_t(s,a) &= 4 L \cdot \sum_{b \neq a}\pi_t(b|s)\sum_{k = k(s)}^{L-1} \sum_{s' \in S_k} \sum_{a' \in A}  \widehat{q}_t(s',a'|s,b) \frac{\Indt{s',a'}}{q_t(s',a')}, \\
\end{align*}
and $\widehat{q}_t(s',a'|s,a)$ is the probability of visiting $(s',a')$ starting from $(s,a)$ under $\pi_t$ and $\bar{P}_{i(t)}$.
Moreover, we have 
\begin{equation}\label{eq:O+W}
\E_t\sbr{\sum_{s \neq s_L}\sum_{a \in A} \sqrt{ \whatq_t(s,a) } \cdot \rbr{O_t(s,a) + W_t(s,a) }} \leq 16 L^2 \sum_{s \neq s_L}\sum_{a \neq \pi(s)} \sqrt{\whatq_t(s,a)},
\end{equation}
for any mapping $\pi: S \rightarrow A$.
\end{lemma}	

\begin{proof}	
	
	First, $\rbr{\widehat{Q}_t(s,a) - \widehat{V}_t(s)}^2$ is bounded by
	\begin{align*}
	\rbr{ \widehat{Q}_t(s,a) - \widehat{V}_t(s) }^2 & = \rbr{ \rbr{ 1 - \pi_t(a|s) } \widehat{Q}_t(s,a)  - \rbr{ \sum_{b\neq a}  \pi_t(b|s) \widehat{Q}_t(s,b)   }      }^2 \\
	& \leq 2 \rbr{ 1 - \pi_t(a|s) }^2  \widehat{Q}_t(s,a)^2 + 2 \rbr{ \sum_{b\neq a} \pi_t(b|s)  \widehat{Q}_t(s,b)   }^2.
	\end{align*}
	
	Following the same idea of \pref{lem:loss_shifting_loss_estimator}, the first term can be bounded as 
	\begin{align}
	& \rbr{ 1 - \pi_t(a|s) }^2  \widehat{Q}_t(s,a)^2 \notag \\
	& = \rbr{ 1 - \pi_t(a|s) }^2 \rbr{  \sum_{k = k(s)}^{L-1} \sum_{s' \in S_k} \sum_{a' \in A}  \widehat{q}_t(s',a'|s,a) \hatl_t(s',a')  }^2  \notag \\
	& \leq 2 \rbr{ 1 - \pi_t(a|s) }^2  \rbr{   \sum_{k = k(s)}^{L-1} \sum_{s' \in S_k} \sum_{a' \in A}  \widehat{q}_t(s',a'|s,a) \frac{\Indt{s',a'}}{ u_t(s',a')  } \cdot \ell_t(s',a') }^2 \notag \\
	& \quad + 2 \rbr{ 1 - \pi_t(a|s) }^2  \rbr{   \sum_{k = k(s)}^{L-1} \sum_{s' \in S_k} \sum_{a' \in A}   \widehat{q}_t(s',a'|s,a) \cdot L \cdot B_{i(t)}(s',a') }^2 \notag \\
	& \leq 2 L \cdot \rbr{ 1 - \pi_t(a|s) }^2   \sum_{k = k(s)}^{L-1} \sum_{s' \in S_k} \sum_{a' \in A}  \widehat{q}_t(s',a'|s,a)^2 \cdot \frac{\Indt{s',a'}}{ u_t(s',a')^2  } \notag \\
	& \quad + 2L^{3} \rbr{ 1 - \pi_t(a|s) }^2 \sum_{k = k(s)}^{L-1} \sum_{s' \in S_k} \sum_{a' \in A}   \widehat{q}_t(s',a'|s,a)   \cdot B_{i(t)}(s',a')^2 \label{eq:bandit_self_boud_term1}
	\end{align}
	where the equality follows from the definition of $\widehat{Q}_t$; the first inequality uses the fact  $(x+y)^2 \leq 2(x^2 + y^2)$; the second inequality applies the Cauchy-Schwarz inequality with the facts $\Indt{s,a} \Indt{s',a'} = 0$ for $(s,a) \neq (s',a')$ and $\sum_{k = k(s)}^{L-1} \sum_{s' \in S_k} \sum_{a' \in A} \widehat{q}_t(s',a'|s,a) \leq L$.
	
	By the same arguments, the second term is bounded as 
	\begin{align}
	& \rbr{ \sum_{ b\neq a} \pi_t(b|s)  \widehat{Q}_t(s,b) }^2 \notag \\
	& = \rbr{ \sum_{k = k(s)}^{L-1} \sum_{s' \in S_k} \sum_{a' \in A}  \rbr{ \sum_{ b\neq a} \pi_t(b|s) \widehat{q}_t(s',a'|s,b)}   \hatl_t(s,a) }^2 \notag \\
	& \leq 2 L  \cdot \sum_{k = k(s)}^{L-1} \sum_{s' \in S_k} \sum_{a' \in A}  \rbr{ \sum_{b \neq a} \pi_t(b|s)\cdot\widehat{q}_t(s',a'|s,b)}^2 \cdot \frac{\Indt{s',a'}}{ u_t(s',a')^2  } \notag\\
	& \quad  + 2L^3 \sum_{k = k(s)}^{L-1} \sum_{s' \in S_k} \sum_{a' \in A}  \rbr{ \sum_{b \neq a} \pi_t(b|s)\cdot \widehat{q}_t(s',a'|s,b)}  \cdot B_{i(t)}(s',a')^2, \label{eq:bandit_self_boud_term2}
	\end{align}
	where in the last step we use $\sum_{k = k(s)}^{L-1} \sum_{s' \in S_k} \sum_{a' \in A}  \rbr{ \sum_{b \neq a} \pi_t(b|s)\cdot \widehat{q}_t(s',a'|s,b)} \leq L$ (after applying Cauchy-Schwarz).
	
	Combining \pref{eq:bandit_self_boud_term1} and \pref{eq:bandit_self_boud_term2}, 
	we show that $\whatq_t(s,a) \rbr{ \widehat{Q}_t(s,a) - \widehat{V}_t(s)  }^2$ can be bounded as 
	\begin{align*}
	& \whatq_t(s,a) \rbr{ \widehat{Q}_t(s,a) - \widehat{V}_t(s)  }^2 \\
	& \leq 4 L  \cdot \whatq_t(s,a) \rbr{ 1 - \pi_t(a|s) }^2   \sum_{k = k(s)}^{L-1} \sum_{s' \in S_k} \sum_{a' \in A}  \widehat{q}_t(s',a'|s,a)^2 \cdot \frac{\Indt{s',a'}}{ u_t(s',a')^2  }  \\
	& \quad + 4 L \cdot \whatq_t(s,a) \sum_{k = k(s)}^{L-1} \sum_{s' \in S_k} \sum_{a' \in A}  \rbr{ \sum_{b \neq a} \pi_t(b|s)\cdot\widehat{q}_t(s',a'|s,b)}^2 \cdot \frac{\Indt{s',a'}}{ u_t(s',a')^2  } \\
	& \quad + 4L^3\whatq_t(s,a)  \rbr{ 1 - \pi_t(a|s) }^2 \sum_{k = k(s)}^{L-1} \sum_{s' \in S_k} \sum_{a' \in A}   \widehat{q}_t(s',a'|s,a)  \cdot B_{i(t)}(s',a')^2 \\
	& \quad + 4L^3 \whatq_t(s,a)  \sum_{k = k(s)}^{L-1}  \sum_{s' \in S_k} \sum_{a' \in A}   \rbr{ \sum_{b \neq a} \pi_t(b|s)\cdot \widehat{q}_t(s',a'|s,b)} \cdot B_{i(t)}(s',a')^2.
	\end{align*}
	
	Moreover, we have the summation of the first two terms bounded as 
	\begin{align*}
	& 4 L  \cdot \whatq_t(s,a) \rbr{ 1 - \pi_t(a|s) }^2   \sum_{k = k(s)}^{L-1} \sum_{s' \in S_k} \sum_{a' \in A}  \widehat{q}_t(s',a'|s,a)^2 \cdot \frac{\Indt{s',a'}}{ u_t(s',a')^2  }  \\
	& \quad + 4 L \cdot \whatq_t(s,a) \sum_{k = k(s)}^{L-1} \sum_{s' \in S_k} \sum_{a' \in A}  \rbr{ \sum_{b \neq a} \pi_t(b|s)\cdot\widehat{q}_t(s',a'|s,b)}^2 \cdot \frac{\Indt{s',a'}}{ u_t(s',a')^2  } \\
	& \leq 4 L  \cdot \rbr{ 1 - \pi_t(a|s) }^2   \sum_{k = k(s)}^{L-1} \sum_{s' \in S_k} \sum_{a' \in A}  \frac{\whatq_t(s,a) \widehat{q}_t(s',a'|s,a)}{u_t(s',a')} \cdot \widehat{q}_t(s',a'|s,a) \frac{\Indt{s',a'}}{ q_t(s',a')  }  \\
	& \quad + 4 L \cdot \sum_{k = k(s)}^{L-1} \sum_{s' \in S_k} \sum_{a' \in A}  \frac{ \sum_{b \neq a} \whatq_t(s,b)\cdot\widehat{q}_t(s',a'|s,b)}{u_t(s',a')} \cdot \rbr{ \sum_{b \neq a}  \pi_t(b|s)\cdot\widehat{q}_t(s',a'|s,b) \frac{\Indt{s',a'}}{q_t(s',a')}} \\
	& \leq  O_t(s,a) + W_t(s,a)
	\end{align*}
	where we use $q_t(s',a') \leq u_t(s',a')$ due to event $\calA_{i}$ in the first step and $\sum_{a \in A}\whatq_t(s,a) \whatq_t(s',a'|s,a) \leq \whatq_t(s',a') \leq u_t(s',a')$ in the second step to bound the fractions by $1$. 
	

	On the other hand, the summation of the other two terms is  bounded as
	\begin{align*}
	& 4L^3\whatq_t(s,a)  \rbr{ 1 - \pi_t(a|s) }^2 \sum_{k = k(s)}^{L-1} \sum_{s' \in S_k} \sum_{a' \in A}   \widehat{q}_t(s',a'|s,a)  \cdot B_{i(t)}(s',a')^2 \\
	& \quad + 4L^3 \whatq_t(s,a)  \sum_{k = k(s)}^{L-1}  \sum_{s' \in S_k} \sum_{a' \in A}   \rbr{ \sum_{b \neq a} \pi_t(b|s)\cdot \widehat{q}_t(s',a'|s,b)} \cdot B_{i(t)}(s',a')^2 \\ 
	& \leq  4 L^3\whatq_t(s)\sum_{k = k(s)}^{L-1}  \sum_{s' \in S_k} \sum_{a' \in A}  \rbr{ \widehat{q}_t(s',a'|s,a) \pi_t(a|s) +  \sum_{b \neq a} \pi_t(b|s)\cdot \widehat{q}_t(s',a'|s,b)  } \cdot B_{i(t)}(s',a')^2 \\
	& = 4 L^3 \sum_{k = k(s)}^{L-1}  \sum_{s' \in S_k} \sum_{a' \in A}  \widehat{q}_t(s',a'|s) \widehat{q}_t(s) \cdot B_{i(t)}(s',a')^2.
	\end{align*}
	
	Note that, taking the summation of the last bound over all state-action pairs yields  
	\begin{align*}
	& 4 L^3 \sum_{s\neq s_L }\sum_{a \in A} \sum_{k = k(s)}^{L-1}  \sum_{s' \in S_k} \sum_{a' \in A}  \widehat{q}_t(s',a'|s) \widehat{q}_t(s) \cdot B_{i(t)}(s',a')^2 \\
	& =  4L^3|A| \sum_{s' \neq s_L } \sum_{a' \in A}  \rbr{  \sum_{k = 0}^{k(s') - 1} \sum_{s \in S_k} \widehat{q}_t(s',a'|s) \widehat{q}_t(s)  } \cdot B_{i(t)}(s',a')^2  \\
	& \leq 4L^4 |A| \sum_{s' \neq s_L } \sum_{a' \in A} \widehat{q}_t(s',a') \cdot B_{i(t)}(s',a')^2.
	\end{align*}	
	
	Therefore, combining everything, we have shown:
	\begin{align*}
	&  \sum_{s \neq s_L}\sum_{a\neq \pi(s)} \whatq_t(s,a)^{\nicefrac{3}{2}}\rbr{ \widehat{Q}_t(s,a) - \widehat{V}_t(s)  }^2  \\
	& \leq  4L^4 |A|  \sum_{s' \neq s_L } \sum_{a' \in A} \widehat{q}_t(s',a') \cdot B_{i(t)}(s',a')^2 + \sum_{s \neq s_L}\sum_{a\neq \pi(s)} \sqrt{ \whatq_t(s,a) } \cdot \rbr{O_t(s,a) + W_t(s,a) } , 
	\end{align*}
	proving the first statement of the lemma.

	To prove the second statement, we first show
	\begin{align*}
	\E_t\sbr{O_t(s,a) + W_t(s,a)}&= \  4 L \rbr{ 1 - \pi_t(a|s) }  \cdot \sum_{k = k(s)}^{L-1} \sum_{s' \in S_k} \sum_{a' \in A}  \widehat{q}_t(s',a'|s,a)    \\
	& \quad +   4 L \cdot \sum_{k = k(s)}^{L-1} \sum_{s' \in S_k} \sum_{a' \in A}  \rbr{ \sum_{b \neq a} \pi_t(b|s)\cdot\widehat{q}_t(s',a'|s,b)}   \\
	& = \  4 L \rbr{ 1 - \pi_t(a|s) }   \sum_{k = k(s)}^{L-1}  1   +   4 L \cdot \sum_{k = k(s)}^{L-1} \rbr{ 1 - \pi_t(a|s)}     \\
	& \leq   8L^2 \rbr{ 1 - \pi_t(a|s) },
	\end{align*}
   and therefore
   \begin{align*}
& \E_t\sbr{\sum_{s \neq s_L}\sum_{a \in A} \sqrt{ \whatq_t(s,a) } \cdot \rbr{O_t(s,a) + W_t(s,a) }}   \\
&\leq  8L^2 \sum_{s \neq s_L}\sum_{a \in A} \sqrt{ \whatq_t(s,a) }\rbr{ 1 - \pi_t(a|s) } \\
&\leq 8L^2 \sum_{s \neq s_L}\sum_{a \neq \pi(s)} \sqrt{ \whatq_t(s,a) }
+8L^2 \sum_{s \neq s_L} \sqrt{ \whatq_t(s) }\rbr{ 1 - \pi_t(\pi(s)|s) } \\
&\leq 16L^2 \sum_{s \neq s_L}\sum_{a \neq \pi(s)} \sqrt{ \whatq_t(s,a) },
   \end{align*}
which proves  \pref{eq:O+W}.
\end{proof} 


Note that both \pref{eq:bobw_bandit_tsallis_reg_worstcase} and \pref{eq:bobw_bandit_tsallis_reg_adaptive} contain a term related to $\sum_{s \neq s_L} \sum_{a \in A} \whatq_t(s,a) \cdot  B_{i(t)}(s,a)^2$.
Below, we show that when summed over $t$, this is only logarithmic in $T$.
\begin{lemma} \label{lem:bobw_tasallis_bound_extra} \pref{alg:bobw_framework} ensures the following: 
	\begin{equation}
	\E\sbr{ \sum_{t=1}^{T} \sum_{s \neq s_L} \sum_{a \in A} \whatq_t(s,a) \cdot  B_{i(t)}(s,a)^2   } = \order\rbr{ L^2|S|^3|A|^2 \ln^2 \pcons  + |S||A| T \cdot \delta }.
	\end{equation}
\end{lemma}

\begin{proof} By \pref{lem:sa_conf_width_bound}, we know that 
	\begin{align*}
	B_i(s,a)^2 & \leq \rbr{ 2 \sqrt{ \frac{ |S_{k(s)+1}| \ln \pcons }{\max \cbr{m_{i}(s,a) , 1}} } + \frac{ 14|S_{k(s)+1}| \ln \pcons }{3\max \cbr{m_{i}(s,a) , 1}} }^2 \\
	& \leq \order\rbr{  \frac{ |S_{k(s)+1}| \ln \pcons }{\max \cbr{m_{i}(s,a) , 1}} + \frac{ |S_{k(s)+1}|^2 \ln^2\pcons }{\max \cbr{m_{i}(s,a) , 1}^2} }.
	\end{align*}
	Then, we have 
	\begin{align*}
	& \E\sbr{ \sum_{t=1}^{T} \sum_{s \neq s_L} \sum_{a \in A} \whatq_t(s,a) \cdot  B_{i(t)}(s,a)^2   }  \\
	& = \E\sbr{ \sum_{t=1}^{T} \sum_{s \neq s_L} \sum_{a \in A} \rbr{ \whatq_t(s,a) - q_t(s,a) }  \cdot  B_{i(t)}(s,a)^2   } +  \E\sbr{  \sum_{t=1}^{T} \sum_{s \neq s_L} \sum_{a \in A} q_t(s,a) \cdot  B_{i(t)}(s,a)^2   }  \\
	& \leq \E\sbr{   \sum_{t=1}^{T} \sum_{s \neq s_L} \sum_{a\in A} r_t(s,a)   } \\
	& \quad + \E\sbr{ 4 \sum_{t=1}^{T} \sum_{s \neq s_L} \sum_{a\in A} \sum_{k=0}^{k(s)-1} \sum_{(u,v,w)\in T_k} q_t(u,v)  \sqrt{ \frac{ P(w|u,v) \ln \rbr{ \frac{T|S||A|}{\delta}} }{ \max\cbr{m_{i(t)}(u,v)  ,1}} }q_t(s,a|w) \cdot B_{i(t)}(s,a)   } \\
	& \quad + \order\rbr{ \E\sbr{    \sum_{t=1}^{T} \sum_{s \neq s_L} \sum_{a \in A} q_t(s,a) \cdot \rbr{  \frac{ |S_{k(s)+1}| \ln \pcons }{\max \cbr{m_{i}(s,a) , 1}} + \frac{ |S_{k(s)+1}|^2 \ln^2\pcons }{\max \cbr{m_{i}(s,a) , 1}^2} }    }   } \\
	& \leq \order\rbr{  \E\sbr{  \sum_{t=1}^{T} \sum_{s \neq s_L} \sum_{a\in A} r_t(s,a)  } }\\
	& \quad + \order\rbr{ \E\sbr{    \sum_{t=1}^{T} \sum_{s \neq s_L} \sum_{a \in A} q_t(s,a) \cdot \rbr{  \frac{ |S_{k(s)+1}| \ln \pcons }{\max \cbr{m_{i}(s,a) , 1}} + \frac{ |S_{k(s)+1}|^2 \ln^2\pcons }{\max \cbr{m_{i}(s,a) , 1}^2} }    }   } 
	\end{align*}
	where the first inequality uses \pref{lem:residual_term_property} and $B_i(s,a)\in [0,1]$, and the last inequality follows from the observation that, the second term in the previous line is bounded by $\sum_{t=1}^{T} \sum_{s \neq s_L} \sum_{a\in A}  r_t(s,a)$ according to the definition of residual terms in \pref{def:residual_terms}. 
	
	Finally, applying \pref{lem:residual_term_property} and \pref{lem:aux_exp}, we have
	\begin{align*}
	& \E\sbr{ \sum_{t=1}^{T} \sum_{s \neq s_L} \sum_{a \in A} \whatq_t(s,a) \cdot  B_{i(t)}(s,a)^2   } \\
	& = \order\rbr{ L^2|S|^3|A|^2 \ln^2 \pcons  + |S||A| T \cdot \delta }  + \order\rbr{  \sum_{k = 0}^{L-1}\rbr{ \abr{ S_{k+1}}\abr{ S_{k}}|A| \ln T \ln \pcons  +  |S_{k(s)+1}|^2\abr{ S_{k}}|A| \ln^2\pcons }} \\
	& = \order\rbr{ L^2|S|^3|A|^2 \ln^2 \pcons  + |S||A| T \cdot \delta },
	\end{align*}
	which completes the proof.
\end{proof}

Finally, we provide a lemma regarding the learning rates.
\begin{lemma}[Learning Rates] \label{lem:bobw_lr_properties} According to the design of the learning rate $\eta_t = \frac{1}{\sqrt{t-t_{i(t)}+1}}$, the following inequalities hold:
	\begin{equation}
	\sum_{t=1}^{T} \eta_t^2  \leq \order\rbr{ |S||A|\log^2 T}, \label{eq:bobw_bandit_lr_1}
	\end{equation}
	\begin{equation}
	\sum_{t=1}^{T} \eta_t  \leq \order\rbr{ \sqrt{ |S||A|T \log T }}.  \label{eq:bobw_bandit_lr_2}
	\end{equation}
\end{lemma}
\begin{proof} By direct calculation, we have 
	\begin{align*}
	\sum_{t=t_i}^{t_{i+1}-1} \eta_t^2= \sum_{n=1}^{t_{i+1} - t_i} \frac{1}{n} \leq 2\int_{1}^{t_{i+1} - t_i + 1} \frac{1}{x} dx = 2\ln\rbr{t_{i+1} - t_i + 1} \leq \order\rbr{ \log T}. 
	\end{align*}
	Combining the inequality with the fact that the total number of epochs $N$ is at most $4|S||A| \rbr{ \log T + 1}$ (\pref{lem:bobw_N_bound}) finishes the proof of \pref{eq:bobw_bandit_lr_1}.
	Following the similar idea, we have
	\begin{align*}
	\sum_{t=t_i}^{t_{i+1}-1} \eta_t = \sum_{n=1}^{t_{i+1} - t_i} \frac{1}{\sqrt{n}} \leq \int_{0}^{t_{i+1} - t_i} \frac{1}{\sqrt{x}} dx \leq 2\sqrt{t_{i+1} - t_i}. 
	\end{align*}
	Taking the summation over $N$ epochs and applying the Cauchy-Schwarz inequality yields \pref{eq:bobw_bandit_lr_2}.
\end{proof}

\subsection{Proof for the Adversarial World (\pref{prop:bobw_bandit_adv_prop})}
Recall the regret decomposition in \pref{eq:bandit_decomposition}:
\[
\E\Biggsbr{\underbrace{\sum_{t=1}^{T} V^{\pi_t}_t(s_0) - \widetilde{V}^{\pi_t}_t(s_0)}_{\textsc{Err}_1 }}  + \E\Biggsbr{\underbrace{\sum_{t=1}^{T} \widetilde{V}^{\pi_t}_t(s_0) - \widetilde{V}^{\pi}_t(s_0) }_{\textsc{EstReg} }}  + \E\Biggsbr{\underbrace{\sum_{t=1}^{T}\widetilde{V}^{\pi}_t(s_0) - V^{\pi}_t(s_0) }_{\textsc{Err}_2}}. 
\]
We bound each of them separately below. 

\paragraph{$\textsc{Err}_1$}
Similarly to the proof for the full-information feedback setting, we have 
\begin{align*}
\textsc{Err}_1 & = \sum_{t=1}^{T} \inner{q_t, \ell_t } - \inner{ \whatq_t,  \wtill_t} \\
& = \sum_{t=1}^{T}\sum_{s \neq s_L} \sum_{a \in A}  \frac{\ell_t(s,a)\whatq_t(s,a)}{u_t(s,a)} \cdot \rbr{u_t(s,a) - q_t(s,a)} + \sum_{t=1}^{T} \inner{q_t - \whatq_t, \ell_t } +   L \cdot \sum_{t=1}^{T} \inner{\whatq_t, B_{i(t)} }  
\end{align*}
where the last two terms have been shown to be at most $\otil\rbr{ L|S|\sqrt{|A|T } + L^3|S|^3|A|}$
 according to the analysis of $\textsc{Err}_1$ in \pref{app:bobw_fullinfo_adv_proof} (see \pref{eq:bobw_fullinfo_adv_term2},  \pref{eq:bobw_fullinfo_adv_term1}  and \pref{eq:bobw_fullinfo_adv_term3}).

Then, we bound the first term as 
\begin{align*}
& \E\sbr{ \sum_{t=1}^{T}\sum_{s \neq s_L} \sum_{a \in A}  \frac{\ell_t(s,a)\whatq_t(s,a)}{u_t(s,a)} \cdot \rbr{u_t(s,a) - q_t(s,a)} } \\
& \leq \E\sbr{ \sum_{t=1}^{T}\sum_{s \neq s_L} \sum_{a \in A} \abr{u_t(s,a) - q_t(s,a)} } \tag{$\whatq_t(s,a) \leq u_t(s,a)$} \\
& \leq \E\sbr{4 \sum_{t=1}^{T}  \sum_{s \neq s_L} \sum_{a \in A}  r_t(s,a) + 16 \sum_{t=1}^{T}  \sum_{s \neq s_L} \sum_{a \in A}\sum_{k=0}^{k(s)-1} \sum_{(u,v,w)\in T_k} q_t(u,v)  \sqrt{ \frac{ P(w|u,v) \ln \pcons }{ \max\cbr{m_{i(t)}(u,v)  ,1}} }q_t(s,a|w) } \tag{\pref{col:residual_term_property} }\\
& \leq \order\rbr{ L^2|S|^3|A|^2 \ln^2 \pcons  + |S||A| T \cdot \delta } + 4L \cdot \E\sbr{  \sum_{t=1}^{T} \sum_{u\neq s_L} \sum_{ v \in A} q_t(u,v) \sqrt{  \frac{ \abr{S_{k(u)+1} } \ln \pcons }{ \max\cbr{m_{i(t)}(u,v)  ,1}}  } } \tag{\pref{lem:residual_term_property} and Cauchy-Schwarz} \\
& \leq \order\rbr{ L^2|S|^3|A|^2 \ln^2 \pcons   + |S||A| T \cdot \delta + L \cdot \sum_{k=0}^{L-1} \sqrt{ \abr{S_{k}} \cdot \abr{S_{k+1} } |A| T\ln \pcons  } } \tag{\pref{lem:aux_exp}}\\
& =  \order\rbr{ L|S|\sqrt{|A|T\ln \pcons} +  L^2|S|^3|A|^2 \ln^2 \pcons   + |S||A| T \cdot \delta }.
\end{align*}

Combining the bounds together,  we have $\E\sbr{ \textsc{Err}_1}$ bounded by:
\[\E\sbr{ \textsc{Err}_1} = 
\otil\rbr{ L|S|\sqrt{|A|T } +  L^3|S|^3|A|^2 }.
\]

\paragraph{$\textsc{Err}_2$} Following the same idea of bounding $\textsc{Err}_2$, by \pref{lem:bandit_optimism} and \pref{lem:exp_high_prob_bound}, we have the expectation of $\textsc{Err}_2$ bounded as 
\[
\E\sbr{\textsc{Err}_2} \leq \delta \cdot 3L|S|T^2 + 0 = \order\rbr{ L|S|T^2 \cdot \delta } = \order(1). 
\]

\paragraph{$\textsc{EstReg}$} According to \pref{eq:bobw_bandit_tsallis_reg_worstcase} of \pref{lem:bobw_bandit_tsallis_reg}, we have  
\begin{align*}
& \EReg(\optpi) =  \E\sbr{ \sum_{t=1}^{T} \inner{ \whatq_t - q^{\bar{P}_{i(t)}, \optpi }, \hatl_t  } }  = \E\sbr{ \sum_{i=1}^{N} \EReg_i(\optpi) } \\
& \leq \order\rbr{  \E\sbr{ \sum_{i=1}^{N}\sum_{t=t_i}^{t_{i+1}-1} \eta_t  \sqrt{L|S||A|}} + \E\sbr{L^2\cdot \sum_{t=1}^{T} \sum_{s \neq s_L} \sum_{a\in A} \whatq_t(s,a) \cdot B_{i(t)}(s,a)^2}} \\
&\qquad + \order\rbr{L^4|S|^2|A|^2\ln^2 T + \delta L|S|T^2} \\
& \leq \otil\rbr{  \E\sbr{ \sum_{t=1}^{T}  \eta_t \cdot  \sqrt{L|S||A|}  }  + L^4|S|^3|A|^2 \ln^2 \pcons  } \tag{\pref{lem:bobw_tasallis_bound_extra}} \\
& \leq \otil\rbr{ |S||A| \sqrt{LT} + L^4|S|^3|A|^2 }. \tag{\pref{eq:bobw_bandit_lr_2}}
\end{align*}

Finally, we combine the bounds of $\textsc{Err}_1$, $\textsc{Err}_2$ and $\textsc{EstReg}$ as:
\begin{align*}
\Reg_T(\optpi) & = \otil\rbr{  L|S|\sqrt{|A|T} + |S||A| \sqrt{LT} + L^4|S|^3|A|^2 },
\end{align*}
finishing the proof.

\subsection{Proof for the Stochastic World (\pref{prop:bobw_bandit_stoc_prop})}
Similarly to the proof of \pref{prop:bobw_fullinfo_stoc_lem}, we decompose $\textsc{Err}_1$ and $\textsc{Err}_2$ jointly into four terms $\textsc{ErrSub}$, $\textsc{ErrOpt}$, $\textsc{OccDiff}$ and $\textsc{Bias}$: 
\begin{align*}
\textsc{Err}_1 + \textsc{Err}_2 &  =  \sum_{t=1}^{T}\sum_{s \neq s_L} \sum_{a \neq \pi^{\star}(s)} q_t(s,a)\widehat{E}^{\pi^{\star}}_t(s,a)  &&(\textsc{ErrSub})\\  
& \quad + \sum_{t=1}^{T}\sum_{s \neq s_L} \sum_{a = \pi^{\star}(s) } \rbr{ q_t(s,a) - q^{\star}_{t}(s,a) } \widehat{E}^{\pi^{\star}}_t(s,a)  &&(\textsc{ErrOpt}) \\
& \quad + \sum_{t=1}^{T}\sum_{s \neq s_L} \sum_{a \in A} \rbr{ q_t(s,a) - \widehat{q}_t(s,a)} \rbr{\widetilde{Q}^{\pi^{\star}}_t(s,a) -\widetilde{V}^{\pi^{\star}}_t(s)} &&(\textsc{OccDiff})  \\
& \quad + \sum_{t=1}^{T}\sum_{s \neq s_L} \sum_{a\neq \pi^{\star}(s)}  q^{\star}_{t}(s,a) \rbr{\widetilde{V}^{\pi^{\star}}_t(s) - V^{\pi^\star}_t(s) } && (\textsc{Bias}) 
\end{align*}
where $\widehat{E}^{\pi}_t$ is defined as 
\begin{align*}
\widehat{E}^{\pi}_t(s,a) =  \ell_t(s,a) + \sum_{s' \in S_{k(s)+1}} P(s'|s,a)\widetilde{V}^{\pi}_t(s') -  \widetilde{Q}^{\pi}_t(s,a).
\end{align*}


By the exact same reasoning as in the full-information setting (\pref{app:bobw_fullinfo_stoc_proof}), we have
$\E\sbr{\textsc{OccDiff}} = \order\rbr{ L^4|S|^3|A|^2 \ln^2 \pcons +  \E\sbr{ \selfterm_3(L^4\ln \pcons) }}$
and $\E\sbr{\textsc{Bias}} = \order(1)$,
but the first two terms $\textsc{ErrSub}$ and $\textsc{ErrOpt}$ are slightly different.
To see this, note that
under event $\calA$, we have
\begin{align*}
\widehat{E}^{\pi^{\star}}_t(s,a) & = \ell_t(s,a) - \wtill_t(s,a) + \sum_{s' \in S_{k(s)+1}} \rbr{ P(s'|s,a) - \bar{P}_{i(t)}(s'|s,a) } \widetilde{V}^{\pi^{\star}}_t(s') \\
& = \ell_t(s,a) \rbr{  1 - \frac{q_t(s,a) }{u_t(s,a) }} + L \cdot B_{i(t)}(s,a) + \sum_{s' \in S_{k(s)+1}} \rbr{ P(s'|s,a) - \bar{P}_{i(t)}(s'|s,a) } \widetilde{V}^{\pi^{\star}}_t(s')  \\
& \leq \frac{ u_t(s,a) - q_t(s,a) }{ q_t(s,a)  } + 2 L^2 \cdot B_{i(t)}(s,a)
\end{align*}
where the last line applies the definition of event $\calA$ and the fact $q_t(s,a) \leq u_t(s,a)$ given this event. 
Importantly, the second term has been studied and bounded in the proof of \pref{prop:bobw_fullinfo_stoc_lem} already,
so we only need to focus on the first term.
Before doing so, note that the range of  $\widehat{E}^{\pi}_t$ is $\order\rbr{L|S|t}$ based on \pref{col:bandit_est_Q_bound},
and thus the range of $\textsc{ErrSub}$ and $\textsc{ErrOpt}$ is $\order\rbr{L^2|S|T^2}$.
Therefore, we only need to add a term $\order\rbr{\delta \cdot L^2|S|T^2}$ to address the event $\calA^c$.

\paragraph{Extra term in $\textsc{ErrSub}$} 
According to previous analysis,
the extra term in $\textsc{ErrSub}$ is
\begin{align*}
&\sum_{t=1}^{T}\sum_{s \neq s_L} \sum_{a \neq \pi^{\star}(s)} q_t(s,a) \cdot \frac{ u_t(s,a) - q_t(s,a) }{ q_t(s,a)  } \leq \sum_{t=1}^{T}\sum_{s \neq s_L} \sum_{a \neq \pi^{\star}(s)} \abr{ u_t(s,a) - q_t(s,a) } \\
& \leq 4 \sum_{t=1}^{T}\sum_{s \neq s_L} \sum_{a \neq \pi^{\star}(s)} r_t(s,a) + 16 \sum_{t=1}^{T}\sum_{s \neq s_L} \sum_{a \neq \pi^{\star}(s)} \sum_{k=0}^{k(s)-1} \sum_{(u,v,w)\in T_k} q_t(u,v)  \sqrt{ \frac{ P(w|u,v) \ln \rbr{ \frac{T|S||A|}{\delta}} }{ \max\cbr{m_{i(t)}(u,v)  ,1}} }q_t(s,a|w) \tag{\pref{col:residual_term_property}} \\
&= 4 \sum_{t=1}^{T}\sum_{s \neq s_L} \sum_{a \neq \pi^{\star}(s)} r_t(s,a)+ 16 \selfterm_3(\ln \pcons) \tag{\pref{def:self_bounding_terms}} \\
&= 16 \selfterm_3(\ln \pcons) + \order\rbr{L^2S^3A^2\ln^2\pcons} \tag{\pref{lem:residual_term_property}}.
\end{align*}


Finally, using \pref{lem:exp_high_prob_bound_cond} and the bound on $\textsc{ErrSub}$ for the full-information setting, we have
\begin{equation*}
\E\sbr{\textsc{ErrSub}} = \order\rbr{ \selfterm_3(\ln \pcons) + \selfterm_1(L^4|S|\ln \pcons) +  L^2|S|^3|A|^2 \ln^2 \pcons }.
\end{equation*}

\paragraph{Extra term in $\textsc{ErrOpt}$} 
Similarly, we consider the extra term in $\textsc{ErrOpt}$: 
\begin{align*}
&\sum_{t=1}^{T}\sum_{s \neq s_L} \sum_{a = \pi^{\star}(s) }\rbr{ q_t(s,a) - q^{\star}_{t}(s,a) } \cdot \frac{ u_t(s,a) - q_t(s,a) }{ q_t(s,a)  } \\
&\leq 4 \sum_{t=1}^{T}\sum_{s \neq s_L} \sum_{a = \pi^{\star}(s) }\frac{ q_t(s,a) - q^{\star}_{t}(s,a) }{q_t(s,a)} r_t(s,a) \\
& \quad +  \sum_{t=1}^{T}\sum_{s \neq s_L} \sum_{a = \pi^{\star}(s) }\frac{ q_t(s,a) - q^{\star}_{t}(s,a) }{q_t(s,a)}\cdot \rbr{ 16 \sum_{u,v,w} q_t(u,v) \sqrt{ \frac{ P(w|u,v) \ln \rbr{\frac{T|S||A|}{\delta}} }{ \max\cbr{ m_{i(t)}(u,v),1} }   } q_t(s,a|w) } \tag{\pref{col:residual_term_property}} \\
& \leq 4 \sum_{t=1}^{T}\sum_{s \neq s_L} \sum_{a = \pi^{\star}(s) }  r_t(s,a) + 16 \selfterm_6(\ln \pcons) \tag{\pref{def:self_bounding_terms}} \\
&= 16 \selfterm_6(\ln \pcons) + \order\rbr{L^2S^3A^2\ln^2\pcons} \tag{\pref{lem:residual_term_property}}.
\end{align*}

Again, considering the term that appears in the full-information setting already, we have
\begin{equation*}
\E\sbr{\textsc{ErrOpt}} = \order\rbr{ \selfterm_6(\ln \pcons) + \selfterm_2(L^4|S|\ln \pcons) +  L^2|S|^3|A|^2 \ln^2 \pcons}.
\end{equation*}


It remains to bound $\textsc{EstReg}$ with terms that enjoy self-bounding properties.
\paragraph{Term $\textsc{EstReg}$} According to \pref{eq:bobw_bandit_tsallis_reg_adaptive} in \pref{lem:bobw_bandit_tsallis_reg},  taking the summation of all the epochs, we have the following bound for $\E\sbr{\textsc{EstReg}}$: 
\begin{align*}
& \order\rbr{ \E\sbr{ \sqrt{|S|L}  \sum_{i=1}^{N} \sum_{t=t_i}^{t_{i+1}-1} \eta_t \cdot \sqrt{\sum_{s  \neq s_L}\sum_{a\neq \pi^\star(s)} \whatq_t(s,a)  } }  +  L^2 \cdot \E\sbr{  \sum_{i=1}^{N} \sum_{t=t_i}^{t_{i+1}-1} \eta_t \cdot \sum_{s \neq s_L} \sum_{a\neq \pi^\star(s)} \sqrt{ \whatq_t(s,a) } } } \\
& \quad + \order\rbr{ \E\sbr{ L^4|A|  \sum_{i=1}^{N} \sum_{t=t_i}^{t_{i+1}-1} \sum_{s \neq s_L } \sum_{a \in A}  \whatq_t(s,a) \cdot B_{i(t)}(s,a)^2    } } \\
& \quad  + \order\rbr{ \delta \cdot \E\sbr{ L|S|T\sum_{i=1}^{N}\rbr{ t_{i+1} - t_i   } }  +   L^4|S|^2|A|^2\ln^2 \pcons} \\ 
& = \order\rbr{ \E\sbr{ \sqrt{|S|L}  \sum_{t=1}^{T} \eta_t \cdot \sqrt{\sum_{s  \neq s_L}\sum_{a\neq \pi^\star (s)} \whatq_t(s,a)  } } } + \order\rbr{ L^2 \cdot \E\sbr{  \sum_{t=1}^{T} \eta_t \cdot \sum_{s \neq s_L} \sum_{a\neq \pi^\star(s)} \sqrt{ \whatq_t(s,a) } } } \\
& \quad  + \order\rbr{ L^6|S|^3|A|^3 \ln^2 \pcons } 
\end{align*}
where the lase line applies \pref{lem:bobw_tasallis_bound_extra}. 

Then, for the first term, we have 
\begin{align*}
& \E\sbr{ \sqrt{|S|L}  \sum_{t=1}^{T} \eta_t \cdot \sqrt{\sum_{s  \neq s_L}\sum_{a\neq \pi^\star(s)} \whatq_t(s,a)  } } \\
& \leq \E\sbr{ \sqrt{|S|L}  \cdot \sqrt{ \sum_{t=1}^{T} \eta_t^2 } \cdot \sqrt{\sum_{t=1}^{T} \sum_{s  \neq s_L}\sum_{a\neq \pi^\star(s)} \whatq_t(s,a)  }   }  \\
& \leq \E\sbr{ \sqrt{4L|S|^2|A|\log^2 T}  \cdot \sqrt{\sum_{t=1}^{T} \sum_{s  \neq s_L}\sum_{a\neq \pi^\star(s)} \whatq_t(s,a)  } }    
\end{align*}
where the second line follows from the Cauchy-Schwarz inequality, and the third line applies \pref{eq:bobw_bandit_lr_1}. 

Then, we separate the term into two parts:
\begin{align*}
& \E\sbr{ \sqrt{4L|S|^2|A|\log^2 T}  \cdot \sqrt{\sum_{t=1}^{T} \sum_{s  \neq s_L}\sum_{a\neq \pi^\star(s)} q_t(s,a)  } } \\
& \quad + \E\sbr{ \sqrt{4L|S|^2|A|\log^2 T}  \cdot \sqrt{\sum_{t=1}^{T} \sum_{s  \neq s_L}\sum_{a\neq \pi^\star(s)} \abr{\whatq_t(s,a) - q_t(s,a)} } } \\
& \leq \E\sbr{2 \cdot \selfterm_4(L|S|^2|A|\log^2 T)} + \E\sbr{  \sum_{t=1}^{T} \sum_{s  \neq s_L}\sum_{a\neq \pi^\star(s)} \abr{\whatq_t(s,a) - q_t(s,a)}  } + 4|S|^2|A|L\log^2 T
\end{align*}
where second line follows from the fact $\sqrt{xy} \leq x + y$ for $x,y\geq0$. Note that, the second term above can be bounded by $\order\rbr{\selfterm_3\rbr{ \ln \pcons  }+ L^2|S|^3|A|^2 \ln^2 \pcons}$ just as in the full-information setting (see \pref{eq:q_difference}). Therefore, we have finished bounding the first term:
\begin{align*}
& \E\sbr{ \sqrt{|S|L}  \sum_{t=1}^{T} \eta_t \cdot \sqrt{\sum_{s  \neq s_L}\sum_{a\neq \pi^\star(s)} \whatq_t(s,a)  } }   \\
& = \order\Bigrbr{ \E\sbr{ \selfterm_4(L|S|^2|A|\log^2 T) + \selfterm_3\rbr{ \ln \pcons  } } +   L^2|S|^3|A|^2 \ln^2 \pcons  }.
\end{align*} 

On the other hand, the second term can be bounded similarly:
\begin{align*}
& L^2 \cdot \E\sbr{  \sum_{t=1}^{T} \eta_t \cdot \sum_{s \neq s_L} \sum_{a\neq \pi^\star (s)} \sqrt{ \whatq_t(s,a) } } \\
& \leq L^2 \cdot \E\sbr{  \sum_{s \neq s_L} \sum_{a\neq \pi^\star(s)}  \cdot \sqrt{ \sum_{t=1}^{T} \eta_t^2 } \cdot \sqrt{ \sum_{t=1}^{T} \whatq_t(s,a) } } \\
& \leq L^2 \sqrt{4|S||A|\log^2 T } \cdot \E\sbr{  \sum_{s \neq s_L} \sum_{a\neq \pi^\star(s)} \cdot \sqrt{ \sum_{t=1}^{T} \whatq_t(s,a) } } \\
& \leq \E\sbr{ 2 \cdot \selfterm_5(L^4|S||A|\log^2 T) } +  \E\sbr{  \sum_{t=1}^{T} \sum_{s  \neq s_L}\sum_{a\neq \pi^\star(s)} \abr{\whatq_t(s,a) - q_t(s,a)}  } +  L^4|S||A|\log^2 T \\
& = \order\Bigrbr{ \E\sbr{ \selfterm_5(L^4|S||A|\log^2 T) + \selfterm_3\rbr{ \ln \pcons  } } +   L^2|S|^3|A|^2 \ln^2 \pcons}.
\end{align*}

So we have the final bound on $\E\sbr{ \textsc{EstReg}}$:
\begin{align*}
\E\sbr{ \textsc{EstReg}} & = \order\Bigrbr{  \E\sbr{ \selfterm_4\rbr{ L|S|^2|A|\log^2 T} + \selfterm_5\rbr{L^4|S||A|\log^2 T} + \selfterm_3\rbr{ \ln \pcons  }  }  + L^6|S|^3|A|^3 \ln^2 \pcons } 
\end{align*} 

Finally, by combining the bounds of each term, we finally have 
\begin{align*}
\Reg_T(\pi^\star) & \leq \order \Big( \E\sbr{ \selfterm_1\rbr{ L^4|S| \ln T  } + \selfterm_3 \rbr{ \ln T } } && \rbr{ \text{from } \textsc{ErrSub} } \\
& \quad +\E\sbr{  \selfterm_2\rbr{ L^4|S| \ln T  } +  \selfterm_6 \rbr{ \ln T  } } && \rbr{ \text{from } \textsc{ErrOpt} }  \\
& \quad + \E\sbr{  \selfterm_3\rbr{ L^4 \ln T } }  && \rbr{ \text{from } \textsc{OccDiff} }  \\
& \quad + \E\sbr{ \selfterm_4\rbr{  L|S|^2|A| \ln^2 T  } + \selfterm_5\rbr{L^4|S||A|\ln^2 T} +  \selfterm_3\rbr{ \ln T } } && \rbr{ \text{from } \textsc{EstReg} } \\
& \quad +  L^6|S|^3|A|^3  \ln^2 T \Big). 
\end{align*}

When Condition~\eqref{eq:loss_condition} holds,
we apply similar self-bounding arguments to obtain a logarithmic regret bound.
Specifically,
for some universal constant $\kappa > 0$, we have 
\begin{align*}
\Reg_T(\pi^\star) & \leq \kappa \Biggrbr{\E\sbr{ \selfterm_1\rbr{ L^4|S| \ln T  } +  \selfterm_2\rbr{ L^4|S| \ln T  } + \selfterm_3\rbr{ L^4 \ln T } } } \\
& \quad +\kappa \Biggrbr{ \E\sbr{   \selfterm_4\rbr{  L|S|^2|A| \log^2 T  } + \selfterm_5\rbr{L^4|S||A|\log^2 T} + \selfterm_6 \rbr{ \ln T}  } }   \\
& \quad +  \kappa \Biggrbr{  L^6|S|^3|A|^3  \ln^2 \pcons }.  
\end{align*}
Then, for any $z>1$, by applying all the self-bounding lemmas (\pref{lem:self_bounding_term_1}-\pref{lem:self_bounding_term_6}) with $\alpha = \beta = \frac{1}{32z\kappa}$, we arrive at
\begin{align*}
\Reg_T(\pi^\star) & \leq \frac{1}{z} \cdot \rbr{ \Reg_T(\pi^\star)  + C } \\
& \quad + z\cdot \order\rbr{ \rbr{  \sum_{s \neq s_L} \sum_{ a\neq \pi^\star(s)} \frac{\kappa^2}{\gap(s,a)} } \cdot \Bigrbr{  L^4|S| \ln T  +  L^6|S| \ln T +  L^4|S||A|\log^2 T }} \\
& \quad + z\cdot \order\rbr{ \frac{\kappa^2}{\gapmin}  \cdot \Bigrbr{ L^5 |S|^2 \ln T  +  L^6|S|^2 \ln T +  L^3|S|^2|A|\ln T +  L|S|^2|A| \log^2 T     } }\\
& \quad + \kappa \cdot \rbr{L^6|S|^3|A|^3 \ln^2 T } \\
& \leq \frac{1}{z} \cdot \rbr{ \Reg_T(\pi^\star)  + C } + \kappa \cdot \rbr{L^6|S|^3|A|^3 \ln^2 T }  \\
& \quad + z\cdot \order\rbr{  \sum_{s \neq s_L} \sum_{ a\neq \pi^\star(s)} \frac{L^6|S|\ln T +  L^4|S||A|\log^2 T}{\gap(s,a)} + \frac{L^6|S|^2 \ln T +  L^3|S|^2|A| \log^2 T}{\gapmin}} \\
& \leq \frac{1}{z} \cdot \rbr{ \Reg_T(\pi^\star)  + C  } + z \cdot \kappa' U + \kappa \cdot V,
\end{align*}
where $\kappa'$ is a universal constant hidden in the $\order(\cdot)$ notation,
and $U$ and $V$ are defined in \pref{prop:bobw_bandit_stoc_prop}).
The last step is to rearrange and pick the optimal $z$, which is almost identical to that in the proof of \pref{prop:bobw_fullinfo_stoc_lem} and finally shows
$\Reg_T(\pi^\star) = \order\rbr{ U + \sqrt{ UC } + V    }$.
This completes the entire proof.



\newpage
\section{General Decomposition, Self-bounding Terms, and Supplementary Lemmas}
\label{app:all_supp_lemmas}
In this section, we provide details of our two key techniques: a general decomposition and self-bounding terms, as well as a set of supplementary Lemmas used throughout the analysis.

\subsection{General Decomposition Lemma}
\label{app:general_decomp_lem}

In this section, we consider measuring the performance difference between a policy $\pi$ and a mapping (deterministic policy) $\pi^\star$, that is, $V^\pi(s_0) - V^{\pi^\star}(s_0)$ where $Q$ and $V$ are the state-action and state value functions associated with some transition $P$ and some loss function $\ell$, that is, 
\[
Q^\pi(s,a) = \ell(s,a) + \sum_{s'\in S_{k(s)+1}} P(s'|s,a) V^\pi(s'), \quad V^\pi(s) = \sum_{a \in A} \pi(a|s) Q^\pi(s,a), 
\]
for all state-action pairs (with $V^\pi(s_L) = 0$). 
Moreover, for some estimated transition $\widehat{P}$ and estimated loss function $\hatl$, define similarly $\whatQ$ and $\whatV$ as the corresponding state-action and state value functions: 
\[
\whatQ^\pi(s,a) = \hatl(s,a) + \sum_{s'\in S_{k(s)+1}} \widehat{P}(s'|s,a) \whatV^\pi(s'), \quad \whatV^\pi(s) = \sum_{a \in A} \pi(a|s) \whatQ^\pi(s,a),
\]
for all state-action pairs (with $\whatV^\pi(s_L) = 0$). 

Again, we denote by $q^{\star}_{\pi}(s,a)$ the probability of visiting a trajectory of the form $(s_0, \pi^\star(s_0)), (s_1, \pi^\star(s_1)), \ldots, (s_{k(s)-1}, \pi^\star(s_{k(s)-1})), (s,a)$ when
executing policy $\pi$. In other words, $q^{\star}_{\pi}$ can be formally defined as 
\begin{align*}
q^{\star}_{\pi}(s,a) = \begin{cases}
\pi(a|s), & s = s_0,  \\
\pi(a|s) \cdot \rbr{ \sum_{s'\in S_{k(s)-1}} q^{\star}_{\pi}(s',  \pi^\star(s)) P(s|s',\pi^\star(s)) }, & \text{otherwise}. 
\end{cases}
\end{align*}
Note that our earlier notation $q_t^\star$ is thus a shorthand for $q_{\pi_t}^\star$. With slight abuse of notations, we define $q^{\star}_{\pi}(s) =\sum_{a\in A}q^{\star}_{\pi}(s,a)$.

Now, we present a general decomposition for $V^\pi(s_0) - V^{\pi^\star}(s_0)$.
\begin{lemma} \label{lem:general_perf_decomp} (General Performance Decomposition) For any policies $\pi$ and $u$, and a mapping (deterministic policy) $\pi^\star: S\rightarrow A$, we have
	\begin{align*}
	 V^\pi(s_0) - V^{\pi^\star}(s_0)  
	& =  \sum_{s\neq s_L} \sum_{a \neq \pi^{\star}(s)} q(s,a) {\widehat{E}^u }(s,a)  &&(\text{Error of Sub-opt actions})\\  
	& \quad + \sum_{s\neq s_L} \sum_{a = \pi^{\star}(s) } \rbr{ q(s,a) - q^{\star}_{\pi}(s,a) }  {\widehat{E}^u }(s,a)  
	&&(\text{Error of Opt actions}) \\
	& \quad +  \sum_{s\neq s_L} \sum_{a \in A} q(s,a) \rbr{\whatQ^{u}(s,a) -\whatV^{u}(s)} &&(\text{Policy Difference})  \\
	& \quad -  \sum_{s\neq s_L} \sum_{a = \pi^{\star}(s)} q^{\star}_{\pi}(s,a) \rbr{ \whatQ^{u}(s,a) - \whatV^{u}(s) } && (\text{Estimation Bias 1})  \\
	& \quad + \sum_{s\neq s_L} \sum_{a\neq \pi^{\star}(s)} q^{\star}_{\pi}(s,a) \rbr{\whatV^{u}(s) - V^{\pi^\star}(s) }, && (\text{Estimation Bias 2})
	\end{align*}
	where $q = q^{P,\pi}$ is the occupancy measure associated with transition $P$ and policy $\pi$, and ${\widehat{E}^\pi }$ is a surplus function with: 
\[
\widehat{E}^\pi (s,a) = \ell(s,a) + \sum_{s' \in S_{k(s)+1}} P(s'|s,a) \whatV^{\pi}(s') - \whatQ^{\pi}(s,a). 
\]
\end{lemma}

Moreover, selecting the surrogate policy $u$ as the mapping $\pi^{\star}$ yields \pref{col:general_perf_decomp}, which is the key decomposition lemma used in our analysis.
\begin{corollary} \label{col:general_perf_decomp} Consider an arbitrary policy sequence $\{ \pi_t \}_{t=1}^{T}$, an arbitrary estimated transition sequence $\{ \widehat{P}_t \}_{t=1}^{T}$, and an arbitrary estimated loss sequence $\{ \hatl_t \}_{t=1}^{T}$. Then, we have
\begin{align*}
& \underbrace{\sum_{t=1}^{T} \rbr{ V^{\pi_t}(s_0) - \whatV_t^{\pi_t}(s_0) }}_{\textsc{Err}_1} + \underbrace{\rbr{ \sum_{t=1}^{T} \whatV_t^{\pi^\star}(s_0) - V^{\pi^\star}(s_0)} }_{\textsc{Err}_2} \\
& = \sum_{t=1}^{T}\sum_{s\neq s_L} \sum_{a \neq \pi^{\star}(s)} q_t(s,a) { \widehat{E}_t}^{\pi^{\star}}(s,a)  &&(\text{Error of Sub-opt actions})\\  
& \quad + \sum_{t=1}^{T}\sum_{s\neq s_L} \sum_{a = \pi^{\star}(s) } \rbr{ q_t(s,a) - q^{\star}_{t}(s,a) }  { \widehat{E}_t }^{\pi^{\star}}(s,a)  
&&(\text{Error of Opt actions}) \\
& \quad + \sum_{t=1}^{T} \sum_{s\neq s_L} \sum_{a \in A} \rbr{ q_t(s,a) - \whatq_t(s,a)} \rbr{\widehat{Q}_t^{\pi^{\star}}(s,a) -\widehat{V}_t^{\pi^{\star}}(s)} &&(\text{Occupancy Difference})  \\
& \quad + \sum_{t=1}^{T}\sum_{s\neq s_L} \sum_{a\neq \pi^{\star}(s)} q^{\star}_{t}(s,a)  \rbr{\widehat{V}_t^{\pi^{\star}}(s) - V_t^{\pi^\star}(s) }, && (\text{Estimation Bias}) 
\end{align*}
where $\widehat{q}_t = q^{\widehat{P}_t,\pi_t}$, $q_t = q^{P, \pi_t}$, $q_t^\star = q_{\pi_t}^\star$,  
$\whatQ_t^{\pi_t}$ and $\whatV_t^{\pi_t}$ are the state-action and state value functions associated with $\pi_t$, $\hatl_t$, and $\widehat{P}_t$,
and ${ \widehat{E} }^{\pi}_t$ is the surplus function defined as:
\[
\widehat{E}^{\pi}_t(s,a) =  \ell(s,a) + \sum_{s' \in S_{k(s)+1}} P(s'|s,a) \widehat{V}_t^{\pi}(s') - \widehat{Q}_t^{\pi}(s,a).
\]
\end{corollary}

\begin{proof} (Proof of \pref{lem:general_perf_decomp}) By direct calculation, for all states $s$, we have 
	\begin{align*} 
	 V^\pi(s) - \whatV^{u}(s) 
	& = \sum_{a\in A} \pi(a|s) \rbr{ Q^\pi(s,a)  - \whatQ^{u}(s,a) } + \sum_{a\in A} \pi(a|s) \rbr{\whatQ^{u}(s,a) - \whatV^{u}(s)}  \\
	& = \sum_{a\in A} \pi(a|s) \sum_{s' \in S_{k(s)+1}} P(s'|s,a) \rbr{V^\pi(s') - \whatV^{u}(s')  }   \\
	& \quad + \sum_{a\in A } \pi(a|s) \underbrace{\rbr{ \ell(s,a) + \sum_{s' \in S_{k(s)+1}} P(s'|s,a) \whatV^{u}(s') - \whatQ^{u}(s,a)  }}_{{\widehat{E}^u }(s,a)} \\
	& \quad + \sum_{a\in A} \pi(a|s) \rbr{\whatQ^{u}(s,a) - \whatV^{u}(s)}.
	\end{align*}
	
	By repeatedly expanding $V^\pi(s') - \whatV^{u}(s')$ in the same way, we conclude 
	\begin{equation}
	\begin{split}  
	V^\pi(s_0) - \whatV^{u}(s_0)  & =  \sum_{s\neq s_L}\sum_{a \in A } q(s,a) {\widehat{E}^u }(s,a)  +  \sum_{s\neq s_L}\sum_{a \in A} q(s,a) \rbr{\whatQ^{u}(s,a) -\whatV^{u}(s)}.
	\end{split} 
	\label{eq:prf_general_decomp_bound_1}
	\end{equation}
	
	On the other hand, we also have for all states $s$:
	\begin{align*}
	& V^\pi(s) - \whatV^{u}(s) \\ 
	& = \sum_{a  = \pi^{\star}(s) } \pi(a|s) \rbr{ Q^\pi(s,a)  - \whatV^{u}(s) } + \sum_{a \neq \pi^{\star}(s) } \pi(a|s) \rbr{ Q^\pi(s,a)  - \whatV^{u}(s) } \\ 
	& = \sum_{a  = \pi^{\star}(s) } \pi(a|s) \sum_{s' \in S_{k(s)+1}} P(s'|s,a) \rbr{ V^\pi(s') - \whatV^{u}(s') } \\
	& \quad + \sum_{a = \pi^{\star}(s) } \pi(a|s) \underbrace{\rbr{ \ell(s,a) + \sum_{s' \in S_{k(s)+1}} P(s'|s,a) \whatV^{u}(s') - \whatQ^{u}(s,a)  }}_{{\widehat{E}^u }(s,a)} \\
	& \quad + \sum_{a  = \pi^{\star}(s) } \pi(a|s) \rbr{ \whatQ^{u}(s,a) - \whatV^{u}(s) } \\
	& \quad + \sum_{a \neq \pi^{\star}(s) } \pi(a|s) \rbr{ Q^\pi(s,a)  - \whatV^{u}(s) }. 
	\end{align*}
	Using \pref{lem:cond_occup_expand} (which repeatedly expands $V^\pi(s') - \whatV^{u}(s')$ in the same way) with 
	\begin{align*}
	C(s) &= \sum_{a = \pi^{\star}(s) } \pi(a|s) \widehat{E}^u (s,a) 
	 + \sum_{a  = \pi^{\star}(s) } \pi(a|s) \rbr{ \whatQ^{u}(s,a) - \whatV^{u}(s) }   \\
	 &\quad + \sum_{a \neq \pi^{\star}(s) } \pi(a|s) \rbr{ Q^\pi(s,a)  - \whatV^{u}(s)}
	\end{align*}
	
	 we obtain
	\begin{equation}
		\begin{split}  
		V^\pi(s_0) - \whatV^{u}(s_0)  
		&= \sum_{s\neq s_L} q^{\star}_{\pi}(s)C(s) \\
		& =  \sum_{s\neq s_L} \sum_{a  = \pi^{\star}(s) }q^{\star}_{\pi}(s,a)   {\widehat{E}^u }(s,a)  \\  
		& \quad +  \sum_{s\neq s_L}  \sum_{a \neq \pi^{\star}(s)} q^{\star}_{\pi}(s,a) \rbr{ Q^\pi(s,a) -\whatV^{u}(s)} \\
		& \quad +  \sum_{s\neq s_L}  \sum_{a  = \pi^{\star}(s)} q^{\star}_{\pi}(s,a) \rbr{ \whatQ^{u}(s,a) - \whatV^{u}(s) }.
		\end{split} 
	\label{eq:prf_general_decomp_bound_2}
	\end{equation}
	
	Combining \pref{eq:prf_general_decomp_bound_1} and \pref{eq:prf_general_decomp_bound_2}, we have the following equality:
	\begin{align}
	& \sum_{s\neq s_L} \sum_{a \neq \pi^{\star}(s)} q^{\star}_{\pi}(s,a)  \rbr{ Q^\pi(s,a) -\whatV^{u}(s)} \notag \\
	& = \sum_{s\neq s_L} \sum_{a \in A } q(a,s) {\widehat{E}^u }(s,a)  \notag \\  
	& \quad +  \sum_{s\neq s_L} \sum_{a \in A} q(s,a) \rbr{\whatQ^{u}(s,a) -\whatV^{u}(s)} \notag  \\
	& \quad - \sum_{s\neq s_L}  \sum_{a  = \pi^{\star}(s) } q^{\star}_{\pi}(s,a)  {\widehat{E}^u }(s,a)  \notag \\  
	& \quad -  \sum_{s\neq s_L}  \sum_{a  = \pi^{\star}(s)} q^{\star}_{\pi}(s,a) \rbr{ \whatQ^{u}(s,a) - \whatV^{u}(s) }\notag \\
	& = \sum_{s\neq s_L}  \sum_{a \neq \pi^{\star}(s)} q(s,a) {\widehat{E}^u }(s,a)  &&(\text{Error of Sub-opt actions}) \label{eq:red_box_eq_2}  \\  
	& \quad + \sum_{s\neq s_L} \sum_{a  = \pi^{\star}(s) } \rbr{ q(s,a) - q^{\star}_{\pi}(s,a) }  {\widehat{E}^u }(s,a)  
	&&(\text{Error of Opt actions}) \notag  \\
	& \quad +  \sum_{s\neq s_L}  \sum_{a \in A} q(s,a) \rbr{\whatQ^{u}(s,a) -\whatV^{u}(s)} &&(\text{Policy Difference})  \notag  \\
	& \quad -  \sum_{s\neq s_L}  \sum_{a  = \pi^{\star}(s)} q^{\star}_{\pi}(s,a)\rbr{ \whatQ^{u}(s,a) - \whatV^{u}(s) } && (\text{Estimation Bias 1}) \label{eq:prf_general_decomp_bound_4},
	\end{align}
	
	Next, we consider the following: 
	\begin{align*}
	& V^\pi(s) - V^{\pi^\star}(s) \\ 
	& = \sum_{a  = \pi^{\star}(s) } \pi(a|s) \rbr{  Q^\pi(s,a) - Q^{\star}(s,a) } + \sum_{a \neq \pi^{\star}(s) }\pi(a|s) \rbr{ Q^\pi(s,a) - V^{\pi^\star}(s) } \\
	& = \sum_{a  = \pi^{\star}(s) } \pi(a|s) \sum_{s'\in S_{k(s)+1}} P(s'|s,a) \rbr{ V^\pi(s') - V^{\pi^\star}(s)  }  + \sum_{a \neq \pi^{\star}(s) }  \pi(a|s) \rbr{ Q^\pi(s,a) - V^{\pi^\star}(s) }. 
	\end{align*}
	By \pref{lem:cond_occup_expand} (which again repeatedly expands $V^\pi(s') - V^{\pi^\star}(s)$ in the same way), we obtain
	\begin{equation}
	\label{eq:red_box_eq_3}
		\begin{split}  
		V^\pi(s_0) - V^{\pi^\star}(s_0)  & =  \sum_{s\neq s_L} \sum_{a\neq \pi^{\star}(s)} q^{\star}_{\pi}(s,a)  \rbr{ Q^\pi(s,a) - V^{\pi^\star}(s)}.
		\end{split} 
	\end{equation}
	
	Finally, combining \pref{eq:red_box_eq_2} and \pref{eq:red_box_eq_3}, we arrive at
	\begin{align*}
	& V^\pi(s_0) - V^{\pi^\star}(s_0)  \\ 
	& =  \sum_{s\neq s_L}  \sum_{a\neq \pi^{\star}(s)} q^{\star}_{\pi}(s,a) \rbr{ Q^\pi(s,a) - \whatV^{u}(s) }  + \sum_{s\neq s_L} \sum_{a\neq \pi^{\star}(s)} q^{\star}_{\pi}(s,a) \rbr{ \whatV^{u}(s) - V^{\pi^\star}(s) }  \\  
	& =  \sum_{s\neq s_L}  \sum_{a \neq \pi^{\star}(s)} q(s,a) {\widehat{E}^u }(s,a)  \tag{\text{Transition Error of Sub-opt actions}} \\  
	& \quad + \sum_{s\neq s_L} \sum_{a  = \pi^{\star}(s) } \rbr{ q(s,a) - q^{\star}_{\pi}(s,a) }  {\widehat{E}^u }(s,a)  
	\tag{\text{Transition Error of Opt actions}} \\
	& \quad +  \sum_{s\neq s_L} \sum_{a \in A} q(s,a)\rbr{\whatQ^{u}(s,a) -\whatV^{u}(s)} \tag{\text{Policy Difference}}  \\
	& \quad -  \sum_{s\neq s_L} \sum_{a  = \pi^{\star}(s)} q^{\star}_{\pi}(s,a) \rbr{ \whatQ^{u}(s,a) - \whatV^{u}(s) } \tag{\text{Estimation Bias 1}}  \\
	& \quad + \sum_{s\neq s_L}  \sum_{a\neq \pi^{\star}(s)} q^{\star}_{\pi}(s,a) \rbr{\whatV^{u}(s) - V^{\pi^\star}(s) } \tag{\text{Estimation Bias 2}}
	\end{align*}
	finishing the proof.
\end{proof}

\begin{proof} (Proof of \pref{col:general_perf_decomp}) By applying \pref{lem:general_perf_decomp} with $u = \pi^{\star}$, we know that $V_t^{\pi_t}(s_0) - V_t^{\pi^\star}(s_0)$ equals to 
\begin{align*}
&  \sum_{s\neq s_L} \sum_{a \neq \pi^{\star}(s)} q_t(s,a) { \widehat{E}_t }^{\pi^{\star}}(s,a) \\  
& \quad + \sum_{s\neq s_L} \sum_{a = \pi^{\star}(s) } \rbr{ q_t(s,a) - q^{\star}_{t}(s,a) }  { \widehat{E}_t }^{\pi^{\star}}(s,a) \\
& \quad +  \sum_{s\neq s_L} \sum_{a \in A} \widehat{q}_t(s,a) \rbr{\widehat{Q}_t^{\pi^{\star}}(s,a) -\widehat{V}_t^{\pi^{\star}}(s)} \\  
& \quad +  \sum_{s\neq s_L} \sum_{a \in A} \rbr{ q_t(s,a) - \widehat{q}_t(s,a)} \rbr{\widehat{Q}_t^{\pi^{\star}}(s,a) -\widehat{V}_t^{\pi^{\star}}(s)}  \\
& \quad -  \sum_{s\neq s_L} \sum_{a = \pi^{\star}(s)} q^{\star}_{t}(s,a)  \rbr{ \widehat{Q}_t^{\pi^{\star}}(s,a) - \widehat{V}_t^{\pi^{\star}}(s) } \tag{Estimation Bias 1} \\
& \quad + \sum_{s\neq s_L}  \sum_{a\neq \pi^{\star}(s)}q^{\star}_{t}(s,a)  \rbr{\widehat{V}^{\pi^{\star}}_t(s) - V^{\pi^\star}_t(s) }. \tag{Estimation Bias 2} 
\end{align*}
Now observe the following two facts.
First, the third term above is in fact equal to $\whatV_t^{\pi_t}(s_0) - \whatV_t^{\pi^\star}(s_0)$ according to the standard performance difference lemma~\citep[Theorem~5.2.1]{kakade2003sample}. 
Second, the first estimation bias term is simply $0$ since
$\widehat{Q}_t^{\pi^{\star}}(s,a) = \widehat{V}_t^{\pi^{\star}}(s)$ when $a = \pi^{\star}(s)$.

Therefore, by taking the summation over $t$, we obtain
\begin{align*}
\textsc{Err1}+\textsc{Err2}&= \sum_{t=1}^{T} \rbr{ V_t^{\pi_t}(s_0) - V_t^{\pi^\star}(s_0) } - \rbr{ \whatV_t^{\pi_t}(s_0) - \whatV_t^{\pi^\star}(s_0) } \\
& = \sum_{s\neq s_L} \sum_{a \neq \pi^{\star}(s)} q_t(s,a) { \widehat{E}_t }^{\pi^{\star}}(s,a) \\  
& \quad + \sum_{s\neq s_L} \sum_{a = \pi^{\star}(s) } \rbr{ q_t(s,a) - q^{\star}_{t}(s,a) }  { \widehat{E}_t }^{\pi^{\star}}(s,a) \\
& \quad +  \sum_{s\neq s_L} \sum_{a \in A} \rbr{ q_t(s,a) - \widehat{q}_t(s,a)} \rbr{\widehat{Q}_t^{\pi^{\star}}(s,a) -\widehat{V}_t^{\pi^{\star}}(s)}  \\
& \quad + \sum_{s\neq s_L}  \sum_{a\neq \pi^{\star}(s)}q^{\star}_{t}(s,a)  \rbr{\widehat{V}^{\pi^{\star}}_t(s) - V^{\pi^\star}_t(s) } 
\end{align*} 
which finishes the proof. 
\end{proof}

\begin{lemma}\label{lem:cond_occup_expand}   For any functions $F: S \rightarrow \fR$ and $C: S \rightarrow \fR$ satisfying the following condition:
	\[
	F(s) = \sum_{a = \pi^\star(s)} \pi(a|s) \sum_{s'\in S_{k(s)+1}} P(s'|s,a) F(s')  +  C(s)
	\]
	and $F(s_L) = 0$,
	we have
	\[
	F(s_0) = \sum_{s \neq s_L} q^{\star}_{\pi}(s) C(s). 
	\]
\end{lemma}
\begin{proof} By definition and direct calculation, we have $F(s_0)$ equal to 
	\begin{align*}
	& \sum_{a =\pi^\star(s_0)} q(s_0,a) \sum_{s'\in S_{1}} P(s'|s_0,a) F(s') +   C(s) \tag{$q(s_0)  = 1$}\\
	& =  \sum_{s_1\in S_{1}} q^{\star}_{\pi}(s_1) F(s_1) +  q^{\star}_{\pi}(s_0) C(s) \\ 
	&  = \sum_{s_1\in S_{1}} q^{\star}_{\pi}(s_1) \rbr{ \sum_{a = \pi^\star(s)} \pi(a|s) \sum_{s'\in S_{2}} P(s'|s,a) F(s') } + \sum_{k=0}^{1} \sum_{s\in S_k}  q^{\star}_{\pi}(s) C(s) \\
	&  = \sum_{s_2\in S_{2}} q^{\star}_{\pi}(s_2) F(s_2)  + \sum_{k=0}^{1} \sum_{s\in S_k} q^{\star}_{\pi}(s) C(s) \tag{definition of $q^{\star}_{\pi}(s)$}  \\
	& = \sum_{s_L \in S_L} q^\star_{\pi}(s_L) F(s_L) + \sum_{k=0}^{L-1} \sum_{s\in S_k} q^{\star}_{\pi}(s) C(s) \tag{repeatedly expanding} \\
	& =    \sum_{s\neq s_L} q^{\star}_{\pi}(s) C(s), \tag{$F(s_L)=0$} 
	\end{align*}
which completes the proof.
\end{proof}

\subsection{Self-bounding Terms}
\label{app:self_bounding_terms}
In this section, we summarize all the self-bounding terms we  use in the proofs for the unknown transition settings. 




\begin{definition}[Self-bounding Terms] \label{def:self_bounding_terms}
For some mapping $\pi^\star : S\rightarrow A$, define the following:
\begin{equation*}
\begin{aligned}
{\selfterm_1}(\cons) & = \sum_{t=1}^{T}\sum_{s\neq s_L} \sum_{a \neq \pi^{\star}(s)} q_t(s,a) \sqrt{\frac{J}{ \max\cbr{m_{i(t)}(s,a)} }}, \\
{\selfterm_2}(\cons) & =  \sum_{t=1}^{T}\sum_{s\neq s_L} \sum_{a = \pi^{\star}(s)} \rbr{ q_t(s,a) - q^\star_t(s,a) }  \sqrt{\frac{\cons}{ \max\cbr{m_{i(t)}(s,a) ,1 } }} ,\\
{\selfterm_3}(\cons) & =  \sum_{t=1}^{T}\sum_{s\neq s_L} \sum_{a \neq \pi^{\star}(s)} \sum_{k=0}^{k(s)-1} \sum_{ (u, v, w)\in T_k } q_t(u,v) \sqrt{\frac{P(w|u,v)\cdot \cons}{ \max\cbr{m_{i(t)}(u,v),1} } } q_t(s,a|w), \\
{\selfterm_4}(\cons) & = \sqrt{\cons\cdot \sum_{t=1}^{T}\sum_{s\neq s_L} \sum_{a \neq \pi^{\star}(s)} q_t(s,a)  } , \\ 
{\selfterm_5}(\cons) & = \sum_{s\neq s_L} \sum_{a \neq \pi^{\star}(s)} \sqrt{\cons \sum_{t=1}^{T} q_t(s,a)  } ,  \\ 
{\selfterm_6}(\cons) & = \sum_{t=1}^{T}\sum_{s\neq s_L} \sum_{a = \pi^{\star}(s)} \frac{q_t(s,a) - q^\star_t(s,a) }{ q_t(s,a)} \rbr{ \sum_{k=0}^{k(s)-1} \sum_{ (u, v, w)\in T_k } q_t(u,v) \sqrt{\frac{P(w|u,v)\cdot \cons}{ \max\cbr{m_{i(t)}(u,v),1} } } q_t(s,a|w)}. \\
\end{aligned}
\end{equation*}
\end{definition}

In the next six lemmas, we show that each of these six functions enjoys a certain self-bounding property under Condition~\eqref{eq:loss_condition} so that they are small whenever the regret of the learner is small.
In all these lemmas, the policy $\pi^\star$ used in $\selfterm_1$-$\selfterm_6$ coincides with the $\pi^\star$ in Condition~\eqref{eq:loss_condition}.
Also note that \pref{lem:main_text_self_bounding} is simply a collection of the first four lemmas.

\begin{lemma} \label{lem:self_bounding_term_1} Suppose Condition~\eqref{eq:loss_condition} holds. Then we have for any $\alpha \in \fR_{+}$,
	\begin{align*}
	\E\sbr{ \selfterm_1(\cons)} \leq \alpha \cdot \rbr{ \Reg_T(\pi^\star)+ C }+   \frac{1}{\alpha } \sum_{s\neq s_L}\sum_{a\neq \pi^{\star}(s)} \frac{8\cons }{\gap(s,a) }.
	\end{align*}
\end{lemma}
\begin{proof}  
	
	Under the condition, for any $\alpha \in \fR_{+}$, we have 
	\begin{align*}
	\selfterm_1(\cons) & = \sum_{t=1}^{T}\sum_{s\neq s_L} \sum_{a \neq \pi^{\star}(s)} q_t(s,a) \rbr{ \sqrt{\frac{\cons}{ \max\cbr{m_{i(t)}(s,a),1} }} - \alpha \gap(s,a)} + \alpha\sum_{t=1}^{T}\sum_{s\neq s_L} \sum_{a \neq \pi^{\star}(s)} q_t(s,a)\gap(s,a)
	\end{align*}
	where the expectation of the last term is bounded by $\alpha \cdot \rbr{\Reg_T(\pi^\star) + C }$.
	It thus remains to bound the first term.
	To this end,
	for a fixed state-action pair $(s,a)$, we define $N_{s,a}$ as the last epoch where
	the term in the bracket is still positive, so that:
	\[
	 m_{N_{s,a}+1}(s,a) \leq \frac{2J}{\alpha^2 \gap(s,a)^2 }
	\]
	due to the doubling epoch schedule.
	Then we have  
	\begin{align*}
	& \E\sbr{ \sum_{t=1}^{T} q_t(s,a) \rbr{ \sqrt{\frac{\cons}{ \max\cbr{m_{i(t)}(s,a),1} }} - \alpha \gap(s,a)} }\\ 
	& = \E\sbr{ \sum_{i=1}^{N} \rbr{ m_{i+1}(s,a) - m_{i}(s,a)  }  \rbr{ \sqrt{\frac{\cons}{ \max\cbr{m_{i}(s,a),1} }} - \alpha \gap(s,a)} } \\
	& \leq \E\sbr{ \sum_{i=1}^{N_{s,a} } \rbr{ m_{i+1}(s,a) - m_{i}(s,a)  }  \rbr{ \sqrt{\frac{\cons}{ \max\cbr{m_{i}(s,a),1} }} - \alpha \gap(s,a)} } \\ 
	& \leq \E\sbr{ 2 \int_{0}^{ m_{N_{s,a}+1}(s,a) }\sqrt{\frac{\cons}{x} } dx }  \leq \E\sbr{ 2 \int_{0}^{  \frac{2J}{\alpha^2 \gap(s,a)^2 } }  \sqrt{\frac{\cons}{ x} } dx }   \\
	& \leq 4 \cdot \sqrt{\cons} \cdot \sqrt{ \frac{2\cons}{\alpha^2 \gap(s,a)^2 } } \leq  \frac{8\cons }{\alpha \gap(s,a) }.
	\end{align*}
		Taking the summation over all state-action pairs $(s,a)$ satisfying $a \neq \pi^{\star}(s)$, we thus have 
		\begin{align*}
		\E\sbr{ \selfterm_2(\cons) } & \leq \alpha \cdot ( \Reg_T(\pi^\star)+ C)  + \sum_{s \neq s_L } \sum_{ a \neq \pi^\star(s) }  \frac{8\cons }{\alpha \gap(s,a) }.
		\end{align*}
\end{proof}

\begin{lemma} \label{lem:self_bounding_term_2} Suppose Condition~\eqref{eq:loss_condition} holds. Then we have for any $\beta \in \fR_{+}$, 
	\begin{align*}
	\E\sbr{ \selfterm_2(\cons) } \leq \beta \cdot ( \Reg_T(\pi^\star)+ C)  +   \frac{1}{\beta } \cdot \frac{8|S|L\cons }{\gapmin }.
	\end{align*}
\end{lemma}
\begin{proof}	Clearly, under the condition, for any $\beta \in \fR_{+}$, we have 
	\begin{align*}
	\selfterm_2(\cons) & = \sum_{t=1}^{T}\sum_{s\neq s_L} \sum_{a = \pi^{\star}(s)} \rbr{ q_t(s,a) - q^\star_t(s,a) }  \rbr{ \sqrt{\frac{\cons}{ \max\cbr{m_{i(t)}(s,a),1} }} - \beta \cdot \frac{\gapmin}{L} } \\
	& + \beta\sum_{t=1}^{T}\sum_{s\neq s_L} \sum_{a =\pi^{\star}(s)} \rbr{ q_t(s,a) - q^\star_t(s,a) } \cdot \frac{\gapmin}{L} 
	\end{align*}
	where the expectation of the last term is bounded by $\beta \cdot \rbr{\Reg_T(\pi^\star) + C}$ according to \pref{lem:perf_diff_lower_bound_gapmin} (deferred to the end of this subsection). 
	It thus remains to bound the first term.
	To this end,
	for a fixed state-action pair $(s,a)$, we similarly define $N_{s,a}$ as the last epoch where
	the term in the bracket is still positive, so that:
	\[
	 m_{N_{s,a}+1}(s,a) \leq \frac{2\cons L^2}{\beta^2 \gapmin^2 }
	\]
	due to the doubling epoch schedule.	
	Then, we have  
	\begin{align*}
	& \E\sbr{ \sum_{t=1}^{T} (q_t(s,a) - q_t^\star(s,a)) \rbr{ \sqrt{\frac{\cons}{ \max\cbr{m_{i(t)}(s,a),1} }} - \beta \cdot \frac{\gapmin}{L}  } }\\ 
	& \leq \E\sbr{ \sum_{i=1}^{N_{s,a} } \rbr{ m_{i+1}(s,a) - m_{i}(s,a)  }  \rbr{ \sqrt{\frac{\cons}{ \max\cbr{m_{i}(s,a),1} }} - \beta \cdot \frac{\gapmin}{L}  } } \tag{$q_t(s,a) \geq q_t^\star(s,a)$ by definition} \\ 
	& \leq \E\sbr{ 2 \int_{0}^{ m_{N_{s,a}+1}(s,a) }\sqrt{\frac{\cons}{x} } dx }  \leq \E\sbr{ 2 \int_{0}^{  \frac{2\cons L^2}{\beta^2 \gapmin^2 } }  \sqrt{\frac{\cons}{ x} } dx }   \\
	& \leq 4 \cdot \sqrt{\cons} \cdot \sqrt{  \frac{2\cons L^2}{\beta^2 \gapmin^2 } } \leq  \frac{8L \cons}{\beta \gapmin}.
	\end{align*}
	Taking the summation over all state-action pairs satisfying $a = \pi^{\star}(s)$, we have 
	\begin{align*}
	\E\sbr{ \selfterm_2(\cons) } & \leq \beta \cdot ( \Reg_T(\pi^\star)+ C)  + \sum_{s \neq s_L } \sum_{ a = \pi^\star(s) } \frac{8L \cons}{\beta \gapmin} \\ 
	& = \beta \cdot \rbr{  \Reg_T(\pi^\star)+ C } +  \frac{8|S| L \cons}{\beta \gapmin}.
	\end{align*}
\end{proof}

\begin{lemma} \label{lem:self_bounding_term_3} Suppose Condition~\eqref{eq:loss_condition} holds. Then we have for any $\alpha,\beta \in \fR_{+}$,
	\begin{align*}
	\E\sbr{ \selfterm_3(\cons)} \leq \rbr{\alpha + \beta } \cdot ( \Reg_T(\pi^\star)+ C )  +   \frac{1}{\alpha } \cdot\sum_{s\neq s_L}\sum_{ a\neq \pi^{\star}(s)} \frac{8L^2|S|\cons }{\gap(s,a) } + \frac{1}{\beta } \cdot \frac{8L^2|S|^2\cons }{\gapmin }  .
	\end{align*}
\end{lemma}
\begin{proof} First we have 
	\begin{align*} 
	\selfterm_3(\cons)  & =  \sum_{t=1}^{T} \sum_{k=0}^{L-1} \sum_{ (u, v, w)\in T_k } q_t(u,v) \sqrt{\frac{P(w|u,v)\cdot \cons}{ \max\cbr{m_{i(t)}(s,a)} } } \rbr{ \sum_{l=k+1}^{L-1} \sum_{s\in S_l} \sum_{a \neq \pi^{\star}(s)}  q_t(s,a|w) } \\
	& = \sum_{t=1}^{T} \sum_{k=0}^{L-1} \sum_{ u \in S_k } \sum_{v \neq \pi^{\star}(s)} q_t(u,v) \rbr{ \sum_{w\in S_{k+1}}\sqrt{\frac{P(w|u,v)\cdot \cons}{\max\cbr{m_{i(t)}(s,a),1 } }}  \sum_{l=k+1}^{L-1} \sum_{s\in S_l} \sum_{a \neq \pi^{\star}(s)} q_t(s,a|w) } \\
	& + \sum_{t=1}^{T} \sum_{k=0}^{L-1} \sum_{ u \in S_k } \sum_{v = \pi^{\star}(s)} q_t(u,v) \rbr{ \sum_{w\in S_{k+1}}\sqrt{\frac{P(w|u,v)\cdot \cons}{\max\cbr{m_{i(t)}(s,a),1 } }}  \sum_{l=k+1}^{L-1} \sum_{s\in S_l} \sum_{a \neq \pi^{\star}(s)} q_t(s,a|w) } \\ 
	& \leq \sum_{t=1}^{T} \sum_{k=0}^{L-1} \sum_{ u \in S_k } \sum_{v \neq \pi^{\star}(s)} q_t(u,v) \cdot \sqrt{\frac{L^2|S|\cdot \cons}{\max\cbr{m_{i(t)}(s,a),1 } }}   \\
	& + \sum_{t=1}^{T} \sum_{k=0}^{L-1} \sum_{ u \in S_k } \sum_{v = \pi^{\star}(s)} q_t(u,v) \rbr{ \sum_{w\in S_{k+1}}\sqrt{\frac{P(w|u,v)\cdot \cons}{\max\cbr{m_{i(t)}(s,a),1 } }}  \sum_{l=k+1}^{L-1} \sum_{s\in S_l} \sum_{a \neq \pi^{\star}(s)} q_t(s,a|w) } 
	\end{align*}
	where the second step separates the optimal and sub-optimal state-action pairs, and the inequality follows from the fact $\sum_{s\neq s_L}\sum_{a\in A} q_t(s,a|w) \leq L$ and the Cauchy-Schwarz inequality. Note that, the first term is simply $\selfterm_1(L^2|S|)$ and can be applied using \pref{lem:self_bounding_term_1}. 
	
	To bound the last term, we first observe the following
	\begin{align*}
	& \sum_{t=1}^{T} \sum_{k=0}^{L-1} \sum_{ u \in S_k } \sum_{v = \pi^{\star}(s)} q_t(u,v) \rbr{ \sum_{w\in S_{k+1}} \rbr{  P(w|u,v) \cdot\frac{\gapmin}{L}  }  \sum_{l=k+1}^{L-1} \sum_{s\in S_l} \sum_{a \neq \pi^{\star}(s)} q_t(s,a|w) } \\
	& = \sum_{t=1}^{T} \sum_{l=0}^{L-1} \sum_{s\in S_l} \sum_{a \neq \pi^{\star}(s)}  \frac{\gapmin}{L} \cdot \rbr{  \sum_{k=0}^{l-1} \sum_{ u \in S_k } \sum_{v = \pi^{\star}(s)} \sum_{w\in S_{k+1}} q_t(u,v)   P(w|u,v)   q_t(s,a|w) } \\
	& \leq \sum_{t=1}^{T} \sum_{l=0}^{L-1} \sum_{s\in S_l} \sum_{a \neq \pi^{\star}(s)}  \frac{\gapmin}{L} \cdot \rbr{   \sum_{k=0}^{l-1}   q_t(s,a) } \\
	& \leq \sum_{t=1}^{T} \sum_{l=0}^{L-1} \sum_{s\in S_l} \sum_{a \neq \pi^{\star}(s)} q_t(s,a) \gapmin 
	\end{align*}
	where the expectation of the last term is bounded by $\Reg_T(\pi^\star) + C$ under Condition~\eqref{eq:loss_condition}.
	
	Let $\text{cilp}\sbr{x} = \max\cbr{x , 0}$ be  the clipping function that removes the negative value. By adding and subtracting $\beta$ times the  term above, we have   
	\begin{align*}
	& \sum_{t=1}^{T} \sum_{k=0}^{L-1} \sum_{ u \in S_k } \sum_{v = \pi^{\star}(s)} q_t(u,v) \rbr{ \sum_{w\in S_{k+1}}\sqrt{\frac{P(w|u,v)\cdot \cons}{\max\cbr{m_{i(t)}(s,a),1 } }}  \sum_{l=k+1}^{L-1} \sum_{s\in S_l} \sum_{a \neq \pi^{\star}(s)} q_t(s,a|w) } \\
	& = \beta \sum_{t=1}^{T} \sum_{k=0}^{L-1} \sum_{ u \in S_k } \sum_{v = \pi^{\star}(s)} q_t(u,v) \rbr{ \sum_{w\in S_{k+1}} \rbr{  P(w|u,v) \cdot\frac{\gapmin}{L}  }  \sum_{l=k+1}^{L-1} \sum_{s\in S_l} \sum_{a \neq \pi^{\star}(s)} q_t(s,a|w) } \\
	& + \sum_{t=1}^{T} \sum_{k=0}^{L-1} \sum_{ u \in S_k } \sum_{v = \pi^{\star}(s)} q_t(u,v) \rbr{ \sum_{w\in S_{k+1}} \rbr{\sqrt{\frac{P(w|u,v)\cdot \cons}{\max\cbr{m_{i(t)}(s,a),1 } }} - \beta \cdot \frac{\gapmin P(w|u,v)}{L} }  \sum_{l=k+1}^{L-1} \sum_{s\in S_l} \sum_{a \neq \pi^{\star}(s)} q_t(s,a|w) }  \\
	& \leq \beta \sum_{t=1}^{T} \sum_{l=0}^{L-1} \sum_{s\in S_l} \sum_{a \neq \pi^{\star}(s)} q_t(s,a) \gapmin  \\
	& +  L \sum_{t=1}^{T} \sum_{k=0}^{L-1} \sum_{ u \in S_k } \sum_{v = \pi^{\star}(s)} \sum_{w\in S_{k+1}}  q_t(u,v) \clip \sbr{ \sqrt{\frac{P(w|u,v)\cdot \cons}{\max\cbr{m_{i(t)}(s,a),1 } }} - \beta \cdot \frac{\gapmin P(w|u,v)}{L}   } 
	\end{align*}
	where the last line follows from the facts $x \leq \clip[x]$ and $\sum_{s\neq s_L}\sum_{a \in A} q_t(s,a|w) \leq L$. 
	
	Fix a tuple $N_{u,v,w}$ where $v = \pi^{\star}(u)$, 
	we similarly define $N_{u,v,w}$ as the last epoch where
	the argument of $\clip(\cdot)$  is still positive, so that:
	\[
	 m_{N_{u,v,w}+1}(s,a) \leq \frac{2\cons L^2}{P(w|u,v)\beta^2 \gapmin^2 }
	\]
	due to the doubling epoch schedule.	
	Then, we have  
	\begin{align*}
	& \E\sbr{ \sum_{t=1}^{T} q_t(u,v) \clip\sbr{ \sqrt{\frac{P(w|u,v)\cdot  \cons}{\max\cbr{m_{i(t)}(s,a),1 } }} - \beta \cdot \frac{\gapmin P(w|u,v)}{L}   }  }\\ 
	& \leq \E\sbr{ \sum_{i=1}^{N_{u,v,w} } \rbr{ m_{i+1}(u,v) - m_{i}(u,v)  }  \clip\sbr{ \sqrt{\frac{P(w|u,v)\cdot \cons}{\max\cbr{m_{i(t)}(s,a),1 } }} - \beta \cdot \frac{\gapmin P(w|u,v)}{L}   } } \\ 
	& \leq \E\sbr{ 2 \int_{0}^{ m_{N_{u,v,w}+1}(s,a) }\sqrt{\frac{P(w|u,v)\cdot \cons}{x} } dx }  \leq \E\sbr{ 2 \int_{0}^{  \frac{2\cons L^2}{P(w|u,v)\beta^2 \gapmin^2 } }  \sqrt{\frac{P(w|u,v)\cons}{ x} } dx }   \\
	& \leq 4 \cdot \sqrt{P(w|u,v)\cdot \cons} \cdot \sqrt{  \frac{2\cons L^2}{P(w|u,v)\beta^2 \gapmin^2 } } \leq  \frac{8L\cons }{\beta \gapmin}.
	\end{align*}
	Taking the summation over  all transition tuple $(u,v,w)$  satisfying $v= \pi^{\star}(s)$ and adding $\E\sbr{\selfterm_1(L^2|S|\cons)}$, we have 
	\begin{align*}
	\E\sbr{\selfterm_3(\cons)} & \leq \beta \cdot ( \Reg_T(\pi^\star)+ C ) + \E\sbr{\selfterm_1(L^2 |S|\cons)} + L \sum_{k=0}^{L-1}\sum_{ u \in S_k} \sum_{ v = \pi^\star(u)} \sum_{w \in S_{k+1}}  \frac{8L\cons }{\beta \gapmin} \\
	& \leq \rbr{\alpha + \beta } \cdot ( \Reg_T(\pi^\star)+ C )  +   \frac{1}{\alpha } \cdot\sum_{s\neq s_L}\sum_{ a\neq \pi^{\star}(s)} \frac{8L^2|S|\cons }{\gap(s,a) } + \frac{1}{\beta } \cdot \frac{8L^2|S|^2\cons }{\gapmin }, 
	\end{align*}
	where the last line follows from the fact $\sum_{k=0}^{L-1} \abr{S_k} \abr{S_k+1} \leq |S|^2$.  
\end{proof}

\begin{lemma} \label{lem:self_bounding_term_4} Suppose Condition~\eqref{eq:loss_condition} holds. Then we have for any $\beta \in \fR_{+}$,
	\begin{align*}
	\E\sbr{ \selfterm_4(\cons)} \leq \beta \cdot \rbr{ \Reg_T(\pi^\star)+ C }+   \frac{1}{\beta } \cdot \frac{\cons }{4\gapmin }.
	\end{align*}
\end{lemma}
\begin{proof} By the fact that $2\sqrt{xy} \leq x + y $ for all $x,y\geq 0$, with Condition~\eqref{eq:loss_condition}, we have 
\begin{align*}
\E\sbr{ \selfterm_4(\cons)} &  = \E\sbr{ \sqrt{ 2 \beta \sum_{t=1}^{T} \sum_{s\neq s_L}\sum_{a\neq \pi^{\star}(s)} q_t(s,a) \gapmin \cdot \frac{\cons}{ 2\beta \gapmin}  }   } \\
& \leq \beta \cdot \E\sbr{ \sum_{t=1}^{T} \sum_{s\neq s_L}\sum_{a\neq \pi^{\star}(s)} q_t(s,a) \gapmin} + \frac{\cons}{4\beta \gapmin} \\
& \leq \beta \cdot  \rbr{ \Reg_T(\pi^\star)+ C }  +  \frac{\cons}{4\beta \gapmin}.
\end{align*}
\end{proof}

\begin{lemma} \label{lem:self_bounding_term_5} Suppose Condition~\eqref{eq:loss_condition} holds. Then we have for any $\alpha \in \fR_{+}$,
	\begin{align*}
	\E\sbr{ \selfterm_5(\cons)} \leq \alpha \cdot \rbr{ \Reg_T(\pi^\star)+ C }+    \sum_{s\neq s_L}\sum_{a\neq \pi^{\star}(s)} \frac{\cons}{4\alpha \gap(s,a)}.
	\end{align*}
\end{lemma}
\begin{proof} By the fact that $2\sqrt{xy} \leq x + y $ for all $x,y\geq 0$, with Condition~\eqref{eq:loss_condition}, we have 
	\begin{align*}
	\E\sbr{ \selfterm_4(\cons)} &  = \E\sbr{ \sum_{s\neq s_L}\sum_{a\neq \pi^{\star}(s)}  \sqrt{ 2 \alpha \sum_{t=1}^{T} q_t(s,a) \gap(s,a) \cdot \frac{\cons}{ 2\alpha \gap(s,a)}  }   } \\
	& \leq \alpha \cdot \E\sbr{ \sum_{t=1}^{T} \sum_{s\neq s_L}\sum_{a\neq \pi^{\star}(s)} q_t(s,a) \gap(s,a)} + \sum_{s\neq s_L}\sum_{a\neq \pi^{\star}(s)} \frac{\cons}{4\alpha \gap(s,a)} \\
	& \leq \alpha \cdot  \rbr{ \Reg_T(\pi^\star)+ C }  +   \sum_{s\neq s_L}\sum_{a\neq \pi^{\star}(s)} \frac{\cons}{4\alpha \gap(s,a)}.
	\end{align*}
\end{proof}

\begin{lemma} \label{lem:self_bounding_term_6} Suppose Condition~\eqref{eq:loss_condition} holds. Then we have for any $\beta \in \fR_{+}$,
	\begin{align*}
	\E\sbr{ \selfterm_6(\cons)} \leq \beta \cdot \rbr{ \Reg_T(\pi^\star)+ C }+   \frac{1}{\beta } \cdot \frac{8L^3|S|^2|A|  \cdot \cons }{\gapmin }.
	\end{align*}
\end{lemma}
\begin{proof} By adding and subtracting terms, we have $\selfterm_6(\cons)$ equals to  
\begin{align*}
& \sum_{t=1}^{T}\sum_{s\neq s_L} \sum_{a = \pi^{\star}(s) } \frac{ q_t(s,a) - q^\star_t(s,a)  }{ q_t(s,a) } \cdot \\
&\quad \rbr{  \sum_{k=0}^{k(s)-1} \sum_{(u,v,w)\in T_k} q_t(u,v) \sqrt{ \frac{ P(w|u,v) \ln \rbr{\frac{T|S||A|}{\delta}} }{ \max\cbr{ m_{i(t)}(u,v),1} }   } q_t(s,a|w) - \beta q_t(s,a) \cdot \frac{\gapmin}{L} } \\
& \quad +  \frac{\beta}{L} \sum_{t=1}^{T}\sum_{s\neq s_L} \sum_{a = \pi^{\star}(s) } \rbr{ q_t(s,a) - q^\star_t(s,a)  } \gapmin
\end{align*}
where the expectation of the last term is bounded by $\beta \cdot \rbr{ \Reg_T(\pi^\star) + C}$ according to \pref{lem:perf_diff_lower_bound_gapmin}. 

To bound the first term, we observe that
\begin{align*}
&  \sum_{k=0}^{k(s)-1} \sum_{(u,v,w)\in T_k} q_t(u,v) \sqrt{ \frac{ P(w|u,v) \ln \rbr{\frac{T|S||A|}{\delta}} }{ \max\cbr{ m_{i(t)}(u,v),1} }   } q_t(s,a|w)  - \beta q_t(s,a) \cdot \frac{\gapmin}{L} \\
& =  \sum_{k=0}^{k(s)-1}  \sum_{(u,v,w)\in T_k} q_t(u,v) \sqrt{ \frac{ P(w|u,v) \ln \rbr{\frac{T|S||A|}{\delta}} }{ \max\cbr{ m_{i(t)}(u,v),1} }   } q_t(s,a|w) \\
&\quad - \beta \cdot \frac{\gapmin}{L^2} \cdot \rbr{  \sum_{k=0}^{k(s)-1} \sum_{(u,v,w)\in T_k} q_t(u,v) P(w|u,v) q_t(s,a|w) }  \\
& = \sum_{k=0}^{k(s)-1} \sum_{(u,v,w)\in T_k} q_t(u,v) \rbr{  \sqrt{ \frac{ P(w|u,v) \ln \rbr{\frac{T|S||A|}{\delta}} }{ \max\cbr{ m_{i(t)}(u,v),1} }   }  -  P(w|u,v) \cdot \beta \cdot \frac{\gapmin}{L^2}  } \cdot q_t(s,a|w)  \\
& \leq \sum_{k=0}^{k(s)-1} \sum_{(u,v,w)\in T_k} q_t(u,v) \underbrace{ \clip\sbr{  \sqrt{ \frac{ P(w|u,v) \ln \rbr{\frac{T|S||A|}{\delta}} }{ \max\cbr{ m_{i(t)}(u,v),1} }   }  -  P(w|u,v) \cdot \beta \cdot \frac{\gapmin}{L^2}  }}_{=h_t(u,v,w)}q_t(s,a|w)
\end{align*} 
where the first equality uses $\sum_{(u,v,w)\in T_k} q_t(u,v) P(w|u,v) q_t(s,a|w) = q_t(s,a)$ for all layer $k = 0,\ldots k(s)-1$.  (Recall $\clip[x] = \max\{x,0\}$.)

Therefore, with Condition~\eqref{eq:loss_condition},  we bound the $\E\sbr{\selfterm_6(\cons)}$ by
\begin{align*}
&  \E\sbr{ \sum_{t=1}^{T}\sum_{s\neq s_L} \sum_{a = \pi^{\star}(s) } \frac{ q_t(s,a) - q^\star_t(s,a) }{ q_t(s,a) } \rbr{   \sum_{u,v,w}  q_t(u,v)  h_t(u,v,w) q_t(s,a|w) } + \beta \cdot \rbr{ \Reg_T(\pi^\star) + C} } \\
& \leq \E\sbr{ \sum_{t=1}^{T}\sum_{s\neq s_L} \sum_{a = \pi^{\star}(s) } \rbr{   \sum_{u,v,w}  q_t(u,v)  h_t(u,v,w) q_t(s,a|w) } } + \beta \cdot \rbr{ \Reg_T(\pi^\star) + C}\\
& \leq L \E\sbr{ \cdot \sum_{t=1}^{T}  \sum_{u,v,w}  q_t(u,v)  h_t(u,v,w) } + \beta \cdot \rbr{ \Reg_T(\pi^\star) + C}
\end{align*}
where the second line applies the fact  $\frac{ q_t(s,a) - q^\star_t(s,a) }{ q_t(s,a) }\leq 1$, and the third line changes summation order and uses the fact that $\sum_{s\neq s_L}\sum_{a\in A} q_t(s,a|w) \leq L$. 

Finally, following the similar idea of handing $\sum_{t=1}q_t(u,v)  h_t(u,v,w)$ as in \pref{lem:self_bounding_term_3}, we have 
\begin{align*}
\E\sbr{ \sum_{t=1}q_t(u,v)  h_t(u,v,w) } \leq \frac{8L^2 \cons}{ \beta \gapmin }. 
\end{align*}

By taking the summation over all transition triples, we have 
\begin{align*}
\E\sbr{ \selfterm_6(\cons)} & \leq \beta \cdot \rbr{ \Reg_T(\pi^\star)+ C }+   L \cdot \sum_{k=0}^{L-1} \sum_{ (u,v,w) \in T_k } \frac{1}{\beta } \cdot \frac{8L^2  \cdot \cons }{\gapmin } \\
& \leq \beta \cdot \rbr{ \Reg_T(\pi^\star)+ C }+    \frac{1}{\beta } \cdot \frac{8L^3|S|^2|A|  \cdot \cons }{\gapmin },
\end{align*}
where the last line follows from the fact that $\sum_{k=0}^{L}\abr{S_k} \abr{S_{k+1}} \leq |S|^2$. 
\end{proof}

\begin{lemma} Under Condition~\eqref{eq:loss_condition},  we have
	\begin{align*}
	\E\sbr{ \sum_{t=1}^{T}\sum_{s\neq s_L} \sum_{a = \pi^{\star}(s) } \rbr{ q_t(s,a) - q^\star_t(s,a) } \gapmin } \leq L \cdot \E\sbr{ {\Reg_T(\pi^\star) + C} }.
	\end{align*}
	\label{lem:perf_diff_lower_bound_gapmin}
\end{lemma}

\begin{proof}

For each $k$, we proceed as
\begin{align*} 
&  \sum_{s\in S_k} \sum_{a = \pi^{\star}(s) } \rbr{ q_t(s,a) - q^\star_t(s,a) } \notag \\ 
& \leq 1 - \sum_{s\in S_{k}}\sum_{a =\pi^\star(s)} q^{\star}_t(s,a) \tag{ $\sum_{s\in S_k} \sum_{a \in A }  q_t(s,a) = 1$  } \\
& = 1 -  \sum_{s\in S_{k}}\sum_{a =\pi^\star(s)} \pi_t(a|s) \Pr \sbr{ \left.\cbr{s_{k} = s} \; \bigcap \rbr{ \bigcap_{\tau=0}^{k-1} \cbr{a_\tau = \pi^\star(s_\tau)}  } \right\rvert  P, \pi_t} \notag \tag{definition of $q^{\star}_t$}\\
& = 1 -  \Pr \sbr{ \left.   \rbr{ \bigcap_{\tau=0}^{k} \cbr{a_\tau = \pi^\star(s_\tau)}  } \right\rvert  P, \pi_t} \notag \\
& = \Pr \sbr{ \left.   \rbr{ \bigcap_{\tau=0}^{k} \cbr{a_\tau = \pi^\star(s_\tau)} }^c \right\rvert  P, \pi_t}  \notag \\
& = \Pr \sbr{ \left.   \rbr{ \bigcup_{\tau=0}^{k} \cbr{a_\tau \neq \pi^\star(s_\tau)}  } \right\rvert  P, \pi_t}  \tag{De Morgan's laws} \\
& \leq \sum_{\tau=0}^{k} \Pr \sbr{ \left.    a_\tau \neq \pi^\star(s_\tau)   \right\rvert  P, \pi_t} \tag{union bound} \\
& = \sum_{\tau=0}^{k} \sum_{s \in S_\tau} \sum_{ a\neq \pi^\star(s) } q_t(s,a)  = \sum_{s \neq s_L} \sum_{a \neq \pi^\star(s)} q_t(s,a).
\end{align*}

Therefore, we have 
\begin{align*}
& \sum_{t=1}^{T}\sum_{s\neq s_L} \sum_{a = \pi^{\star}(s) } \rbr{ q_t(s,a) - q^{\star}_{\pi}(s,a) } \gapmin \\
& \leq L \cdot  \sum_{t=1}^{T}\sum_{s\neq s_L} \sum_{a \neq \pi^{\star}(s) } q_t(s,a) \cdot \gap(s,a) \\
& \leq L \cdot \E\sbr{ \Reg_T(\pi^\star)+ C} 
\end{align*}
where the last line follows from Condition~\eqref{eq:loss_condition}.
\end{proof}

\subsection{Supplementary Lemmas}



\begin{lemma} (Occupancy Measure Difference) For any policy $\pi$ and  transition functions $P_1$ and $P_2$, with $q_1 = q^{P_1, \pi}$ and $q_2 = q^{P_2, \pi}$ we have for all $s$,
	\begin{equation}
	\begin{split}
	q_1(s) - q_2(s)  & = \sum_{k=0}^{k(s)-1}\sum_{u\in S_{k}} \sum_{v\in A} \sum_{w\in S_{k+1}}q_1(u,v) \sbr{ P_1(w|u,v) - P_2(w|u,v)  } q_2(s|w) \\
	& = \sum_{k=0}^{k(s)-1}\sum_{u\in S_{k}} \sum_{v\in A} \sum_{w\in S_{k+1}}q_2(u,v) \sbr{ P_1(w|u,v) - P_2(w|u,v)  } q_1(s|w) 
	\end{split}
	\label{eq:exetnded_mdp_occup_diff}
	\end{equation}
	where the conditional occupancy measure $q_1(s'|s)$ (similarly for $q_2(s'|s)$) is defined recursively as 
	\begin{equation}
	q_1(s'|s) = 
	\begin{cases}
	0, & k(s') < k(s)  \text{ or }  ( k(s') = k(s) \text{ and } s' \neq s )\\
	1, & k(s') = k(s) \text{ and } s' = s \\
	\sum_{u\in S_{k(s') - 1}} q_1(u|s) \rbr{ \sum_{v \in A} \pi(v|u) P(s'|u,v)  }, & k(s') > k(s) 
	\end{cases}
	\label{eq:exetnded_mdp_state_transfer_def}
	\end{equation}
	which is the conditional probability of visiting state $s'$ from $s$ under $\pi$ and transition $P_1$. 
	\label{lem:general_occup_diff}
\end{lemma}

\begin{proof} Fix a state $s$. We proceed as:
	\begin{align*}
	& q_1(s) - q_2(s) \\
	& = \sum_{s' \in S_{k(s)-1}} \sum_{a' \in A} \rbr{ q_1(s',a') P_1(s|s',a') - q_2(s',a') P_2(s',a')} \\
	& = \sum_{s' \in S_{k(s)-1}} \sum_{a' \in A} \rbr{ q_1(s') - q_2(s')}  P_1(s|s',a') \pi(a'|s') \\ 
	& \quad +  \sum_{s' \in S_{k(s)-1}} \sum_{a' \in A} q_2(s',a')\rbr{ P_1(s|s',a') - P_2(s|s',a')} 
	\end{align*}
	where the second step follows by subtracting and adding $q_2(s',a') P_1(s|s',a')$. 
	Note that, $\sum_{a' \in A} \pi(a'|s') P_1(s|s',a')$ is exactly the conditional probability of transiting to state $s$ from state $s'$ with transition $P_1$. 
	Therefore, we have $\sum_{a' \in A} \pi(a'|s') P_1(s|s',a') = q_1(s|s')$ according to \pref{eq:exetnded_mdp_state_transfer_def}, and further expand $q_1(s) - q_2(s)$ as: 
	\begin{align*}
	& \sum_{s' \in S_{k(s)-1}} \sum_{a' \in A} \rbr{ q_1(s') - q_2(s')}  P_1(s|s',a') \pi(a'|s') \\ 
	& \quad +  \sum_{s' \in S_{k(s)-1}} \sum_{a' \in A} q_2(s',a')\rbr{ P_1(s|s',a') - P_2(s|s',a')}  \\ 
	& =  \sum_{s' \in S_{k(s)-1}} q_1(s|s') \rbr{ q_1(s') - q_2(s')} \\
	& \quad + \sum_{s' \in S_{k(s)-1}} \sum_{a'\in A} q_2(s',a') \sbr{ P_1(s|s',a') - P_2(s|s',a')  } q_1(s|s) 
	\end{align*}
	where the second line follows from the fact that $q_1(s|s) = 1$.  
	
	Therefore, we can recursively expand $q_1(s)-q_2(s)$ as:
	\begin{align*}
	&q_1(s) - q_2(s)  \\
	& = \sum_{s' \in S_{k(s)-1}}  \rbr{ q_1(s') - q_2(s')} q_1(s|s') \\
	& \quad + \sum_{s' \in S_{k(s)-1}} \sum_{a'\in A} q_2(s',a') \sbr{ P_1(s|s',a') - P_2(s|s',a')  } q_1(s|s) \\
	& = \sum_{s' \in S_{k(s)-1}} \rbr{ q_1(s') - q_2(s')} q_1(s|s')  \\
	& \quad + \sum_{k = k(s)}^{k(s)} \sum_{(u,v,w)\in T_k} q_2(u,v) \sbr{ P_1(w|u,v) - P_2(w|u,v)  } q_1(s|w) \\
	& = \sum_{s' \in S_{k(s)-1}} \rbr{ \sum_{s'' \in S_{k(s)-2}} \rbr{ q_1(s'') - q_2(s'')}  q_1(s'|s'')} q_1(s|s') \\
	& \quad + \sum_{k = k(s)-1}^{k(s)} \sum_{(u,v,w)\in T_k} q_2(u,v) \sbr{ P_1(s|s',a') - P_2(s|s',a')  } q_1(s|w) \\
	& = \sum_{s'' \in S_{k(s)-2}} \rbr{ q_1(s'') - q_2(s'')}  q_1(s|s'') + \sum_{k = k(s)-1}^{k(s)} \sum_{(u,v,w)\in T_k} q_2(u,v) \sbr{ P_1(s|s',a') - P_2(s|s',a')  } q_1(s|w) \\
  & = \sum_{k=0}^{k(s)-1}\sum_{u\in S_{k}} \sum_{v\in A} \sum_{w\in S_{k+1}}q_2(u,v) \sbr{ P_1(w|u,v) - P_2(w|u,v)  } q_1(s|w). \tag{expand recursively}
	\end{align*}
	where the second step follows from the fact that $q(s'|s) = 0$ for all states $s \neq s'$ with $k(s) = k(s')$, and the third step follows from the fact $\sum_{s' \in S_k} q(s'|s'') q(s|s') = q(s|s'')$ for all state pairs that $k(s)>k>k(s'')$.
	
	By applying the same technique, we also have 
	\[
	q_2(s) - q_1(s) = \sum_{k=0}^{k(s)-1}\sum_{u\in S_{k}} \sum_{v\in A} \sum_{w\in S_{k+1}}q_1(u,v) \sbr{ P_2(w|u,v) - P_1(w|u,v)  } q_2(s|w).
	\]
	Flipping this equality finishes the proof for the second statement of the lemma:
	\[
	q_1(s) - q_2(s)  = \sum_{k=0}^{k(s)-1}\sum_{u\in S_{k}} \sum_{v\in A} \sum_{w\in S_{k+1}}q_1(u,v) \sbr{ P_1(w|u,v) - P_2(w|u,v)  } q_2(s|w).
	\]\end{proof}


\begin{lemma} \label{lem:sa_conf_width_bound} The following holds:
\begin{equation*}
B_i(s,a) \leq 2 \sqrt{ \frac{ |S_{k(s)+1}| \ln \rbr{\frac{T|S||A|}{\delta}}}{\max \cbr{m_{i}(s,a) , 1}} } + \frac{ 14|S_{k(s)+1}| \ln\rbr{\frac{T|S||A|}{\delta}}}{3\max \cbr{m_{i}(s,a) , 1}}.
\end{equation*}
\end{lemma}
\begin{proof} By the definition of $B_i(s,a)$, we have 
\begin{align*}
B_i(s,a) & = \sum_{s' \in S_{k(s)+1}} B_i(s,a,s') \\
& = \sum_{s' \in S_{k(s)+1}} \rbr{ 2 \sqrt{\frac{\bar{P}_i(s'|s,a)\ln \rbr{\frac{T|S||A|}{\delta}}}{\max \cbr{m_{i}(s,a) , 1} }} + \frac{ 14 \ln\rbr{\frac{T|S||A|}{\delta}}}{3\max \cbr{m_{i}(s,a) , 1} }  }  \\ 
& \leq 2 \sqrt{ \frac{ |S_{k(s)+1}| \ln \rbr{\frac{T|S||A|}{\delta}}}{\max \cbr{m_{i}(s,a) , 1} }} + \frac{ 14|S_{k(s)+1}| \ln\rbr{\frac{T|S||A|}{\delta}}}{3\max \cbr{m_{i}(s,a) , 1} }  
\end{align*}
where the last line follows from the Cauchy-Schwarz inequality. 
\end{proof}

\begin{lemma} Conditioning on event $\calA$, we have
\begin{equation} \label{eq:conf_width_bound}
B_i(s,a,s') \leq 4 \sqrt{ \frac{ P(s'|s,a ) \ln \rbr{\frac{T|S||A|}{\delta}}}{\max \cbr{m_{i}(s,a) , 1} }} + \frac{40\ln\rbr{\frac{T|S||A|}{\delta}}}{3\max \cbr{m_{i}(s,a) , 1} }.
\end{equation}
\label{lem:conf_width_bound}
\end{lemma}
\begin{proof}
By direct calculation based on \pref{eq:confidence_width_def} and the condition of event $\calA$, we have 
\[
\begin{split}
B_i(s,a,s') & \leq 2 \sqrt{\frac{\bar{P}_i(s'|s,a)\ln \rbr{\frac{T|S||A|}{\delta}}}{\max \cbr{m_{i}(s,a) , 1} }} + \frac{ 14 \ln\rbr{\frac{T|S||A|}{\delta}}}{3\max \cbr{m_{i}(s,a) , 1} } \\
& \leq 2 \sqrt{ \frac{\rbr{  P(s'|s,a ) + B_i(s,a,s') }\ln \rbr{\frac{T|S||A|}{\delta}}}{\max \cbr{m_{i}(s,a) , 1} }} + \frac{14\ln\rbr{\frac{T|S||A|}{\delta}}}{3\max \cbr{m_{i}(s,a) , 1} } \\
& \leq 2 \sqrt{ \frac{ P(s'|s,a ) \ln \rbr{\frac{T|S||A|}{\delta}}}{\max \cbr{m_{i}(s,a) , 1} }} +  \sqrt{ \frac{ 4B_i(s,a,s') \ln \rbr{\frac{T|S||A|}{\delta}}}{\max \cbr{m_{i}(s,a) , 1} }} + \frac{14\ln\rbr{\frac{T|S||A|}{\delta}}}{3\max \cbr{m_{i}(s,a) , 1} } \\
& \leq 2\sqrt{ \frac{ P(s'|s,a ) \ln \rbr{\frac{T|S||A|}{\delta}}}{\max \cbr{m_{i}(s,a) , 1} }} + \frac{B_i(s,a,s')}{2} + \frac{20\ln\rbr{\frac{T|S||A|}{\delta}}}{3\max \cbr{m_{i}(s,a) , 1} }, 
\end{split}
\]
where the third line applies the fact that $\sqrt{x+y} \leq \sqrt{x} + \sqrt{y}$, and the last line follows from the fact $2\sqrt{xy} \leq x + y$ for $x,y>0$. 

Rearranging the terms yields that 
\[
B_i(s,a,s')  \leq 4 \sqrt{ \frac{ P(s'|s,a ) \ln \rbr{\frac{T|S||A|}{\delta}}}{\max \cbr{m_{i}(s,a) , 1} }} + \frac{40\ln\rbr{\frac{T|S||A|}{\delta}}}{3\max \cbr{m_{i}(s,a) , 1} }. 
\]
\end{proof}

Combining with the fact $B_i(s,a,s')\leq1$, we have the following tighter bound of confidence width. 
\begin{corollary}
\label{col:conf_width_bound} Conditioning on event $\calA$, we have
\begin{equation*} 
\begin{aligned}
B_i(s,a,s') & \leq \min\cbr{ 4  \sqrt{ \frac{ P(s'|s,a ) \ln \rbr{\frac{T|S||A|}{\delta}}}{\max \cbr{m_{i}(s,a) , 1} }}+ \frac{40\ln\rbr{\frac{T|S||A|}{\delta}}}{3\max \cbr{m_{i}(s,a) , 1} }, 1}\\
& \leq \min\cbr{  4  \sqrt{ \frac{ P(s'|s,a ) \ln \rbr{\frac{T|S||A|}{\delta}}}{\max \cbr{m_{i}(s,a) , 1} }} ,1 } + \min\cbr{  \frac{40\ln\rbr{\frac{T|S||A|}{\delta}}}{3\max \cbr{m_{i}(s,a) , 1} }, 1}. 
\end{aligned}
\end{equation*} 
\end{corollary}


We often use the following two lemmas to deal with the small-probability event $\calA^c$ when taking expectation.
\begin{lemma} \label{lem:exp_high_prob_bound} Suppose that a random variable $X$ satisfies the following conditions:
\begin{itemize}
	\item Conditioning on event $\calE$, $ X < Y$ where $Y > 0$ is another random variable; 
	\item $ X < C $ holds always for some fixed $C \in \fR+$. 
\end{itemize}
Then, we have
\begin{align*}
\E\sbr{ X } \leq C \cdot \Pr\sbr{ \calE^c  } + \E\sbr{Y}.   
\end{align*}
\end{lemma}
\begin{proof} By writing the random variable $X$ as $X \cdot \Ind{\calE} + X \cdot \Ind{\calE^c}$, and noting
\[
X \cdot \Ind{\calE} \leq Y \cdot \Ind{\calE} \leq Y , \text{ and } X \cdot \Ind{\calE^c} \leq C \cdot \Ind{\calE^c},
\]
we prove the statement after
taking the expectations. 
\end{proof}

\begin{lemma} \label{lem:exp_high_prob_bound_cond} Suppose that a random variable $X$ satisfies the following conditions:
	\begin{itemize}
		\item Conditioning on event $\calE$, $ X < Y$ where $Y>0$ is another random variable;
		\item $ X < C $ holds where $C$ is another random variable which ensures $\E\sbr{C|\calE^c } \leq D$ for some fixed $D\in \fR_{+}$. 
	\end{itemize}
	Then, we have
	\begin{align*}
	\E\sbr{ X } \leq D \cdot \Pr\sbr{ \calE^c  } + \E\sbr{Y}.   
	\end{align*}
\end{lemma}
\begin{proof} By writing the random variable $X$  as $X \cdot \Ind{\calE} + X \cdot \Ind{\calE^c}$, and noting
	\[
	X \cdot \Ind{\calE} \leq Y \cdot \Ind{\calE} \leq Y, \quad  X \cdot \Ind{\calE^c} \leq C \cdot \Ind{\calE^c}, \quad \E\sbr{C \cdot \Ind{\calE^c}} \leq \E\sbr{C|\calE^c },
	\]
we prove the statement after
taking the expectations. 
\end{proof}

\begin{lemma} (\citep[Lemma 10]{jin2019learning})\label{lem:aux}
	With probability at least $1 - 2\delta$, we have for all $k=0,\ldots L-1$, 
	\begin{equation}\label{eq:aux1}
	\sum_{t=1}^T\sum_{s\in S_k, a\in A} \frac{q_t(s,a)}{\max\{1, m_{i(t)}(s,a)\}} =
	\scO\left( |S_k||A|\ln T + \ln(L/\delta) \right)
	\end{equation}
	and
	\begin{equation}\label{eq:aux2}
	\sum_{t=1}^T\sum_{s\in S_k, a\in A} \frac{q_t(s,a)}{\sqrt{\max\{1, m_{i(t)}(s,a)\}}} =  \scO\left(\sqrt{|S_k||A| T}+ |S_k||A|\ln T + \ln(L/\delta) \right).
	\end{equation}	
	Simultaneously, for all $k<h$, we have 
	\begin{equation} \label{eq:aux3}
	\begin{aligned}
	&\sum_{t=1}^{T} \sum_{(u,v,w)\in T_k} \sum_{(x,y,z)\in T_h} q_t(u,v) \sqrt{ \frac{P(w|u,v)}{ \max\{1, m_{i(t)}(u,v)\} }  } \cdot  q_t(x,y|w) \sqrt{ \frac{P(z|x,y)}{ \max\{1, m_{i(t)}(x,y)\} }  } \\
	& =  \order\rbr{ \rbr{ \abr{A} \ln T + \ln\rbr{\nicefrac{L}{\delta}}} \cdot \sqrt{ \abr{S_k}\abr{S_{k+1}}\abr{ S_{h} }\abr{S_{h+1}} } }.
	\end{aligned}
	\end{equation}
	\label{lem:extened_mdp_bobw_tran_error_jin2019}
\end{lemma}

\begin{proof} \pref{eq:aux1} and \pref{eq:aux2} are from \cite{jin2019learning}. 
	For \pref{eq:aux3}, by direct calculation we have 
	\begin{align*}
	& \sum_{t=1}^{T} \sum_{(u,v,w)\in T_k} \sum_{(x,y,z)\in T_h} q_t(u,v) \sqrt{ \frac{P(w|u,v)}{ \max\{1, m_{i(t)}(u,v)\} }  } \cdot  q_t(x,y|w) \sqrt{ \frac{P(z|x,y)}{ \max\{1, m_{i(t)}(x,y)\} }  }  \\
	& = \sum_{t=1}^{T} \sum_{(u,v,w)\in T_k} \sum_{(x,y,z)\in T_h}  \sqrt{ \frac{q_t(u,v) P(z|x,y)q_t(x,y|w) }{ \max\{1, m_{i(t)}(u,v)\} }  } \cdot  \sqrt{ \frac{q_t(u,v)P(w|u,v)q_t(x,y|w) }{ \max\{1, m_{i(t)}(x,y)\} }  } \\
	& \leq \sqrt{ \sum_{t=1}^{T} \sum_{(u,v,w)\in T_k} \sum_{(x,y,z)\in T_h}  \frac{q_t(u,v) P(z|x,y)q_t(x,y|w) }{ \max\{1, m_{i(t)}(u,v)\} }  } \cdot  \sqrt{ \sum_{t=1}^{T} \sum_{(u,v,w)\in T_k} \sum_{(x,y,z)\in T_h}  \frac{q_t(u,v)P(w|u,v)q_t(x,y|w) }{ \max\{1, m_{i(t)}(x,y)\} }  } \\
	& \leq \sqrt{ |S_{k+1}| \sum_{t=1}^{T} \sum_{u \in S_k} \sum_{a\in A} \frac{q_t(u,v) }{ \max\{1, m_{i(t)}(u,v)\} }  } \cdot  \sqrt{ |S_{h+1}| \sum_{t=1}^{T}  \sum_{x\in S_h} \sum_{a\in A} \frac{q_t(x,y) }{ \max\{1, m_{i(t)}(x,y)\} }  }  \\
	& \leq \order\rbr{ \rbr{ \abr{A} \ln T + \ln\rbr{\nicefrac{L}{\delta}}} \cdot \sqrt{ \abr{S_k}\abr{S_{k+1}}\abr{ S_{h} }\abr{S_{h+1}} } }.
	\end{align*}
\end{proof}

\begin{lemma}
	For all $k = 0, \ldots, L-1$, we have 
	\begin{equation}\label{eq:aux1_exp}
	\E\sbr{ \sum_{t=1}^T\sum_{s\in S_k, a\in A} \frac{q_t(s,a)}{\max\{1, m_{i(t)}(s,a)\}} } =
	\scO\left( |S_k||A|\ln T +  |S_k||A|\right)
	\end{equation}
	and
	\begin{equation}\label{eq:aux2_exp}
	\E\sbr{\sum_{t=1}^T\sum_{s\in S_k, a\in A} \frac{q_t(s,a)}{\sqrt{\max\{1, m_{i(t)}(s,a)\}}}} =  \scO\left(\sqrt{|S_k||A| T} +  |S_k||A|\right).
	\end{equation}
	\label{lem:aux_exp}
\end{lemma}
\begin{proof} For each state-action pair $(s,a)$, we have  
\begin{align*}
& \E\sbr{\sum_{t=1}^{T}  \frac{ q_t(s,a)}{\max\{1, m_{i(t)}(s,a)\}} } \\
& =  \E\sbr{\sum_{t=1}^{T}   \frac{ \Indt{s,a} }{\max\{1, m_{i(t)}(s,a)\}}   }  =  \E\sbr{\sum_{i=1}^{N} \sum_{t=t_i}^{t_{i+1}-1}   \frac{ \Indt{s,a} }{\max\{1, m_{i}(s,a)\}}   } \\
& = \E\sbr{\sum_{i=1}^{N}   \frac{ m_{i+1}(s,a) - m_{i}(s,a) }{\max\{1, m_{i}(s,a)\}}   } \\
& \leq 2 \E\sbr{ 1 + \int_{1}^{ 1 + m_{N+1}(s,a) } \frac{dx}{x} } \leq  2\rbr{ 2 \ln T + 1 } 
\end{align*} 
where the second line follows from the definition of the indicator and occupancy measure $q_t$, and the last line applies the fact $m_{i+1}(s,a) \leq 2 m_{i}(s,a)$ when $m_{i}(s,a) \geq 1$. Taking the summation over all state-action pairs at layer $k$ finishes the proof of \pref{eq:aux1_exp}.

Similarly, we have 
\begin{align*}
& \E\sbr{\sum_{t=1}^{T}  \frac{ q_t(s,a)}{ \sqrt{\max\{1, m_{i(t)}(s,a)\}} }} \\
& =  \E\sbr{\sum_{t=1}^{T}   \frac{ \Indt{s,a} }{\sqrt{\max\{1, m_{i(t)}(s,a)\}}}   }  =  \E\sbr{\sum_{i=1}^{N} \sum_{t=t_i}^{t_{i+1}-1}   \frac{ \Indt{s,a} }{\sqrt{\max\{1, m_{i}(s,a)\}}}   } \\
& = \E\sbr{\sum_{i=1}^{N}   \frac{ m_{i+1}(s,a) - m_{i}(s,a) }{\sqrt{\max\{1, m_{i}(s,a)\}}}   } \\
& \leq 2 \E\sbr{ 1 + \int_{0}^{ m_{N+1}(s,a) } \frac{dx}{\sqrt{x}} } \leq  2\rbr{ 2 \sqrt{ m_{N+1}(s,a)} + 1 } 
\end{align*} 
where $m_{N+1}(s,a)$ is the total number of visiting state-action pair $(s,a)$. Taking the summation over all state-action pairs of layer $k$ yields that 
\begin{align*}
& \E\sbr{\sum_{s\in S_k}\sum_{a\in A} \sum_{t=1}^{T}  \frac{ q_t(s,a)}{ \sqrt{\max\{1, m_{i(t)}(s,a)\}} }}  \\
& \leq \sum_{s\in S_k}\sum_{a\in A} 2\rbr{ 2 \sqrt{ m_{N+1}(s,a)} + 1 }  \leq 2\rbr{ 2 \sqrt{ |S_k||A| T } + |S_k| |A| } 
\end{align*}
where the last inequality follows from the Cauchy-Schwarz inequality. 
\end{proof}

\begin{definition} (Residual Term)  We define the residual term $r_t(s,a)$ as 
\begin{equation}
\begin{aligned}
r_t(s,a) & =  \frac{40}{3} \sum_{k=0}^{k(s)-1}\sum_{(u,v,w)\in T_k} q_t(u,v) \cdot \frac{ P(w|u,v) \ln \rbr{ \frac{T|S||A|}{\delta}} }{ \max\cbr{m_{i(t)}(u,v)  ,1}} \cdot q_t(s,a|w)  \\
& + \sum_{k=0}^{k(s)-1} \sum_{h = k+1}^{k(s)-1} \sum_{(u,v,w)\in T_k} \sum_{(x,y,z)\in T_h} q_t(u,v) B_{i(t)}(u,v,w) q_t(x,y|w) B_{i(t)}(x,y,z) \\
& + \Ind{ \calA^c }. 
\end{aligned}	
\end{equation}
for all state-action pair $(s,a) \in S\times A$ and all episodes $t\in [T]$. 
\label{def:residual_terms}
\end{definition}

\begin{lemma} 
The following hold:
\[
 \abr{ q_t(s,a) - \widehat{q}_t(s,a) }  \leq r_t(s,a)  + 4 \sum_{k=0}^{k(s)-1} \sum_{(u,v,w)\in T_k} q_t(u,v)  \sqrt{ \frac{ P(w|u,v) \ln \pcons }{ \max\cbr{m_{i(t)}(u,v)  ,1}} }q_t(s,a|w)
\]
and
\begin{align*}
\E\sbr{ \sum_{t=1}^{T} \sum_{s\neq s_L} \sum_{a \in A} r_t(s,a) } = \order\rbr{ L^2|S|^3|A|^2 \ln^2 \rbr{ \frac{T|S||A|}{\delta}  } + |S||A| T \cdot \delta } .
\end{align*}
\label{lem:residual_term_property}
\end{lemma}
\begin{proof} For simplicity, we let $\pcons = \frac{T|S||A|}{\delta}$ and assume $\delta \in (0,1)$. 
According to the \pref{lem:general_occup_diff}, conditioning on event $\calA$, we have  
\begin{align*}
\abr{ q_t(s,a) - \widehat{q}_t(s,a) } & = \abr{ \sum_{k = 0}^{k(s)-1}  \sum_{(u,v,w)\in T_k} q_t(u,v) \rbr{ P(w|u,v) - \bar{P}_{i(t)}(w|u,v) } \whatq_t(s,a|w) } \\ 
& \leq \sum_{k = 0}^{k(s)-1} \sum_{(u,v,w)\in T_k}  q_t(u,v)  \abr{ P(w|u,v) - \bar{P}_{i(t)}(w|u,v)} \whatq_t(s,a|w) \\
& \leq \sum_{k = 0}^{k(s)-1} \sum_{(u,v,w)\in T_k}  q_t(u,v)  B_{i(t)}(u,v,w) \whatq_t(s,a|w)
\end{align*}
Moreover, we apply \pref{lem:general_occup_diff} again to conditional occupancy measure and obtain
\begin{align*}
\abr{ q_t(s,a|w) - \widehat{q}_t(s,a|w) } & \leq \sum_{h  = k(w) }^{k(s)-1} \sum_{(x,y,z)\in T_h }  q_t(x,y|w)  B_{i(t)}(x,y,z) \whatq_t(s,a|z) \\
& \leq \sum_{h  = k(w) }^{k(s)-1} \sum_{(x,y,z)\in T_h }  q_t(x,y|w)  B_{i(t)}(x,y,z)
\end{align*}
where the second line applies the fact $ \whatq_t(s,a|z) \leq 1$. 

Combining these inequalities yields (under the event $\calA$)
\begin{align*}
& \abr{ q_t(s,a) - \widehat{q}_t(s,a) } \\
& \leq \sum_{k = 0}^{k(s)-1} \sum_{(u,v,w)\in T_k}  q_t(u,v)  B_{i(t)}(u,v,w) q_t(s,a|w) \\
& \quad + \sum_{k = 0}^{k(s)-1} \sum_{(u,v,w)\in T_k}  q_t(u,v)  B_{i(t)}(u,v,w)  \rbr{ \sum_{h  = k(w) }^{k(s)-1} \sum_{(x,y,z)\in T_h }  q_t(x,y|w)  B_{i(t)}(x,y,z) } \\
& \leq 4 \sum_{k=0}^{k(s)-1} \sum_{(u,v,w)\in T_k} q_t(u,v)  \sqrt{ \frac{ P(w|u,v) \ln \pcons }{ \max\cbr{m_{i(t)}(u,v)  ,1}} }q_t(s,a|w) \\
& \quad + \frac{40}{3} \sum_{k=0}^{k(s)-1}\sum_{(u,v,w)\in T_k} q_t(u,v) \cdot \frac{ P(w|u,v) \ln \pcons }{ \max\cbr{m_{i(t)}(u,v)  ,1}} \cdot q_t(s,a|w)  \\
& \quad + \sum_{k=0}^{k(s)-1} \sum_{h = k+1}^{k(s)-1} \sum_{(u,v,w)\in T_k} \sum_{(x,y,z)\in T_h} q_t(u,v) B_{i(t)}(u,v,w) q_t(x,y|w) B_{i(t)}(x,y,z) 
\end{align*}
where the second line follows from \pref{lem:conf_width_bound}. 

On the other hand,  $ \abr{ q_t(s,a) - \widehat{q}_t(s,a) } \leq  1$ holds always. Combining the bounds of these two cases finishes the first statement. 

Recall the definition of the residual terms, we decompose the following into three terms $\textsc{Sum}_1$, $\textsc{Sum}_2$ and $\textsc{Sum}_3$:
\begin{align*}
& \E\sbr{ \sum_{t=1}^{T}\sum_{s\neq s_L} \sum_{a \in A} r_t(s,a)} \\
& = \underbrace{\frac{40}{3} \E\sbr{ \sum_{t=1}^{T}\sum_{s\neq s_L} \sum_{a \in A} \sum_{k=0}^{k(s)-1}\sum_{(u,v,w)\in T_k} q_t(u,v) \cdot \frac{ P(w|u,v) \ln \pcons }{ \max\cbr{m_{i(t)}(u,v)  ,1}} \cdot q_t(s,a|w) }}_{\triangleq \textsc{Sum}_1} \\
& + \underbrace{\E\sbr{ \sum_{t=1}^{T}\sum_{s\neq s_L} \sum_{a \in A}  \Ind{ \calA^c } }}_{\triangleq \textsc{Sum}_2}\\ 
& + \underbrace{\E\sbr{ \sum_{t=1}^{T}\sum_{s\neq s_L} \sum_{a \in A} \sum_{k=0}^{k(s)-1} \sum_{h = k+1}^{k(s)-1} \sum_{(u,v,w)\in T_k} \sum_{(x,y,z)\in T_h} q_t(u,v) B_{i(t)}(u,v,w) q_t(x,y|w) B_{i(t)}(x,y,z) }}_{\triangleq \textsc{Sum}_3}.
\end{align*}
Then, we show that these terms are all logarithmic in $T$.

\paragraph{$\textsc{Sum}_1$} By direct calculation, we have 
\begin{align}
\textsc{Sum}_1 & = \frac{40}{3} \E\sbr{ \sum_{t=1}^{T}\sum_{s\neq s_L} \sum_{a \in A} \sum_{k=0}^{k(s)-1}\sum_{(u,v,w)\in T_k} q_t(u,v) \cdot \frac{ P(w|u,v) \ln \pcons }{ \max\cbr{m_{i(t)}(u,v)  ,1}} \cdot q_t(s,a|w) } \notag \\
& = \frac{40}{3} \E\sbr{ \sum_{t=1}^{T}\sum_{k = 0}^{L-1}\sum_{(u,v,w)\in T_k} q_t(u,v) \cdot \frac{\ln \pcons }{ \max\cbr{m_{i(t)}(u,v)  ,1}} \cdot \rbr{ \sum_{s\neq s_L} \sum_{a \in A}   P(w|u,v) q_t(s,a|w)} } \notag \\
& \leq \frac{40L}{3}\ln \pcons \E\sbr{ \sum_{t=1}^{T}\sum_{u \neq s_L}\sum_{v\in A}  \cdot \frac{q_t(u,v)  }{ \max\cbr{m_{i(t)}(u,v)  ,1}}  } \notag \\
& =  \frac{80L}{3}\ln \pcons \rbr{ \sum_{k=0}^{L-1} |S_k||A| \rbr{ \ln T +1  } } = \order\rbr{ L|S||A| \ln^2\pcons   } \label{eq:residual_bound_sum_1}
\end{align}
where the first line follows from the property of occupancy measures, and the last line applies \pref{eq:aux1_exp} of \pref{lem:aux_exp}.

\paragraph{$\textsc{Sum}_2$} According to the definition of event $\calA$, we have 
\begin{align}
\textsc{Sum}_2 = \E\sbr{ \sum_{t=1}^{T}\sum_{s\neq s_L} \sum_{a \in A}  \Ind{ \calA^c } } = |S||A|T \cdot \E\sbr{ \Ind{ \calA^c }} = |S||A| T \cdot \delta. \label{eq:residual_bound_sum_2}
\end{align}

\paragraph{$\textsc{Sum}_3$} First, we consider the term inside the expectation bracket and show the following conditioning on event $\calA$:
\begin{align*}
& \sum_{t=1}^{T}\sum_{s\neq s_L} \sum_{a \in A} \sum_{k=0}^{k(s)-1} \sum_{h = k+1}^{k(s)-1} \sum_{(u,v,w)\in T_k} \sum_{(x,y,z)\in T_h} q_t(u,v) B_{i(t)}(u,v,w) q_t(x,y|w) B_{i(t)}(x,y,z)  \\
& \leq 4 \sum_{t=1}^{T}\sum_{s\neq s_L} \sum_{a \in A} \sum_{k=0}^{k(s)-1} \sum_{h = k+1}^{k(s)-1} \sum_{(u,v,w)\in T_k} \sum_{(x,y,z)\in T_h} q_t(u,v) \sqrt{\frac{ P(w|u,v) \ln \pcons }{ \max\cbr{m_{i(t)}(u,v)  ,1}}}  q_t(x,y|w) B_{i(t)}(x,y,z)  \\
& \quad + \frac{40}{3} \sum_{t=1}^{T}\sum_{s\neq s_L} \sum_{a \in A} \sum_{k=0}^{k(s)-1} \sum_{h = k+1}^{k(s)-1} \sum_{(u,v,w)\in T_k} \sum_{(x,y,z)\in T_h} q_t(u,v) \rbr{\frac{ P(w|u,v) \ln \pcons }{ \max\cbr{m_{i(t)}(u,v)  ,1}}}  q_t(x,y|w) B_{i(t)}(x,y,z)  \\
& \leq 16|S||A|\ln\pcons \sum_{t=1}^{T} \sum_{k<h}\sum_{(u,v,w)\in T_k} \sum_{(x,y,z)\in T_h} q_t(u,v) \sqrt{\frac{ P(w|u,v)  }{ \max\cbr{m_{i(t)}(u,v)  ,1}}}  q_t(x,y|w) \sqrt{\frac{ P(z|x,y)  }{ \max\cbr{m_{i(t)}(x,y)  ,1}}} \\
& \quad + \frac{160|S||A|}{3} \sum_{t=1}^{T}\sum_{k < h }\sum_{(u,v,w)\in T_k} \sum_{(x,y,z)\in T_h} q_t(u,v) \sqrt{\frac{ P(w|u,v) \ln \pcons }{ \max\cbr{m_{i(t)}(u,v)  ,1}}}  q_t(x,y|w) \min\cbr{\frac{ P(z|x,y) \ln \pcons }{ \max\cbr{m_{i(t)}(x,y)  ,1}},1}    \\
& \quad + \frac{40|S||A|}{3} \sum_{t=1}^{T} \sum_{k<h} \sum_{(u,v,w)\in T_k} \sum_{(x,y,z)\in T_h} q_t(u,v) \rbr{\frac{ P(w|u,v) \ln \pcons }{ \max\cbr{m_{i(t)}(u,v)  ,1}}}  q_t(x,y|w)
\end{align*}
where the second inequality follows from \pref{lem:conf_width_bound} and \pref{col:conf_width_bound}.

Then we consider bounding these three different terms with the help of previous analysis. According to \pref{eq:aux3} of \pref{lem:extened_mdp_bobw_tran_error_jin2019}, The first term is bounded with probability at least $1-2\delta'$:
\begin{align*}
& 16|S||A|\ln\pcons \sum_{t=1}^{T} \sum_{k<h}\sum_{(u,v,w)\in T_k} \sum_{(x,y,z)\in T_h} q_t(u,v) \sqrt{\frac{ P(w|u,v)  }{ \max\cbr{m_{i(t)}(u,v)  ,1}}}  q_t(x,y|w) \sqrt{\frac{ P(z|x,y)  }{ \max\cbr{m_{i(t)}(x,y)  ,1}}}  \\
& \leq 16|S||A|\ln\pcons \cdot \order\rbr{ \rbr{ |A|\ln T + \ln(L/\delta')} \sum_{k<h} \sqrt{ \abr{S_k} \abr{S_{k+1}}\abr{S_h} \abr{S_{h+1}}    }   } \\
& \leq 16|S||A|\ln\pcons \cdot \order\rbr{ \rbr{ |A|\ln T + \ln(L/\delta')}  \sum_{k<h} \rbr{  \abr{S_k} \abr{S_{k+1}} + \abr{S_h} \abr{S_{h+1}}    } }  \\
& \leq \order\rbr{ \rbr{ |A|\ln T + \ln(L/\delta')}  L|S|^3|A| \ln \pcons },
\end{align*}
where the third line follows from the AM-GM inequality. Taking the expectation with $\delta' = \frac{L}{\pcons}$, we have the expectation of the first term bounded by $ \order\rbr{ L |S|^3|A|^2 \ln^2\pcons }$ using \pref{lem:exp_high_prob_bound}.

On the other hand, for the second term, we have 
\begin{align*}
& \frac{160|S||A|}{3} \sum_{t=1}^{T}\sum_{k < h }\sum_{(u,v,w)\in T_k} \sum_{(x,y,z)\in T_h} q_t(u,v) \sqrt{\frac{ P(w|u,v) \ln \pcons }{ \max\cbr{m_{i(t)}(u,v)  ,1}}}  q_t(x,y|w) \min\cbr{\frac{ P(z|x,y) \ln \pcons }{ \max\cbr{m_{i(t)}(x,y)  ,1}},1 } \\
& \leq \frac{80|S||A|}{3} \sum_{t=1}^{T}\sum_{k < h }\sum_{(u,v,w)\in T_k} \sum_{(x,y,z)\in T_h} q_t(u,v) P(w|u,v) q_t(x,y|w) \rbr{\frac{ P(z|x,y) \ln \pcons }{ \max\cbr{m_{i(t)}(x,y)  ,1}}}  \\
& \quad + \frac{80|S||A|}{3} \sum_{t=1}^{T}\sum_{k < h }\sum_{(u,v,w)\in T_k} \sum_{(x,y,z)\in T_h} q_t(u,v)  \frac{ \ln \pcons }{ \max\cbr{m_{i(t)}(u,v)  ,1}}  q_t(x,y|w) \\
& \leq \frac{80L|S||A|}{3}\ln \pcons \sum_{t=1}^{T}  \sum_{x \in S }\sum_{y \in A} \rbr{\frac{ q_t(x,y)  }{ \max\cbr{m_{i(t)}(x,y)  ,1}}}  \\
& \quad +  \frac{80L|S|^2|A|}{3}  \ln \pcons \sum_{t=1}^{T} \sum_{u \neq s_L}\sum_{v\in A}  \rbr{\frac{q_t(u,v)}{ \max\cbr{m_{i(t)}(u,v)  ,1}} } \\
& \leq \frac{160L|S|^2|A|}{3}  \ln \pcons \sum_{t=1}^{T} \sum_{u \neq s_L}\sum_{v\in A}  \frac{q_t(u,v)}{ \max\cbr{m_{i(t)}(u,v)  ,1}}
\end{align*}
where the expectation of the final term is bounded $\order\rbr{  L|S|^3|A|^2 \ln^2\pcons }$ with the help from \pref{lem:aux_exp}. Similarly, we have the expectation of the third term bounded by $\order\rbr{  L|S|^3|A|^2 \ln^2\pcons}$ following the same idea. 

Therefore, we have $\textsc{Sum}_3$ bounded as 
\begin{align}
\textsc{Sum}_3 & = \order\rbr{ L|S|^3|A|^2 \ln^2\pcons   +L|S|^3|A|^2 \ln^2\pcons + |S||A| T \cdot \delta   } \notag \\
& =  \order\rbr{  L|S|^3|A|^2 \ln^2\pcons + |S||A| T \cdot \delta   } \label{eq:residual_bound_sum_3}
\end{align}
where the $|S||A|T\cdot \delta$ comes from the range of $\textsc{Sum}_3$ and the probability of event $\calA^c$. 

Combining the bounds of $\textsc{Sum}_1$, $\textsc{Sum}_2$, and $\textsc{Sum}_3$ stated in \pref{eq:residual_bound_sum_1}, \pref{eq:residual_bound_sum_2} and \pref{eq:residual_bound_sum_3} finishes the proof. 
\end{proof}

\begin{corollary}
The following holds:
\[
\abr{ q_t(s,a) - u_t(s,a) }  \leq 4 r_t(s,a)  + 16 \sum_{k=0}^{k(s)-1} \sum_{(u,v,w)\in T_k} q_t(u,v)  \sqrt{ \frac{ P(w|u,v) \ln \rbr{ \frac{T|S||A|}{\delta}} }{ \max\cbr{m_{i(t)}(u,v)  ,1}} }q_t(s,a|w). 
\]
where $q_t$ is the true occupancy measure of episode $t$, and $u_t$ is the upper occupancy bound of episode $t$ associated with confidence set $\calP_{i(t)}$ and policy $\pi_t$. 
\label{col:residual_term_property}
\end{corollary}

\begin{proof} Fix the state-action pair $(s,a)$ and episode $t$ . Let $\widehat{P}$ be the transition in $\calP_{i(t)}$ that realizes the maximum in the definition of $u_t(s,a)$, and $\widetilde{q}_t = q^{\widehat{P}, \pi_t}$ bet the associated occupancy measure. Therefore, we have $\widetilde{q}_t(s,a) = u_t(s,a)$. 

Conditioning on event $\calA$, we have 
\begin{align*}
\abr{ q_t(s,a) - \widetilde{q}_t(s,a) } & = \abr{ \sum_{k = 0}^{k(s)-1}  \sum_{(u,v,w)\in T_k} q_t(u,v) \rbr{ P(w|u,v) - \widehat{P}(w|u,v) } \widetilde{q}_t(s,a|w) } \\ 
& \leq \sum_{k = 0}^{k(s)-1} \sum_{(u,v,w)\in T_k}  q_t(u,v)  \abr{ P(w|u,v) - \widehat{P}(w|u,v)} \widetilde{q}_t(s,a|w) \\
& \leq 2 \sum_{k = 0}^{k(s)-1} \sum_{(u,v,w)\in T_k}  q_t(u,v)  B_{i(t)}(u,v,w) \widetilde{q}_t(s,a|w).
\end{align*}
Moreover, we apply \pref{lem:general_occup_diff} to terms $\whatq_t(s,a|w)$ and obtain
\begin{align*}
\abr{ q_t(s,a|w) - \widetilde{q}_t(s,a|w) } & \leq 2 \sum_{h  = k(w) }^{k(s)-1} \sum_{(x,y,z)\in T_h }  q_t(x,y|w)  B_{i(t)}(x,y,z) \widetilde{q}_t(s,a|z) \\
& \leq 2 \sum_{h  = k(w) }^{k(s)-1} \sum_{(x,y,z)\in T_h }  q_t(x,y|w)  B_{i(t)}(x,y,z)
\end{align*}
where the second line uses $ \whatq_t(s,a|z) \leq 1$. 

Combining these inequalities yields (under the event $\calA$) 
\begin{align*}
& \abr{ q_t(s,a) - \widehat{q}_t(s,a) } \\
& \leq 4 \sum_{k = 0}^{k(s)-1} \sum_{(u,v,w)\in T_k}  q_t(u,v)  B_{i(t)}(u,v,w) q_t(s,a|w) \\
& \quad + 4 \sum_{k = 0}^{k(s)-1} \sum_{(u,v,w)\in T_k}  q_t(u,v)  B_{i(t)}(u,v,w)  \rbr{ \sum_{h  = k(w) }^{k(s)-1} \sum_{(x,y,z)\in T_h }  q_t(x,y|w)  B_{i(t)}(x,y,z) } \\
& \leq 16 \sum_{k=0}^{k(s)-1} \sum_{(u,v,w)\in T_k} q_t(u,v)  \sqrt{ \frac{ P(w|u,v) \ln \pcons }{ \max\cbr{m_{i(t)}(u,v)  ,1}} }q_t(s,a|w) \\
& \quad + \frac{160}{3} \sum_{k=0}^{k(s)-1}\sum_{(u,v,w)\in T_k} q_t(u,v) \cdot \frac{ P(w|u,v) \ln \pcons }{ \max\cbr{m_{i(t)}(u,v)  ,1}} \cdot q_t(s,a|w)  \\
& \quad + 4\sum_{k=0}^{k(s)-1} \sum_{h = k+1}^{k(s)-1} \sum_{(u,v,w)\in T_k} \sum_{(x,y,z)\in T_h} q_t(u,v) B_{i(t)}(u,v,w) q_t(x,y|w) B_{i(t)}(x,y,z) 
\end{align*}
where the second line follows from \pref{lem:conf_width_bound}. 

On the other hand,  $ \abr{ q_t(s,a) - \widetilde{q}_t(s,a) } \leq  1$ holds always. Combining the bounds of these two cases finishes the proof. 
\end{proof}

\begin{lemma}
\label{lem:bobw_N_bound}
\pref{alg:bobw_framework} ensures $N \leq 4|S||A|\rbr{ \log T + 1}$ where $N$ is the number of epochs.  
\end{lemma}

\begin{proof} For a fixed state-action pair $(s,a)$, let the $i_1\leq i_2 \leq \ldots \leq i_k$ denotes the epochs that triggered by this state-action pair, that is 
\begin{align*}
\cbr{ i_1, i_2, \ldots, i_k } = \cbr{  i : i \in 1, \ldots N, m_{i}(s,a) \geq \max\cbr{1, 2\cdot m_{i-1}(s,a)}  }.
\end{align*}

Clearly, it holds that 
\begin{align*}
1 = m_{i_1}(s,a), \text{ and } m_{i_\tau}(s,a) \geq 2 m_{i_{\tau-1}}(s,a) \tau \in 2,\ldots, k 
\end{align*}
which indicates that $m_{i_k}(s,a) \geq 2^{k-1}$. Combining with the fact that $m_{i_k}(s,a) \leq T$, we have 
\begin{align*}
k = \abr{\cbr{ i_1, i_2, \ldots, i_k } } \leq 4\log T + 4. 
\end{align*} 
Taking the summation over all state-action pairs finishes the proof. 
\end{proof}

\end{document}